\documentclass{article}
\usepackage{iclr2027_conference,times}

\usepackage{amsmath,amsfonts,amssymb,amsthm,mathtools,bm,mathrsfs}
\usepackage{graphicx}
\usepackage{booktabs,array,makecell,multirow}
\usepackage{enumitem}
\usepackage{algorithm}
\usepackage{algpseudocode}
\usepackage{placeins}
\usepackage{float}
\usepackage[most]{tcolorbox}
\usepackage{diagbox}
\usepackage{ifpdf}
\usepackage{url}
\definecolor{cvprblue}{rgb}{0.21,0.49,0.74}
\usepackage[pagebackref,breaklinks,colorlinks,citecolor=cvprblue]{hyperref}
\usepackage{cleveref}

\makeatletter
\renewcommand{\theHALG@line}{\thealgorithm.\arabic{ALG@line}}
\makeatother

\ifpdf
  \DeclareGraphicsExtensions{.pdf,.png,.jpg,.jpeg}
\else
  \DeclareGraphicsExtensions{.eps}
\fi
\graphicspath{{pdf_fig/}{./}}
\numberwithin{equation}{section}

\theoremstyle{plain}
\newtheorem{theorem}{Theorem}[section]
\newtheorem{lemma}[theorem]{Lemma}
\newtheorem{proposition}[theorem]{Proposition}
\newtheorem{corollary}[theorem]{Corollary}
\theoremstyle{definition}
\newtheorem{assumption}[theorem]{Assumption}
\theoremstyle{remark}
\newtheorem{remark}[theorem]{Remark}

\crefname{theorem}{Theorem}{Theorems}
\Crefname{theorem}{Theorem}{Theorems}
\crefname{lemma}{Lemma}{Lemmas}
\Crefname{lemma}{Lemma}{Lemmas}
\crefname{proposition}{Proposition}{Propositions}
\Crefname{proposition}{Proposition}{Propositions}
\crefname{corollary}{Corollary}{Corollaries}
\Crefname{corollary}{Corollary}{Corollaries}
\crefname{assumption}{Assumption}{Assumptions}
\Crefname{assumption}{Assumption}{Assumptions}
\crefname{remark}{Remark}{Remarks}
\Crefname{remark}{Remark}{Remarks}

\newenvironment{keywords}{\par\noindent\textbf{Keywords: }\ignorespaces}{\par}

\definecolor{MorandiTheoremBack}{RGB}{224,232,238}
\definecolor{MorandiTheoremFrame}{RGB}{116,139,154}
\definecolor{MorandiLemmaBack}{RGB}{232,226,218}
\definecolor{MorandiLemmaFrame}{RGB}{150,134,118}
\tcolorboxenvironment{theorem}{enhanced,breakable,
  colback=MorandiTheoremBack,colframe=MorandiTheoremFrame,
  boxrule=0.5pt,arc=1mm,left=1.2mm,right=1.2mm,top=0.9mm,bottom=0.9mm,
  before skip=3pt,after skip=3pt}
\tcolorboxenvironment{lemma}{enhanced,breakable,
  colback=MorandiLemmaBack,colframe=MorandiLemmaFrame,
  boxrule=0.45pt,arc=1mm,left=1.2mm,right=1.2mm,top=0.9mm,bottom=0.9mm,
  before skip=3pt,after skip=3pt}

\newcommand{\R}{\mathbb{R}}

\newcommand{\T}{\mathbb{T}}
\newcommand{\E}{\mathbb{E}}
\newcommand{\Prob}{\mathbb{P}}
\newcommand{\G}{\mathsf{G}}
\newcommand{\g}{\mathfrak{g}}
\newcommand{\X}{\mathsf{X}}
\newcommand{\Path}{\Omega}
\newcommand{\cP}{\mathcal{P}}

\newcommand{\dd}{\mathrm{d}}
\newcommand{\KL}{\mathrm{KL}}
\newcommand{\Div}{\operatorname{div}}

\newcommand{\Exp}{\operatorname{Exp}}
\newcommand{\Log}{\operatorname{Log}}

\newcommand{\wrap}{\operatorname{wrap}}
\newcommand{\inner}[2]{\left\langle #1,#2\right\rangle}
\newcommand{\norm}[1]{\left\lVert #1\right\rVert}
\newcommand{\abs}[1]{\left\lvert #1\right\rvert}

\newcommand{\TV}{\mathrm{TV}}
\newcommand{\MMD}{\operatorname{MMD}}
\newcommand{\SW}{\operatorname{SW}}
\newcommand{\Vol}{\operatorname{Vol}}
\newcommand{\pdffig}[2][width=\linewidth]{\includegraphics[#1]{#2.pdf}}

\algrenewcommand\algorithmicrequire{\textbf{Require:}}
\algrenewcommand\algorithmicensure{\textbf{Return:}}
\algrenewcommand\algorithmiccomment[1]{\hfill$\triangleright$~#1}

\title{Schr\"odinger Bridges on Lie Group Manifolds for Probabilistic Intrinsic Generation}

\author{%
\begin{minipage}{0.96\textwidth}
\centering\normalfont
Shizhe Zhang\textsuperscript{1,*}\quad
Mingyang Zhao\textsuperscript{2,*,$\dagger$}\quad
Lei Ma\textsuperscript{1,$\dagger$}\\[3pt]
{\small
\textsuperscript{1}National Biomedical Imaging Center, College of Future Technology, Peking University, Beijing, China\\
\textsuperscript{2}State Key Laboratory of Mathematical Sciences, Academy of Mathematics and Systems Science, Chinese Academy of Sciences, Beijing, China\\[2pt]
\texttt{szzhang25@stu.pku.edu.cn}\quad
\texttt{zhaomingyang@amss.ac.cn}\quad
\texttt{lei.ma@pku.edu.cn}\\[1pt]
\textsuperscript{*}Equal contribution. \quad \textsuperscript{$\dagger$}Corresponding authors.}
\end{minipage}%
}

\iclrfinalcopy

\newsavebox{\arxivProteinHistogramDependencyBox}
\sbox{\arxivProteinHistogramDependencyBox}{%
  \includegraphics{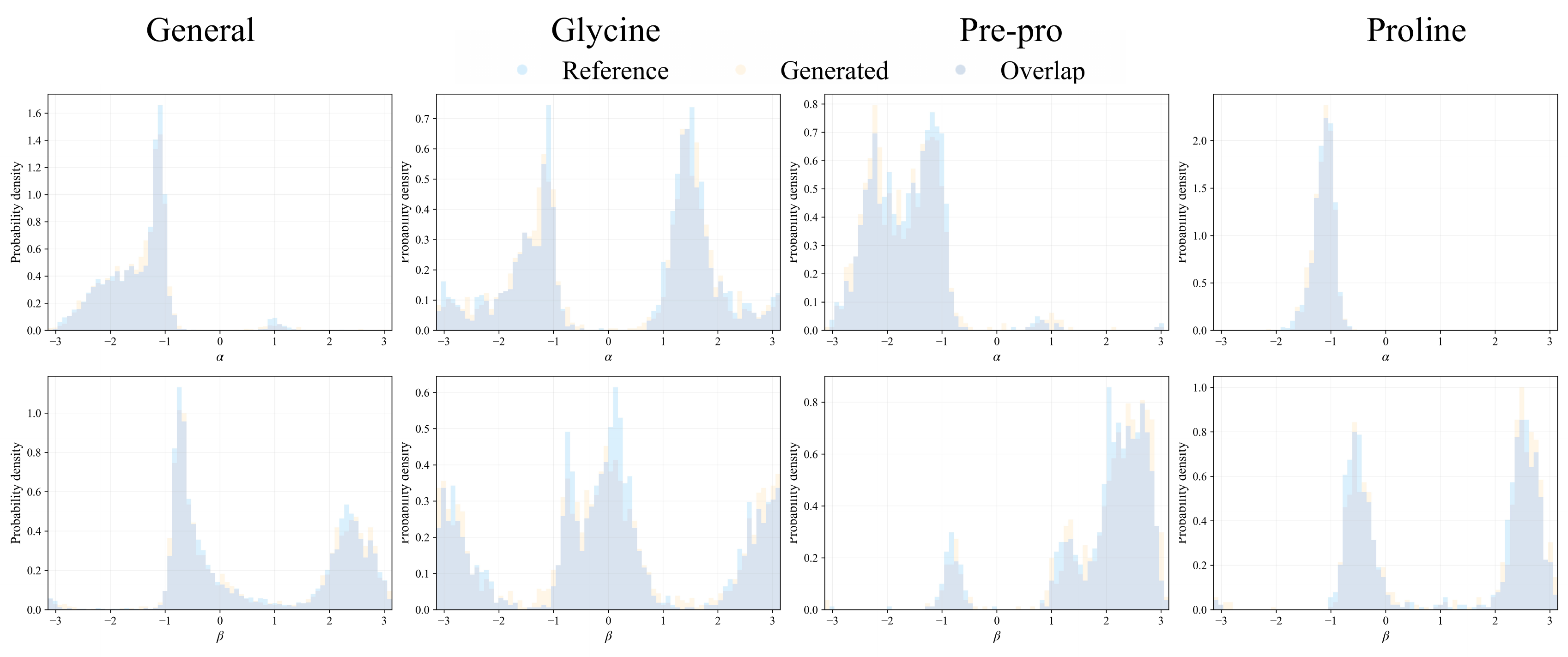}%
}

\ifpdf
\hypersetup{
  pdftitle={Schrödinger Bridges on Lie Group Manifolds for Probabilistic Intrinsic Generation},
  pdfauthor={Shizhe Zhang, Mingyang Zhao, and Lei Ma}
}
\fi

\begin{document}
\maketitle

\pagestyle{plain}
\thispagestyle{plain}

\begin{abstract}
Generative modeling directly on geometric manifolds can avoid errors introduced by flattening non-Euclidean data, repeated ambient projection, and coordinate inconsistency in Euclidean representations. Schr\"odinger bridges provide a probabilistic generative framework for entropy-regularized transport between prescribed endpoint distributions. We study Schr\"odinger bridges for kinetic dynamics on Lie group manifolds with state \(X_t=(g_t,\bm\xi_t)\in\G\times\g\), allowing endpoint observations to constrain only the variables that are actually measured. In particular, the entropy projection determines the conditional law of the unobserved endpoint velocities. For the same observed endpoint bridge, we develop two computational realizations: \emph{Wrapped-Kernel Bridge Calibration} (WKBC) uses an explicit periodized kinetic kernel on compact \emph{Abelian} groups, whereas \emph{Reciprocal Conditional-Control Bridge Matching} (RCCBM) handles compact \emph{non-Abelian} groups through two-sided endpoint calibration and mollified conditional-control matching. The canonical teacher-mixture path law is itself a Markov reciprocal law, so forward generation uses a calibrated initial law and one learned Doob controller. Moreover, we establish a modular error bound in the bounded-Lipschitz path metric that provides a clean separation of errors due to endpoints, control regression, initialization, discretization, and related approximations.
Experiments on multiple Lie group manifold datasets validate the feasibility and consistency of our proposed method, covering protein and RNA torsions, \(SO(3)\), \(U(n)\), and the Protein Conformational Transition Pathway Generation task using mdCATH trajectories in a compact reduced representation. The source code is publicly available at
\url{https://github.com/cafferyzhang12/Schr-dinger_Bridge_on_LieGroup}.
\end{abstract}

\begin{keywords}
Schr\"odinger bridge, Lie group, Lie algebra, manifold generative modeling, stochastic optimal control, Doob $h$-transform
\end{keywords}


\section{Introduction}\label{sec:introduction}

Many scientific data are intrinsically non-Euclidean: molecular torsions are periodic, orientations live on rotation groups, and rigid configurations are naturally described by Lie group variables \citep{falorsi2019relie,rezende2020tori}. Generative and transport models for such data should therefore preserve \emph{intrinsic geometry} throughout their dynamics \citep{mathieu2020rcnf,lou2020manifoldode,debortoli2022riemannian}.

Generative modeling aims to learn complex data distributions together with transformations that produce realistic samples \citep{ho2020ddpm,song2021sde,lipman2023fm} and, increasingly, meaningful stochastic paths between prescribed distributions \citep{debortoli2021sb,shi2023dsbm,liu2024gsbm}. Standard diffusion and flow matching pipelines are commonly formulated by connecting a simple reference distribution to the data distribution \citep{ho2020ddpm,song2021sde,lipman2023fm}; this prior-to-data setting is highly effective for unconditional generation, but it is less directly aligned with scientific problems in which both endpoint populations are observed and the stochastic evolution between them is itself of interest. Schr\"odinger bridges provide a natural alternative: they solve a stochastic optimal transport problem between two prescribed endpoint laws while remaining close, in relative entropy, to a chosen reference dynamics \citep{schrodinger1931,follmer1988,leonard2014,chen2021siamreview}. This distribution-to-distribution formulation supports generative modeling between prescribed endpoint distributions while retaining stochastic path information \citep{debortoli2021sb,shi2023dsbm,liu2024gsbm}, making it well suited to scientific evolution and constrained stochastic generation.

To accommodate scientific data for which only part of the state is observed, let \(Q\in\cP(C([0,T],\X))\) be a reference path law, let \(\mathcal O_i:\X\to\mathsf Y_i\), \(i\in\{0,T\}\), be measurable endpoint observation maps, and let \(\rho_i\) be prescribed observed endpoint laws. We consider the observed endpoint Schr\"odinger problem
\begin{equation}\label{eq:observed-sb}
P^{\rm SB}
=
\operatorname*{argmin}_{P\in\cP(C([0,T],\X))}
\left\{
\KL(P\|Q):
(\mathcal O_0)_\#P_0=\rho_0,\;
(\mathcal O_T)_\#P_T=\rho_T
\right\}.
\end{equation}
Under the endpoint kernel condition introduced in Section~\ref{subsec:conventions}, the entropy projection is unique and admits the two-sided Schr\"odinger scaling
\begin{equation}\label{eq:observed-path-factor}
\frac{\dd P^{\rm SB}}{\dd Q}
=
f_0(\mathcal O_0(X_0))\,g_T(\mathcal O_T(X_T)),
\end{equation}
with positive endpoint factors \(f_0\) and \(g_T\). The propagated terminal factor
\begin{equation*}
h_t(x)
=
\E_Q\!\left[g_T(\mathcal O_T(X_T))\mid X_t=x\right]
\end{equation*}
generates the forward Doob transform, whereas the source factor determines the entropy-optimal initial conditional law of any latent state variables. Thus the same two-sided entropy projection determines the endpoint coupling, latent endpoint conditionals, and the stochastic law of intermediate trajectories.

However, a direct Euclidean parameterization of the state space fails to respect the intrinsic geometry of manifold-valued observations. Periodic coordinates, rotations, and rigid configurations are not unconstrained vectors; hence, one is forced either to repeatedly project Euclidean updates onto the manifold or to transport learned features across location-dependent tangent spaces, both of which incur geometric overhead and compound numerical errors \citep{mathieu2020rcnf,lou2020manifoldode,debortoli2022riemannian}. While intrinsic manifold models preserve the geometric support, they still leave the representation of stochastic velocities, scores, and controls complicated by spatially varying tangent spaces \citep{mathieu2020rcnf,debortoli2022riemannian,chen2024rfm}. For Lie group data, however, the \emph{group structure} offers a natural way out: tangent vectors can be uniformly represented in a single \emph{fixed Lie algebra}, while the configuration itself evolves intrinsically on the group \citep{kong2024klmc,tdm2025}.




Let \(\G\) be a finite-dimensional Lie group with Lie algebra \(\g=T_e\G\). We specialize the state space in Eq.~\eqref{eq:observed-sb} to the \emph{kinetic Lie group state space}
\begin{equation*}
\X:=\G\times\g,
\qquad
X_t=(g_t,\bm\xi_t),
\end{equation*}
where \(g_t\) is the geometric configuration and \(\bm\xi_t\) is a Lie algebra velocity. A general reference process is
\begin{equation}\label{eq:intro-reference-process}
\dd g_t=T_eL_{g_t}(\bm\xi_t)\,\dd t,
\qquad
\dd\bm\xi_t=\bm b(t,g_t,\bm\xi_t)\,\dd t+\alpha\,\dd\mathbf W_t.
\end{equation}
The group coordinate is therefore reconstructed intrinsically, while stochastic forcing and learned controls act in the fixed vector space \(\g\) \citep{kong2024klmc,tdm2025}. In this kinetic setting, the normalized forward Doob control and the corresponding bridge dynamics are
\begin{equation}\label{eq:common-optimal-control}
\bm u_t^*(x)
=
\alpha\nabla_{\bm\xi}\log h_t(x),
\end{equation}
\begin{equation}\label{eq:common-forward-bridge}
\dd g_t=T_eL_{g_t}(\bm\xi_t)\,\dd t,
\qquad
\dd\bm\xi_t=
\bigl[\bm b(t,g_t,\bm\xi_t)+\alpha\bm u_t^*(X_t)\bigr]\dd t
+\alpha\,\dd\mathbf W_t.
\end{equation}
This representation keeps scores, momentum-like variables, and controls in a common Euclidean-type coordinate space without relaxing the geometric constraint on \(g_t\). When only the group coordinate is observed at the endpoints, the Lie algebra velocities remain latent, with their endpoint conditional law determined by the same entropy projection.

\begin{figure}[!tbp]
    \centering
    \includegraphics[width=\textwidth,height=0.58\textheight,keepaspectratio]{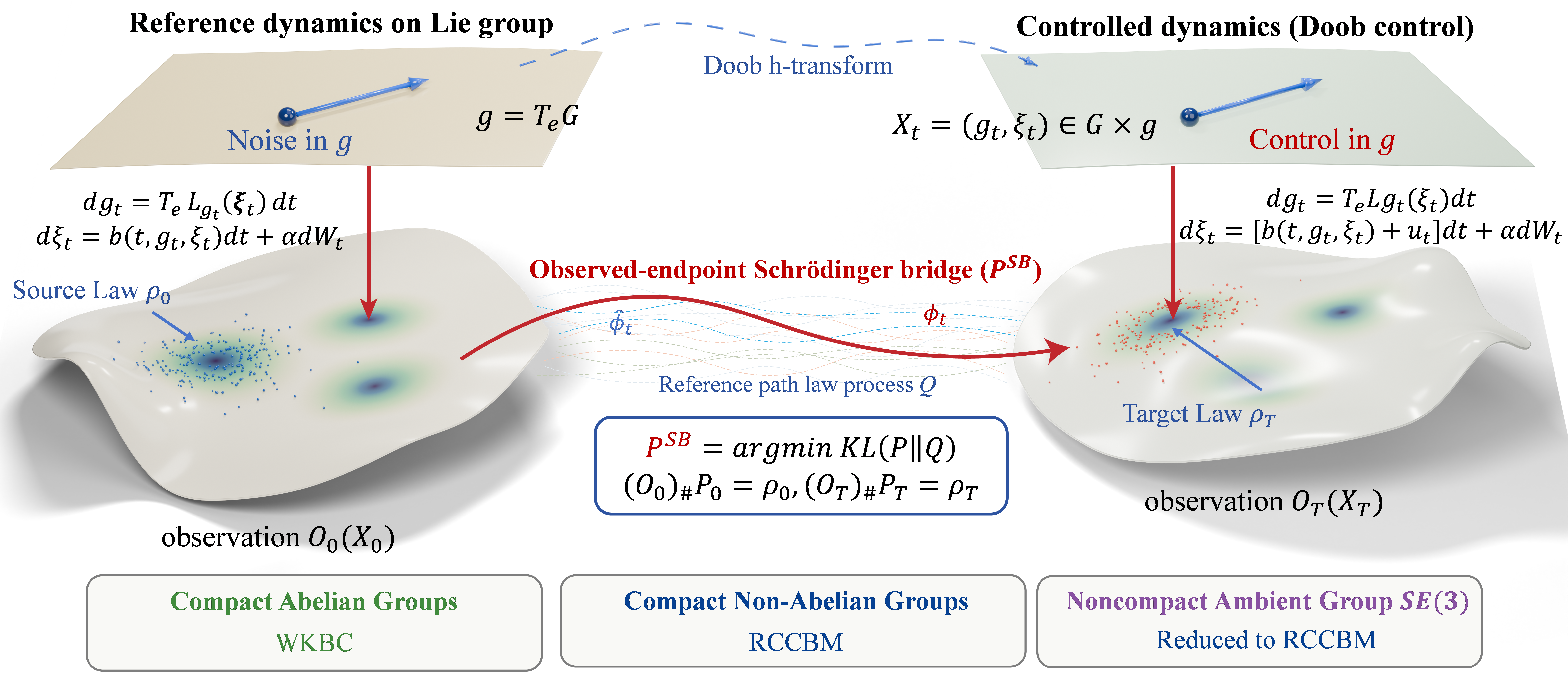}
    \caption{Overview of the observed-endpoint kinetic Schr\"odinger bridge on $\G\times\g$. The reference dynamics are transformed through the Doob control to match the prescribed endpoint observations. WKBC exploits an explicit wrapped kinetic kernel on compact Abelian groups, whereas RCCBM uses endpoint calibration and conditional-control matching on compact non-Abelian groups and on compact reductions of rigid frame data.}
    \label{fig:main-methodology}
\end{figure}

This viewpoint leads to a single observed-endpoint kinetic Schr\"odinger bridge, for which we devise \emph{two computational realizations}. When the kinetic transition kernel is known explicitly, as is the case on compact Abelian groups after periodization, \emph{Wrapped-Kernel Bridge Calibration} (WKBC) directly furnishes the calibrated bridge. When the kernel is not explicitly available, notably on compact non-Abelian groups, \emph{Reciprocal Conditional-Control Bridge Matching} (RCCBM) recovers the Doob field by fusing two-sided endpoint calibration with mollified finite-energy conditional teachers. For rigid frame data, we remove the global rotational component before learning, thereby restricting RCCBM to a compact reduced group. Both procedures target the same bridge, while their computational strategies are determined by the available geometric and kernel structure. Fig.~\ref{fig:main-methodology} summarizes the observed-endpoint kinetic bridge and the two structure-adapted computational realizations developed in this paper.


\paragraph{Contributions} 
This work makes three main contributions.
\begin{enumerate}[label=(\roman*),leftmargin=2.4em]
\item \textbf{Observed-endpoint kinetic bridges.} We develop Schr\"odinger bridges on Lie group manifolds under general endpoint observations, where latent velocities are selected through entropy projection. Subject to the endpoint-kernel condition, we prove two-sided scaling and derive the kinetic Doob representation.
\item \textbf{Structure-adapted bridge computation.} WKBC provides an explicit computational scheme for compact Abelian groups, while RCCBM addresses compact non-Abelian groups via endpoint calibration and reciprocal conditional-control matching. We further establish a bounded modular conditional path-law consistency bound and employ compact gauge reduction for rigid frame data.
\item \textbf{Scientific validation.} Experiments on protein and RNA torsions, \(SO(3)\), \(U(n)\), and the Protein Conformational Transition Pathway Generation task constructed from mdCATH trajectories provide quantitative validation across the evaluated settings. Ablation studies and numerical consistency checks further corroborate the model design and the predicted modular error behavior.
\end{enumerate}

\section{Related Work}\label{sec:related}

This section reviews the literature most closely related to ours, focusing on Schr\"odinger bridges and reciprocal matching, generative modeling on manifolds and Lie groups, kinetic Lie group dynamics, and Schr\"odinger bridges on geometric state spaces.

\paragraph{Schr\"odinger bridges and reciprocal matching}
The Schr\"odinger problem originated as an entropy projection of a reference stochastic process onto prescribed endpoint laws \citep{schrodinger1931,follmer1988}, and was subsequently connected to optimal transport and stochastic control \citep{leonard2014,chen2021siamreview,pavon1991}. This viewpoint leads to computational formulations tailored to generative modeling: De Bortoli et al. \citep{debortoli2021sb} used iterative proportional fitting in Diffusion Schr\"odinger Bridge, Shi et al. \citep{shi2023dsbm} developed simulation-based bridge matching, and Liu et al. \citep{liu2024gsbm} formulated generalized bridge matching through conditional stochastic optimal control. More recent approaches have studied adjoint/corrector constructions \citep{liu2025asbs} and plug-in estimation of bridge drifts from entropic potentials \citep{pooladian2025plugin}. These developments progressively shift computation from explicit potential iteration toward learning conditional or controlled dynamics while retaining the same two-endpoint path-space objective. In parallel, kinetic formulations incorporate auxiliary velocity or momentum variables into the reference dynamics. Chiarini et al. \citep{chiarini2022kinetic} analyzed kinetic Schr\"odinger problems with prescribed position marginals and latent velocities, while Chen et al. \citep{chen2023dmsb} developed a phase-space multi-marginal bridge that reconstructed latent momentum from position observations. Peluchetti \citep{peluchetti2023dbmt} connected conditioned bridge mixtures, reciprocal classes, and Markov transports. Together, these developments provide the probabilistic basis for our observed-endpoint kinetic formulation and for learning a Markov control from endpoint-conditioned bridge laws.

\paragraph{Manifold and Lie group generative modeling}
The development of intrinsic generative models has provided the geometric tools needed to model data whose support is not Euclidean. Falorsi et al. \citep{falorsi2019relie} introduced reparameterizable distributions on Lie groups, and Rezende et al. \citep{rezende2020tori} constructed normalizing flows on tori and spheres. Continuous time formulations then extended these ideas to general curved spaces: Mathieu and Nickel \citep{mathieu2020rcnf} developed Riemannian continuous normalizing flows, while Lou et al. \citep{lou2020manifoldode} formulated neural manifold ODEs. Diffusion and flow methods subsequently provided intrinsic stochastic and transport-based generative mechanisms. De Bortoli et al. \citep{debortoli2022riemannian} developed score-based generative modeling on Riemannian manifolds, Leach et al. \citep{leach2022so3} specialized diffusion modeling to rotational distributions on \(SO(3)\), and Chen and Lipman \citep{chen2024rfm} introduced Riemannian Flow Matching as a geometry-aware extension of flow matching \citep{lipman2023fm}. Conforti et al. \citep{conforti2025kl} established KL-convergence guarantees for score-based diffusion models. For Lie group data, this paradigm renders rotations and periodic variables intrinsic to the model coordinates, as opposed to extrinsic quantities subject to repeated ambient projection~\citep{tdm2025}. These works establish complementary intrinsic regimes for probabilistic, diffusion, and flow-based modeling; our work, in turn, generalizes this principle while additionally incorporating a two-endpoint entropy projection on the kinetic Lie group state space.


\paragraph{Kinetic Lie group dynamics and Lie algebra auxiliary variables}
Kong and Tao \citep{kong2024klmc} analyzed kinetic Langevin Monte Carlo on compact Lie groups, where the momentum evolves in the Lie algebra while the configuration is updated intrinsically on the group. Building on the same geometric separation, Zhu et al. \citep{tdm2025} introduced a trivialized momentum diffusion model, using a fixed Lie algebra auxiliary variable for one-marginal diffusion generative modeling. Bertolini et al. \citep{bertolini2025egsm} developed generalized score matching on Lie groups for group-valued generation. These studies motivate the use of Lie algebra auxiliary coordinates for stochastic dynamics while preserving the group-valued configuration. By contrast, we develop a general framework that lifts Schr\"odinger bridges to Lie group manifolds, accommodating two-sided endpoint observations and entropy-based velocity selection.



\paragraph{Schr\"odinger bridges on manifolds and position of this work}
Thornton et al. \citep{thornton2022rdsb} developed Riemannian Diffusion Schr\"odinger Bridge for manifold-valued state spaces. More recently, Mahmood et al. \citep{mahmood2026compact} formulated Schr\"odinger bridges for Brownian diffusion directly on compact connected Lie groups and established existence and projective uniqueness of the associated Schr\"odinger system. Our formulation instead uses the \emph{kinetic state space} \(\G\times\g\), with intrinsic group reconstruction and stochastic control acting through the Lie algebra velocity, and permits general endpoint observation maps so that unobserved endpoint velocities remain latent. Computationally, WKBC exploits explicit wrapped kinetic kernels on compact Abelian groups, whereas RCCBM uses reciprocal endpoint calibration and conditional-control matching when such kernels are unavailable on compact non-Abelian groups. The present work is therefore positioned as an observed-endpoint kinetic Lie group Schr\"odinger bridge framework with structure-adapted computational realizations.

\section{Methodology}\label{sec:method}

\begin{table}[!tbp]
\centering
\caption{Core notations shared by WKBC and RCCBM.}
\label{tab:notation-guide}
\scriptsize
\setlength{\tabcolsep}{4.0pt}
\renewcommand{\arraystretch}{1.10}
\begin{tabular}{p{0.22\textwidth}p{0.70\textwidth}}
\toprule
Symbol & Meaning \\
\midrule
\(x=(g,\bm\xi)\in\X\) & Kinetic state on \(\G\times\g\). \\
\(Q,P^{\rm SB}\) & Reference law and the common observed-endpoint bridge targeted by both algorithms. \\
\(f_0,g_T\) & Source and terminal endpoint scaling factors. \\
\(h_t,\widehat h_t\) & Backward Doob factor and forward density factor. \\
\(\beta=(\beta_0,\beta_T)\) & Centered endpoint log-potentials used for RCCBM calibration. \\
\(Q^z,Q^{z,\varepsilon},P^{\beta,\varepsilon}\) & Exact conditional bridge, mollified conditional teacher law, and canonical Markov mixture. \\
\(\bm v^{z,\varepsilon},\bm u^{\beta,\varepsilon},\bm u_\theta\) & Conditional teacher control, its population Markov Doob control, and the learned forward control. \\
\bottomrule
\end{tabular}
\end{table}

In Section~\ref{sec:introduction}, we introduced the kinetic reference process, the observed-endpoint entropy projection, its two-sided path factorization, and the forward Doob dynamics. We now present the endpoint regularity underlying these objects, derive the remaining marginal and potential identities, and develop WKBC and RCCBM as two computational realizations of the same bridge. All auxiliary regularity conditions and complete theoretical proofs are provided in the \emph{Appendix}; core notation conventions are summarized in Table~\ref{tab:notation-guide}.



\subsection{Kinetic generator and observed-endpoint Schr\"odinger system}\label{subsec:conventions}

For the kinetic state space introduced in Section~\ref{sec:introduction}, assume that \(\G\) is a finite-dimensional matrix Lie group of dimension \(d\), choose an orthonormal basis \(E_1,\ldots,E_d\) of \(\g\), and let \(\widetilde E_i\) be the associated left-invariant vector fields. Writing \(\bm\xi=\sum_i\xi_iE_i\), define
\begin{equation*}
A(g,\bm\xi)=T_eL_g(\bm\xi),
\qquad
AF=\sum_{i=1}^d\xi_i\widetilde E_iF.
\end{equation*}
Then the generator of Eq.~\eqref{eq:intro-reference-process} is
\begin{equation*}
L_tF=AF+\inner{\bm b_t}{\nabla_{\bm\xi}F}+\frac{\alpha^2}{2}\Delta_{\bm\xi}F.
\end{equation*}
The compact group experiments use the kinetic Ornstein--Uhlenbeck specialization
\begin{equation}\label{eq:tdm-ref}
\bm b(t,g,\bm\xi)=-\gamma\bm\xi,
\qquad
\alpha=\sqrt{2\gamma},
\qquad \gamma>0.
\end{equation}

Write
\begin{equation*}
\mathcal R_{0T}=(\mathcal O_0(X_0),\mathcal O_T(X_T))_\#Q.
\end{equation*}
The exact scaling results use the following uniform endpoint kernel condition; its compact group-marginal verification and the reference-compatible lift of a prescribed group marginal are deferred to Appendix~\ref{app:core-bridge-details}.

\begin{assumption}[Observed endpoint kernel condition]\label{ass:endpoint-kernel}
The measure \(\mathcal R_{0T}\) is equivalent to the product measure
\(\rho_0\otimes\rho_T\), and
\begin{equation*}
k=\frac{\dd\mathcal R_{0T}}{\dd(\rho_0\otimes\rho_T)}
\end{equation*}
satisfies \(0<m\leq k\leq M<\infty\) almost everywhere.
\end{assumption}

This endpoint condition turns the observed-endpoint entropy projection into a bounded two-sided scaling problem. We therefore obtain the following well-posedness result.

\begin{theorem}[Observed endpoint Schr\"odinger system]\label{thm:observed-existence}
Under Assumption~\ref{ass:endpoint-kernel}, the static endpoint problem
\begin{equation*}
\inf_{\Gamma\in\Pi(\rho_0,\rho_T)}\KL(\Gamma\|\mathcal R_{0T})
\end{equation*}
has a unique minimizer \(\Gamma^*\). There exist positive factors \(f_0,f_0^{-1}\in L^\infty(\rho_0)\) and \(g_T,g_T^{-1}\in L^\infty(\rho_T)\), unique up to reciprocal scaling, such that
\begin{equation*}
\Gamma^*(\dd y_0\dd y_T)
=f_0(y_0)g_T(y_T)\mathcal R_{0T}(\dd y_0\dd y_T).
\end{equation*}
With \(\int g_T\,\dd\rho_T=1\), one may choose
\begin{equation*}
\frac1M\leq f_0\leq\frac1m,
\qquad
\frac mM\leq g_T\leq\frac Mm.
\end{equation*}
\end{theorem}

This theorem establishes the well-posed two-sided endpoint scaling used throughout the remainder of the paper. For compact connected \(\G\) with observations restricted to the group coordinate and continuous strictly positive source, target, and reference group densities, Assumption~\ref{ass:endpoint-kernel} is verified in Lemma~\ref{prop:compact-observed-kernel}. The two-sided bound is used for projective contraction and uniform factor control.


\subsection{Bridge factorization}\label{subsec:sch-system}

The following result records the marginal and dual-potential identities used by both WKBC and RCCBM algorithms.

\begin{theorem}[Bridge marginal and dual-potential factorization]\label{thm:bridge-factors}
Under Assumption~\ref{ass:endpoint-kernel} we have 
\begin{equation}\label{eq:bridge-marginal-factor}
P_t^{\rm SB}(\dd x)
=\E_Q[f_0(\mathcal O_0(X_0))\mid X_t=x]h_t(x)Q_t(\dd x).
\end{equation}
If \(Q_t=q_t\mu\), define
\(\widehat h_t(x)=q_t(x)\E_Q[f_0(\mathcal O_0(X_0))\mid X_t=x]\). Then
\begin{equation}\label{eq:bridge-density-dual-potential}
p_t^{\rm SB}=h_t\widehat h_t,
\qquad
\partial_t h_t+L_th_t=0,
\qquad
\partial_t\widehat h_t=L_t^\dagger\widehat h_t,
\end{equation}
with
\begin{equation}\label{eq:factor-boundary-traces}
h_T=g_T\circ\mathcal O_T,
\qquad
\widehat h_0=q_0(f_0\circ\mathcal O_0).
\end{equation}
\end{theorem}

The regularity conditions and Doob transform calculation under Eqs.~\eqref{eq:common-optimal-control}--\eqref{eq:common-forward-bridge}, together with path-stability identities, are collected in Appendix~\ref{app:doob-auxiliary-details}.

\subsection{Two computational regimes}
We now introduce the same Schr\"odinger bridge instantiated under two complementary kernel-access settings: WKBC employs an explicit wrapped kinetic kernel, while RCCBM relies on endpoint calibration together with conditional reference bridges.


\paragraph{Compact Abelian groups: WKBC}\label{subsec:wkbc}

Let \(\G=\T^m=\R^m/\Lambda\), where \(\Lambda\subset\R^m\) is a full-rank lattice defining the torus. The reference dynamics in Eq.~\eqref{eq:intro-reference-process} under the Ornstein--Uhlenbeck specialization (Eq.~\eqref{eq:tdm-ref}) has a Gaussian lift in \(\R^m\times\R^m\), and periodization over \(\Lambda\) gives the exact kinetic kernel. The full covariance calculation, wrapped joint density, and group-marginal kernel are stated in Lemma~\ref{prop:wrapped-kinetic-kernel}; in particular Eq.~\eqref{eq:wrapped-position-kernel} is an explicit observation kernel when only the group coordinate is observed at the endpoints.

For a terminal log-factor \(\beta\), define
\begin{equation*}
h_t^\beta(x)=\int q_{T-t}^{\mathcal O_T}(y\mid x)e^{\beta(y)}\,\nu_T(\dd y).
\end{equation*}

The explicit wrapped kernel makes this propagated terminal factor computable, turning the abstract bridge factorization above into an explicit control representation. This yields the following corollary.

\begin{corollary}[Exact WKBC representation]\label{thm:wkbc}
For \(\beta^*=\log g_T\), differentiation under the wrapped-kernel integral gives
\begin{equation*}
h_t^{\beta^*}=h_t,
\qquad
\bm u_t^{\rm WKBC}=\sqrt{2\gamma}\nabla_{\bm\xi}\log h_t^{\beta^*}=\bm u_t^*.
\end{equation*}
The forward process must be initialized from
\begin{equation*}
p_0^{\rm SB}(x)=q_0(x)f_0(\mathcal O_0(x))h_0^{\beta^*}(x).
\end{equation*}
When only the source group coordinate is observed,
\begin{equation*}
P_0^{\rm SB}(\dd\bm\xi\mid g)
=
\frac{h_0^{\beta^*}(g,\bm\xi)Q_0(\dd\bm\xi\mid g)}
{\int h_0^{\beta^*}(g,\bm\eta)Q_0(\dd\bm\eta\mid g)}.
\end{equation*}
\end{corollary}

The wrapped kernel therefore provides an explicit numerical oracle for the same bridge characterized by Theorems~\ref{thm:observed-existence}--\ref{thm:bridge-factors}. The bridge-weighted score-regression identity and the corresponding kernel/quadrature/Sinkhorn/interpolation error decomposition are provided in Appendix~\ref{app:wkbc-theory-details}; Proposition~\ref{thm:wkbc-stability} reports the population and numerical consistency statement without an additional occupation-density-ratio constant.

{
\paragraph{{Likelihood-facing WKBC specialization}}
For held-out likelihood evaluation, we use a prior-to-data benchmark with the Haar-Gaussian stationary source employed by TDM~\citep{tdm2025}. Let \(\nu_{\G}\) denote the normalized Haar measure on the compact group and define
\[
\varphi_d(\bm\xi)=(2\pi)^{-d/2}\exp\!\left(-\frac12\norm{\bm\xi}^2\right),
\qquad
\pi^\star(\dd g\,\dd\bm\xi)=\nu_{\G}(\dd g)\varphi_d(\bm\xi)\,\dd\bm\xi .
\]
The compact group component of \(\pi^\star\) is Haar-distributed, and its Lie algebra component is standard Gaussian. Following the TDM benchmark data augmentation, each group-valued training datum \(g\) is paired independently with \(\bm\xi\sim\mathcal N(0,I_d)\). Hence
\[
P_0^{\rm lik}=\pi^\star,
\qquad
P_T^{\rm lik}(\dd g\,\dd\bm\xi)
=
\rho_{\rm train}(\dd g)\varphi_d(\bm\xi)\,\dd\bm\xi,
\qquad
\mathcal O_0^{\rm lik}(g,\bm\xi)
=
\mathcal O_T^{\rm lik}(g,\bm\xi)
=
(g,\bm\xi).
\]
This specification defines the likelihood protocol used in the benchmark tables. The scientific bridge experiments use the observed-endpoint formulation above, with latent velocity conditionals determined by the entropy projection.

WKBC obtains the full kinetic marginal score directly from the calibrated wrapped kernel. Besides the backward factor \(h_t\), the explicit wrapped joint kernel propagates the source factor as
\begin{equation}\label{eq:wkbc-forward-factor-likelihood}
\widehat h_t(x)
=
\int_{\X}
q_{0,t}(x\mid x_0)\,q_0(x_0)\,
f_0(\mathcal O_0(x_0))\,\mu(\dd x_0),
\end{equation}
so that Theorem~\ref{thm:bridge-factors} gives
\begin{equation}\label{eq:wkbc-full-score-likelihood}
\bm s_t^{\rm WKBC}(x)
:=
\nabla_{\bm\xi}\log p_t^{\rm SB}(x)
=
\nabla_{\bm\xi}\log h_t(x)
+
\nabla_{\bm\xi}\log \widehat h_t(x).
\end{equation}
Both terms are obtained from the calibrated factors and wrapped kernel. The likelihood evaluator uses this full marginal score, while stochastic WKBC generation uses the forward Doob control in Corollary~\ref{thm:wkbc}. The probability-flow change-of-variables formula and its numerical implementation are given in Appendix~\ref{app:tdm-likelihood}.
}

\paragraph{Compact non-Abelian groups: RCCBM}\label{subsec:apbm}

Let \(\G\) be compact, connected, and non-Abelian. RCCBM targets the same \(P^{\rm SB}\) without requiring a closed-form transition kernel. The construction has four stages:
\begin{enumerate}[label=\arabic*.,leftmargin=2.1em]
    \item calibrate the two-sided endpoint reciprocal law;
    \item construct finite energy mollified conditional-control teachers;
    \item regress one Markov forward control using the canonical teacher mixture and importance correction;
    \item perform terminal refinement only under a matching-loss constraint.
\end{enumerate}
The main text keeps the identities that connect these stages; empirical-process assumptions and implementation details are deferred to Appendix~\ref{app:rccbm-detailed-theory}.

\paragraph{Endpoint reciprocal calibration}
Set \(Y=(Y_0,Y_T)=(\mathcal O_0(X_0),\mathcal O_T(X_T))\). For bounded \(\beta=(\beta_0,\beta_T)\), define
\begin{align}
Z(\beta)&=\E_{\mathcal R_{0T}}e^{\beta_0(Y_0)+\beta_T(Y_T)},
\label{eq:rcm-endpoint-partition}\\
\frac{\dd\Gamma^\beta}{\dd\mathcal R_{0T}}
&=Z(\beta)^{-1}e^{\beta_0(y_0)+\beta_T(y_T)},
\label{eq:rcm-endpoint-tilt}\\
\mathscr D(\beta)&=\E_{\rho_0}\beta_0+\E_{\rho_T}\beta_T-\log Z(\beta).
\label{eq:rcm-population-dual}
\end{align}
We fix the gauge by
\begin{equation}\label{eq:rccbm-double-centering}
\E_{\rho_0}\beta_0=0,
\qquad
\E_{\rho_T}\beta_T=0.
\end{equation}
Let \(Q^y\) be the conditional reference law given \(Y=y\), and define \(P_\Gamma=\int Q^y\Gamma(\dd y)\); the entropy-disintegration justification is presented in Lemma~\ref{lem:endpoint-disintegration} in Appendix~\ref{app:rccbm-detailed-theory}. With the gauge fixed, the centered dual provides a direct parameterization of the reciprocal endpoint correction. Its relation to the Schr\"odinger bridge is given by the following theorem.

\begin{theorem}[Endpoint dual and reciprocal correction]\label{thm:rcm-endpoint-dual}
Under Assumption~\ref{ass:endpoint-kernel}, \(\mathscr D\) has a unique maximizer \(\beta^*\) in the centered class, \(\Gamma^{\beta^*}=\Gamma^*\), and the endpoint tilt agrees with the Schr\"odinger factors up to the fixed gauge. Moreover, for every bounded centered \(\beta\), the following holds
\begin{equation}\label{eq:rcm-dual-gap}
\mathscr D(\beta^*)-\mathscr D(\beta)
=\KL(\Gamma^*\|\Gamma^\beta)
=\KL(P^{\rm SB}\|P^\beta),
\qquad P^\beta:=P_{\Gamma^\beta}.
\end{equation}
\end{theorem}

Hence endpoint calibration is not an auxiliary heuristic: its population dual gap is exactly a path-law KL error. The empirical dual, half-bridge interpretation, and calibration consistency bound are given in Appendix~\ref{app:rccbm-detailed-theory}.

\paragraph{Exact conditional-control identity}
Once the reciprocal endpoint law has been calibrated, the next step is to identify its Markov forward control from endpoint-conditioned reference bridges. Let \(Z_c=(X_0,Y_T)\), \(\Lambda^\beta=(Z_c)_\#P^\beta\), and let \(Q^z\) be the conditional reference law given \(Z_c=z=(x_0,y)\). Denoting the reference observation density by \(r_{t,T}^{\mathcal O_T}(y\mid x)\), this leads to the following conditional-control identity. 

\begin{theorem}[Exact conditional-control matching]\label{thm:rccbm-exact-matching}
Under the conditional-kernel regularity in Assumption~\ref{ass:rccbm-conditional-kernel}, we have
\begin{equation}\label{eq:rccbm-lifted-disintegration}
P^\beta=\int Q^z\,\Lambda^\beta(\dd z).
\end{equation}
Its backward factor and Markov Doob control are
\begin{equation}\label{eq:rccbm-beta-backward-factor}
h_t^\beta(x)
=\int r_{t,T}^{\mathcal O_T}(y\mid x)e^{\beta_T(y)}\,\nu_T(\dd y),
\end{equation}
\begin{equation}\label{eq:rccbm-beta-control}
\bm u_t^\beta=\alpha\nabla_{\bm\xi}\log h_t^\beta.
\end{equation}
For \(z=(x_0,y)\), the exact conditional bridge has preterminal control
\begin{equation}\label{eq:rccbm-exact-teacher-control}
\bm v_t^z(x)=\alpha\nabla_{\bm\xi}\log r_{t,T}^{\mathcal O_T}(y\mid x),
\qquad t<T,
\end{equation}
and
\begin{equation*}
\E[\bm v_t^{Z_c}(X_t)\mid X_t=x]=\bm u_t^\beta(x),
\qquad P_t^\beta\text{-a.e.}
\end{equation*}
Equivalently, on every preterminal strip the conditional squared-error risk has excess risk
\begin{equation}\label{eq:rccbm-exact-excess-risk}
\mathcal L_{\beta,\tau}(\bm u)-\mathcal L_{\beta,\tau}(\bm u^\beta)
=\int_0^{T-\tau}\lambda_\tau(t)
\int\norm{\bm u-\bm u_t^\beta}^2P_t^\beta(\dd x)\,\dd t.
\end{equation}
\end{theorem}

The identity shows why conditional teachers can recover the Markov Doob field. Because point-pinned controls can become singular as \(t\rightarrow T\), RCCBM uses the mollified terminal condition below to construct finite-energy teachers on the training horizon.

\paragraph{Canonical mollified teachers}
Let \(\kappa_\varepsilon(y,y')>0\) be a continuous approximate identity and define, for \(z=(x_0,y)\),
\begin{equation}\label{eq:rccbm-mollified-conditional-law}
\frac{\dd Q^{z,\varepsilon}}{\dd Q^{x_0}}
=\frac{\kappa_\varepsilon(y,\mathcal O_T(X_T))}{Z_{z,\varepsilon}},
\qquad
Z_{z,\varepsilon}
=\E_{Q^{x_0}}\kappa_\varepsilon(y,\mathcal O_T(X_T)).
\end{equation}
With
\begin{equation*}
g_{T,\varepsilon}^\beta(y')
=\int\kappa_\varepsilon(y,y')e^{\beta_T(y)}\,\nu_T(\dd y),
\end{equation*}
and
\begin{equation*}
Z_\varepsilon(\beta)
=\E_Q\!\left[e^{\beta_0(\mathcal O_0(X_0))}
 g_{T,\varepsilon}^\beta(\mathcal O_T(X_T))\right],
\end{equation*}
define the canonical condition law
\begin{equation}\label{eq:rccbm-mollified-condition-law}
\Lambda^{\beta,\varepsilon}(\dd x_0\dd y)
=\frac{e^{\beta_0(\mathcal O_0(x_0))+\beta_T(y)}Z_{(x_0,y),\varepsilon}}
{Z_\varepsilon(\beta)}Q_0(\dd x_0)\nu_T(\dd y).
\end{equation}
The \(Z_{(x_0,y),\varepsilon}\) factor is essential. Then, it gives
\begin{equation}\label{eq:rccbm-mollified-reciprocal-law}
P^{\beta,\varepsilon}
:=\int Q^{z,\varepsilon}\Lambda^{\beta,\varepsilon}(\dd z),
\qquad
\frac{\dd P^{\beta,\varepsilon}}{\dd Q}
=\frac{e^{\beta_0(\mathcal O_0(X_0))}g_{T,\varepsilon}^\beta(\mathcal O_T(X_T))}
{Z_\varepsilon(\beta)}.
\end{equation}
Thus the conditional teacher mixture is a Markov reciprocal law. Its Markov structure identifies the corresponding Doob control through conditional regression, leading to the following proposition.

\begin{proposition}[Mollified population regression]\label{thm:rcm-control-regression}
Let \(\bm v^{z,\varepsilon}\) be the Doob control of \(Q^{z,\varepsilon}\), and let \(\bm u^{\beta,\varepsilon}\) be the Doob control of \(P^{\beta,\varepsilon}\). Then
\begin{equation*}
\E[\bm v_t^{Z,\varepsilon}(X_t)\mid X_t=x]
=\bm u_t^{\beta,\varepsilon}(x).
\end{equation*}
For any time density \(\lambda_{\rm bm}\geq\underline\lambda>0\), the population matching loss satisfies
\begin{equation}\label{eq:rccbm-mollified-excess-risk}
\mathcal L_{\beta,\varepsilon}(\bm u)
-\mathcal L_{\beta,\varepsilon}(\bm u^{\beta,\varepsilon})
=\int_0^T\lambda_{\rm bm}(t)
\int\norm{\bm u-\bm u_t^{\beta,\varepsilon}}^2
P_t^{\beta,\varepsilon}(\dd x)\,\dd t.
\end{equation}
\end{proposition}

The production time sampler is the explicit preterminal/terminal-window mixture specified in Appendix~\ref{app:lsapbm-algorithms}. Its density has a deterministic lower bound obtained directly from the two mixture components, so the constant entering Theorem~\ref{thm:rcm-conditional-consistency} is derived from the sampling design rather than tuned independently.

This proposition is the central RCCBM learning identity: finite-energy conditional teachers can be pooled without changing the population target away from the Markov Doob control. Conditional Stochastic Optimal Control (CondSOC)~\citep{liu2024gsbm}, replay importance correction, clipping, constrained refinement, and the calibrated empirical initial law are implementation devices for approximating this population objective; their precise definitions are provided in  Appendix~\ref{app:rccbm-detailed-theory} and Algorithm~\ref{alg:apbm-training}.

\paragraph{Conditional path-law error propagation}
Let \(\widehat\beta_N\) be the learned endpoint potential and define
\begin{equation*}
\varepsilon_{{\rm end},N}
:=\KL(\Gamma^*\|\Gamma^{\widehat\beta_N}),
\qquad
\varepsilon_{{\rm mol},N}
:=d_{\rm BL}(P^{\widehat\beta_N,\varepsilon_N},P^{\widehat\beta_N}),
\end{equation*}
\begin{equation*}
\Delta_{{\rm ctrl},N}
:=\mathcal L_{\widehat\beta_N,\varepsilon_N}(\bm u_{\theta_N})
-\mathcal L_{\widehat\beta_N,\varepsilon_N}
(\bm u^{\widehat\beta_N,\varepsilon_N}).
\end{equation*}
The remaining assumptions are stated together in Appendix~\ref{app:rccbm-detailed-theory}: endpoint calibration, conditional-kernel regularity, teacher/empirical-risk approximation, uniform mollification, initial-law stability, and Lie-Trotter convergence. These conditions connect the preceding calibration and regression stages to the implemented generator, leading to the following path-law bound. 

\begin{theorem}[Conditional RCCBM path-law error propagation]\label{thm:rcm-conditional-consistency}
Under Assumptions~\ref{ass:rcm-calibration}, \ref{ass:rccbm-conditional-kernel}, \ref{ass:rccbm-teacher-consistency}, \ref{ass:rccbm-mollification-regularity}, and~\ref{ass:rccbm-statistical-stability}, we have 
\begin{equation}\label{eq:rccbm-control-risk-bound}
\Delta_{{\rm ctrl},N}
\leq\varepsilon_{{\rm app},N}+\varepsilon_{{\rm opt},N}
+4\varepsilon_{{\rm gen},N}+4\varepsilon_{{\rm teach},N}
+2\varepsilon_{{\rm iw},N}+2\varepsilon_{{\rm clip},N}
+\delta_{{\rm ref},N}.
\end{equation}
If \(P_{N,\Delta t}^{\theta_N}\) is the implemented generator initialized from the calibrated empirical initial law in Eq.~\eqref{eq:rcm-calibrated-initial-law}, then
\begin{align}\label{eq:rccbm-statistical-bl-bound}
d_{\rm BL}(P^{\rm SB},P_{N,\Delta t}^{\theta_N})
&\leq\sqrt{2\varepsilon_{{\rm end},N}}
+\varepsilon_{{\rm mol},N}
+\sqrt{\frac{\Delta_{{\rm ctrl},N}}{\underline\lambda}}
\notag\\
&\quad+C_{\rm init}\varepsilon_{{\rm init},N}
+\varepsilon_{{\rm disc},N}.
\end{align}
Consequently, if the component errors specified in the Appendix vanish, then
\(P_{N,\Delta t}^{\theta_N}\Rightarrow P^{\rm SB}\) in bounded-Lipschitz path distance in probability.
\end{theorem}

Theorem~\ref{thm:rcm-conditional-consistency} propagates identifiable module-level errors through the RCCBM pipeline to the final path law under the stated consistency assumptions for the individual modules. At the population level, exact endpoint calibration and exact mollified regression recover \(P^{\beta^*,\varepsilon}\), and the mollification result in Proposition~\ref{prop:rccbm-condsoc-consistency} gives
\begin{equation*}
d_{\rm BL}(P^{\beta^*,\varepsilon},P^{\rm SB})\to0
\qquad(\varepsilon\rightarrow0).
\end{equation*}

The preceding theorem is stated at a high level because RCCBM is intended for non-Abelian settings in which the transition kernel is not explicitly available. The assumptions are nevertheless jointly realizable in a concrete model class. The following specialization uses the torus, where the wrapped kinetic kernel makes every population object explicit and permits a direct verification of the approximation chain.

\begin{corollary}[Concrete torus sieve consistency]\label{cor:rccbm-torus-consistency}
Let \(\G=\T^m\), \(\mathcal O_0=\mathcal O_T=g\), and let the reference process be Eq.~\eqref{eq:intro-reference-process} with the Ornstein--Uhlenbeck specialization \eqref{eq:tdm-ref}. Assume that the source, target, and reference group-marginal densities are \(C^2\), strictly positive, and bounded above and below with respect to normalized Haar measure. Let \(\kappa_{\varepsilon}\) be the periodic heat kernel on \(\T^m\), with \(\varepsilon_N\rightarrow0\).

For endpoint calibration, let \(\mathcal B_N\) be the doubly centered trigonometric-polynomial sieve with frequencies \(\|k\|_\infty\le K_N\), restricted to a fixed bounded Lipschitz ball containing the Fej\'er approximants of the exact endpoint potentials. For direct control regression, let \(\mathcal U_N\) be a clipped tensor-product sieve, restricted to a common Lipschitz/linear-growth envelope, built from periodic trigonometric functions in \(g\), splines in \(t\), and polynomials in \(\bm\xi\) on \(\{\|\bm\xi\|\le R_N\}\), with \(R_N,B_N\to\infty\). Denote
\[
d_N=\dim(\mathcal B_N),\qquad p_N=\dim(\mathcal U_N),
\]
and let
\[
M_N:=1+\E_{\mathfrak M^{\beta^*,\varepsilon_N}}
       \bigl[\|\bm v^{Z,\varepsilon_N}\|^4\bigr].
\]
Choose the sieve growth and mollification schedule slowly enough that
\[
\frac{d_N\log N}{N}\to0,\qquad
\frac{M_N B_N^4 p_N\log N}{N}\to0,
\qquad
B_N^2 e^{-cR_N^2}\to0
\]
for some \(c>0\). Use the canonical mollified teacher law, available from the wrapped kernel, so that replay importance weights are identically one; use exact empirical optimization over the compact finite-dimensional sieves; form the calibrated self-normalized initial pool from \(N\) independent reference samples; and let the Lie--Trotter step satisfy \(\Delta t_N\rightarrow0\).

Then Assumptions~\ref{ass:rcm-calibration}, \ref{ass:rccbm-conditional-kernel}, \ref{ass:rccbm-teacher-consistency}, \ref{ass:rccbm-mollification-regularity}, and~\ref{ass:rccbm-statistical-stability} hold for this torus sieve regime. Consequently,
\[
d_{\rm BL}\!\left(P^{\rm SB},P_{N,\Delta t_N}^{\theta_N}\right)
\longrightarrow0
\qquad\text{in probability}.
\]
\end{corollary}

Corollary~\ref{cor:rccbm-torus-consistency} provides a concrete torus specialization of the RCCBM error-propagation theorem. Its proof in Appendix~\ref{app:torus-rccbm-verification} verifies the endpoint, mollification, empirical-risk, initial-law, and discretization requirements using the explicit wrapped kernel and finite-dimensional sieve construction.

\subsection{Reduction of \(SE(3)^N\) data}\label{subsec:noncompact-general}

For rigid frames \(h_i=(\mathbf{R}_i,\bm p_i)\in SE(3)\), gauge fixing by
\begin{equation}\label{eq:relative-frame-gauge}
\widetilde h_i=h_1^{-1}h_i,
\qquad i=2,\ldots,N,
\end{equation}
removes the global rigid motion. Retained relative rotations and periodic internal coordinates define the compact learning group
\begin{equation*}
\G_{\rm red}=SO(3)^{N-1}\times\T^K,
\qquad
\g_{\rm red}=\mathfrak{so}(3)^{N-1}\times\R^K,
\end{equation*}
with reduction
\begin{equation}\label{eq:se3-reduction-map}
\mathcal R_{\rm red}(h_{1:N})
=(\mathbf{R}_1^\top \mathbf{R}_2,\ldots,\mathbf{R}_1^\top \mathbf{R}_N,\theta_1,\ldots,\theta_K).
\end{equation}
A deterministic reconstruction may use
\begin{equation}\label{eq:forward-kinematics-reconstruction}
\bm p_{i+1}(z)=\bm p_i(z)+\mathbf{R}_i(z)\bm b_i.
\end{equation}

The following proposition formalizes how this reduction transfers the path-law comparison between the reduced and reconstructed representations.

\begin{proposition}[Compact reduction and path-law pushforward]\label{thm:compact-reduction}
Under the reduced-representation assumption in Appendix~\ref{app:reduction-details}, the stochastic learning problem is an RCCBM problem on \(\G_{\rm red}\). If the reconstruction is a \(C^2\)-diffeomorphism onto the modeled conformation manifold, the induced path map preserves relative entropy,
\begin{equation}\label{eq:compact-reduction-kl-invariance}
\KL(P_{\rm red}^{\rm SB}\|\widehat P_{\rm red})
=\KL((\mathbf\Phi_{\rm red})_\#P_{\rm red}^{\rm SB}
\|(\mathbf\Phi_{\rm red})_\#\widehat P_{\rm red}).
\end{equation}
Without injectivity, only the data-processing inequality
\begin{equation}\label{eq:se3-data-processing}
\KL((\mathbf{\mathcal R}_{\rm red})_\#P\|
(\mathbf{\mathcal R}_{\rm red})_\#Q)\leq\KL(P\|Q)
\end{equation}
is available.
\end{proposition}

Accordingly, the stochastic guarantee is a guarantee on the compact reduced state space. Cartesian reconstruction is postprocessing and must be evaluated separately whenever it is approximate.

\section{Experiments}\label{sec:experiments}
In this section, we conduct extensive experiments to validate the two computational regimes separately: explicit-kernel WKBC and sample-based RCCBM. Each experiment specifies the ordered endpoint pair, observation maps, reference process, calibrated initial-state sampler, and integration grid. The compact Abelian benchmarks additionally report held-out Negative Log-Likelihood (NLL), while the \(SO(3)\) benchmark reports log-likelihood under the common Trivialized Momentum Diffusion Models (TDM) benchmark convention~\citep{tdm2025}. RCCBM is evaluated using intrinsic feature Maximum Mean Discrepancy (MMD), geodesic Wasserstein proxies, mode coverage, local precision and coverage radii, endpoint-calibration diagnostics, held-out teacher error, post-refinement matching error, normalized-control energy, and group-constraint error. For Protein Conformational Transition Pathway Generation on mdCATH trajectories, we additionally report target-aware tortuosity and normalized Lie algebra roughness to characterize path detour and local fluctuation in the reduced state space.

\subsection{Experimental protocol}\label{subsec:experimental-protocol}

For the compact Abelian experiments, we use a common set of endpoint, path, and geometric diagnostics across methods, together with held-out Negative Log-Likelihood (NLL) for benchmark comparison. We compare WKBC with Riemannian Diffusion Models (RDM)~\citep{huang2022rdm}, Riemannian Flow Matching (RFM)~\citep{chen2024rfm}, and TDM~\citep{tdm2025}. RDM extends continuous-time diffusion generative modeling to general Riemannian manifolds, RFM constructs geometry-aware flow-matching dynamics on manifolds, and TDM exploits the trivialization structure of Lie groups to perform diffusion generative modeling with auxiliary variables represented in a fixed Lie algebra. These baselines are selected because they are representative generative methods specifically designed for manifold- or Lie-group-valued data, thereby providing geometry-aware comparisons with WKBC. For all compact Abelian benchmark experiments, WKBC and the RDM, RFM, and TDM baselines are trained from the same simple prior toward the same target data distribution.

For compact non-Abelian Lie groups, including \(SO(3)\) and \(U(n)\), we evaluate the model using log-likelihood where a common likelihood protocol is available, together with intrinsic geometric and distributional diagnostics and qualitative visualizations. All production feature maps and other evaluation parameters are fixed before checkpoint selection to prevent the evaluation protocol from being influenced by model selection.

The final RCCBM checkpoint is selected by a fixed validation distribution objective. Endpoint dual, teacher, control matching, and sample quality diagnostics are recomputed after the terminal refinement stage so that all reported quantities correspond to the same final controller.

The NLL/LL tables use the prior-to-data likelihood protocol defined in Section~3.3. Held-out likelihoods are evaluated by the intrinsic probability-flow change-of-variables construction used by TDM~\citep{tdm2025}, using the same auxiliary Gaussian variable and group-volume convention. We report
\[
\mathrm{LL}
=
\frac1{N_{\rm test}}\sum_{i=1}^{N_{\rm test}}\widehat\ell(g_i),
\qquad
\mathrm{NLL}=-\mathrm{LL}.
\]
WKBC obtains the required full marginal score from Eq.~\eqref{eq:wkbc-full-score-likelihood}; RCCBM obtains it from an auxiliary score estimator fitted to trajectories generated by the frozen final controller. Complete formulas and proofs are given in Appendix~\ref{app:tdm-likelihood}.

\subsection{Compact Abelian torsion bridges}\label{subsec:compact-abelian-experiments}

\begin{table}[!tbp]
\centering
\caption{Held-out negative log-likelihood (NLL) on the four two-dimensional protein torsion datasets and the seven-dimensional RNA torsion dataset. NLL is evaluated from the prior-to-data Haar-Gaussian benchmark specialization described in Section~\ref{subsec:experimental-protocol}. Lower NLL indicates better likelihood fit. Boldface indicates the best-performing value.}

\label{tab:tdm-table1-nll}
\scriptsize
\setlength{\tabcolsep}{2.4pt}
\renewcommand{\arraystretch}{1.06}
\begin{tabular*}{\textwidth}{@{\extracolsep{\fill}}lccccc@{}}
\toprule
Setting & \makecell{General\\(2D)} & \makecell{Glycine\\(2D)} & \makecell{Proline\\(2D)} & \makecell{Pre-Pro\\(2D)} & \makecell{RNA\\(7D)} \\
\midrule
\multicolumn{1}{l}{\diagbox{Model}{Dataset size}} & 138208 & 13283 & 7634 & 6910 & 9478 \\
\midrule
RDM~\citep{huang2022rdm} & 1.04 & 1.97 & 0.12 & 1.24 & -3.70 \\
RFM~\citep{chen2024rfm} & 1.01 & 1.90 & 0.15 & 1.18 & -5.20 \\
TDM~\citep{tdm2025} & 0.69 & 1.04 & -0.60 & 0.52 & -6.86 \\
\midrule
WKBC  & \textbf{0.59} & \textbf{0.84} & \textbf{-1.02} & \textbf{0.17} & \textbf{-8.86} \\
\bottomrule
\end{tabular*}
\end{table}

Protein backbone torsion pairs are modeled on \(SO(2)^2\simeq\T^2\), and RNA torsion profiles on \(SO(2)^7\simeq\T^7\). We compare WKBC with other baselines under the same benchmark data split and report both held-out NLL and geometric diagnostics. Table~\ref{tab:tdm-table1-nll} reports the held-out NLL values.

{
Table~\ref{tab:tdm-table1-nll} uses the likelihood protocol defined in Section~\ref{subsec:experimental-protocol}: WKBC is instantiated from \(\pi^\star\) to the training torsion law, and held-out samples are scored with the TDM-aligned probability-flow evaluator in Appendix~\ref{app:tdm-likelihood}. Table~\ref{tab:torus-results} reports the corresponding two-endpoint bridge diagnostics.
}

As shown in Table~\ref{tab:tdm-table1-nll}, WKBC attains the lowest NLL in every dataset column. Relative to TDM, the absolute NLL reductions are \(0.10\), \(0.20\), \(0.42\), \(0.35\), and \(2.00\) on General, Glycine, Proline, Pre-Pro, and RNA, respectively. The consistent likelihood improvement indicates that the explicit wrapped-kernel construction captures periodic endpoint densities without sacrificing the intrinsic torus geometry. This supports the advantage of WKBC when an explicit kinetic kernel is available, because the kernel provides a direct representation of the bridge on the torus.

\begin{table}[!tbp]
\centering
\caption{Endpoint and path diagnostics for the compact Abelian torsion tasks. $\downarrow$ ($\uparrow$) indicates that lower (higher) values are better.}
\label{tab:torus-results}
\scriptsize
\setlength{\tabcolsep}{2.0pt}
\renewcommand{\arraystretch}{1.08}
\begin{tabular*}{\textwidth}{@{\extracolsep{\fill}}llccccc@{}}
\toprule
Setting & Method & GeoRMSE $\downarrow$ & \makecell{Sinkhorn\\$W_2\downarrow$} & \makecell{Rama\\JSD $\downarrow$} & \makecell{ValidRate\\$\uparrow$} & \makecell{Path energy\\$\downarrow$} \\
\midrule
\multirow{2}{*}{General} & TDM & 1.504 & 0.282 & 0.277 & 0.595 & 2.944 \\
& WKBC & \textbf{1.483} & \textbf{0.221} & \textbf{0.125} & \textbf{0.873} & \textbf{1.634}\\
\midrule
\multirow{2}{*}{Glycine} & TDM & 1.915 & 0.384 & 0.481 & 0.300 & 2.883 \\
& WKBC & \textbf{1.874} & \textbf{0.296} & \textbf{0.243} & \textbf{0.701} & \textbf{1.672} \\
\midrule
\multirow{2}{*}{Proline} & TDM & 1.453 & 0.914 & 0.238 & 0.769 & 5.740 \\
& WKBC & \textbf{1.421} & \textbf{0.221} & \textbf{0.102} & \textbf{0.907} & \textbf{2.395} \\
\midrule
\multirow{2}{*}{Pre-Pro} & TDM & 1.123 & 0.324 & 0.484 & 0.311 & 2.280 \\
& WKBC & \textbf{1.084} & \textbf{0.258} & \textbf{0.219} & \textbf{0.760} & \textbf{1.532} \\
\midrule
\multirow{2}{*}{RNA} & TDM & 0.843 & 1.024 & 0.335 & 0.511 & 8.169 \\
& WKBC & \textbf{0.854} & \textbf{0.627} & \textbf{0.201} & \textbf{0.733} & \textbf{5.134} \\
\bottomrule
\end{tabular*}
\end{table}

Table~\ref{tab:torus-results} separates distributional endpoint fidelity from paired path diagnostics. Geodesic Root-Mean-Square Error (GeoRMSE) is computed only from the pairing file fixed before training and released with the evaluator, which is a paired diagnostic. To ensure a fair comparison, all methods are evaluated using the same reference process, diffusion coefficient, and other relevant experimental settings.

Table~\ref{tab:torus-results} shows that WKBC improves Sinkhorn distance, Ramachandran JSD, ValidRate, and path energy over TDM on all five reported tasks. GeoRMSE is also lower on the four protein subsets; on RNA it is slightly higher, \(0.854\) versus \(0.843\). Taken together, these metrics show that the likelihood gains are accompanied by improved target-distribution fidelity, geometrically valid trajectories, and lower control-induced path cost.

\begin{figure}[!tbp]
\centering
\includegraphics[width=0.96\textwidth,height=0.58\textheight,keepaspectratio]{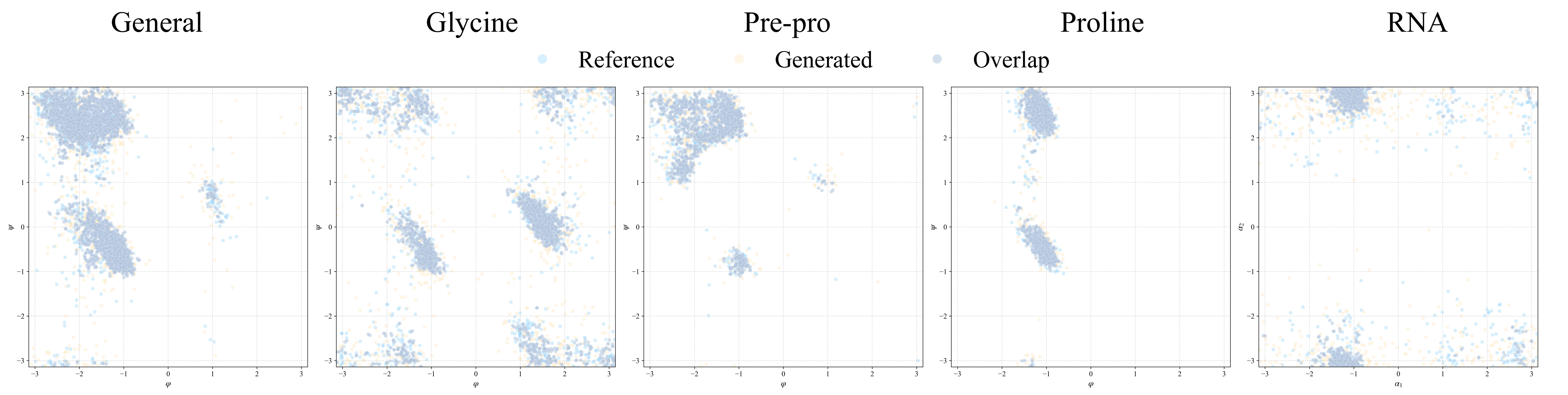}
\caption{Reference and WKBC-generated endpoint samples in wrapped angular coordinates for the General, Glycine, Pre-Pro, Proline, and RNA torsion tasks, together with their overlap. More overlap reflects improved generative performance.
The two-dimensional views visualize the agreement of the dominant periodic support between the generated and reference distributions.}
\label{fig:abelian-scatter-overlap}
\end{figure}

Fig.~\ref{fig:abelian-scatter-overlap} provides the corresponding visualization of the generated and target torsion support. Across the five tasks, the generated samples closely overlap with the reference distributions in the dominant high-density regions and reproduce their principal multimodal structures, indicating that WKBC effectively captures the periodic endpoint geometry and the major modes of the target torsion distributions.

\subsection{Verification against an explicit wrapped bridge}\label{subsec:apbm-wkbc-verification}

To verify whether RCCBM can recover the correct intermediate path distributions and approximate the optimal Doob control provided by WKBC, we perform the numerical consistency experiment. Since higher-order terms in the Magnus expansion generally do not vanish on non-Abelian groups, the RCCBM algorithm is developed for settings in which an explicit wrapped transition kernel is unavailable during training. Accordingly, this experiment applies RCCBM with endpoint calibration and conditional-control matching to recover the same Schrödinger bridge. We then compare RCCBM with the explicit-kernel WKBC reference on a two-dimensional torus task, where such a high-accuracy reference bridge is available. The comparison evaluates source and target discrepancies, intermediate one-time marginals, and forward-control regression against the WKBC optimal control on a common WKBC path pool, which can be  expressed as the following formulation: 

\begin{equation}\label{eq:empirical-forward-control-discrepancy}
    \widehat{\mathcal E}_{\rm F}(\bm u)
    =\frac{1}{N_{\rm path}N_t}
      \sum_{n=1}^{N_{\rm path}}\sum_{k=1}^{N_t}
      \norm{\bm u(t_k,X_{t_k}^{(n)})-
            \bm u_{\rm WKBC}^{\star}(t_k,X_{t_k}^{(n)})}^2,
\end{equation}
where \(\bm u_{\rm WKBC}^{\star}\) is the high-accuracy wrapped-kernel control used as the common numerical teacher. Thus the WKBC row measures the error of its finite numerical realization.

\begin{table}[!tbp]
\centering
\caption{Numerical verification of RCCBM against a high-accuracy WKBC reference bridge on $\mathbb{T}^2$. 
The comparison reports source and target Sinkhorn discrepancies, the mean intermediate-marginal Sinkhorn discrepancy, and the forward-control regression error $\widehat{E}_{F}$. 
Lower values are better.}
\label{tab:apbm-wkbc-verification}
\small
\setlength{\tabcolsep}{3.8pt}
\renewcommand{\arraystretch}{1.08}
\begin{tabular*}{\textwidth}{@{\extracolsep{\fill}}lcccc@{}}
\toprule
Method & Source Sinkhorn $\downarrow$ & Target Sinkhorn $\downarrow$& Mean int. Sinkhorn $\downarrow$& \(\widehat{\mathcal E}_{\rm F}\) $\downarrow$\\
\midrule
WKBC & 0.104 & 0.111 & 0.108 & 0.086 \\
RCCBM & 0.158 & 0.159 & 0.160 & 0.122 \\
\bottomrule
\end{tabular*}
\end{table}

As summarized in Table~\ref{tab:apbm-wkbc-verification}, WKBC achieves lower discrepancies across all four metrics. Nevertheless, RCCBM remains within the same numerical scale as the explicit reference, indicating that it can recover a reasonable approximation to the same Schrödinger bridge. In particular, the learned RCCBM control provides a close approximation to the optimal Doob control induced by WKBC.

\subsection{Compact non-Abelian \(SO(3)\) and \(U(n)\) bridges}
\label{subsec:compact-nonabelian-experiments}

\begin{table}[!tbp]
\centering
\caption{Log-likelihood comparison on the four $SO(3)$ benchmarks under the common TDM evaluation protocol. 
RSGM, TDM, and RCCBM are evaluated using the same likelihood convention; higher values indicate better likelihood fit.}
\label{tab:rcm-nonabelian-results}
\small
\setlength{\tabcolsep}{4.2pt}
\renewcommand{\arraystretch}{1.08}
\begin{tabular*}{\textwidth}{@{\extracolsep{\fill}}lccc@{}}
\toprule
Experiment & RSGM~\citep{debortoli2022riemannian} & TDM~\citep{tdm2025} & RCCBM \\
\midrule
\(SO(3)\)-GMM32    & $0.200$ & $0.292$ & $\mathbf{0.328}$ \\
\(SO(3)\)-GMM64    & $0.185$ & $0.174$ & $\mathbf{0.378}$ \\
\(SO(3)\)-GMM128   & $0.108$ & $0.112$ & $\mathbf{0.200}$ \\
\(SO(3)\)-RingBand & $0.401$ & $0.428$ & $\mathbf{0.834}$ \\
\bottomrule
\end{tabular*}
\end{table}

For \(SO(3)\)-GMM32, the RSGM and TDM log-likelihood values in Table~\ref{tab:rcm-nonabelian-results} are taken directly from the TDM study~\citep{tdm2025}, which reports the RSGM baseline of De Bortoli et al.~\citep{debortoli2022riemannian} alongside TDM. Our \(SO(3)\)-GMM32 benchmark uses the same dataset realization as the released TDM experimental code, so the published values are evaluated on the same underlying rotation distribution. For \(SO(3)\)-GMM64, \(SO(3)\)-GMM128, and \(SO(3)\)-RingBand, the RSGM and TDM entries are our retrained and re-evaluated baselines using the corresponding methods~\citep{debortoli2022riemannian,tdm2025} under the same dataset construction and likelihood protocol used for RCCBM. Accordingly, Table~\ref{tab:rcm-nonabelian-results} distinguishes the published GMM32 baseline values from the three reproduced baseline rows evaluated under the common protocol.

Next, we evaluate RCCBM on multimodal and narrow-support \(SO(3)\) distributions and on physically parameterized \(U(n)\) datasets. For \(SO(3)\), we consider synthetic rotation distributions of varying complexity. \(SO(3)\)-GMM32, \(SO(3)\)-GMM64, and \(SO(3)\)-GMM128 are multimodal mixture distributions with 32, 64, and 128 components, respectively, with increasing numbers of modes and distributional complexity across the three benchmarks. In contrast, \(SO(3)\)-RingBand has a narrow annular support and is designed to assess the ability of the model to capture structured rotation distributions concentrated on a geometrically restricted region of the group. For \(U(n)\), each data point represents the unitary time-evolution operator \(e^{-itH}\) of a quantum system, where \(H\) denotes the Hamiltonian; thus, modeling \(U(n)\) amounts to learning a distribution over quantum dynamical processes. We consider two classes of quantum systems: quantum oscillators with random potentials and the Transverse Field Ising Model (TFIM)~\citep{stinchcombe1973ising}, whose Hamiltonian contains random coupling parameters and transverse-field strengths, thereby yielding an ensemble of unitary evolution operators. Table~\ref{tab:rcm-nonabelian-results} compares the \(SO(3)\) log-likelihood scores of Riemannian score-based generative modeling (RSGM)~\citep{debortoli2022riemannian}, TDM, and RCCBM under the common TDM benchmark convention.

{
Table~\ref{tab:rcm-nonabelian-results} leverages the prior-to-data likelihood protocol defined in Section~\ref{subsec:experimental-protocol}. Following the TDM benchmark~\citep{tdm2025}, each \(SO(3)\)-valued training rotation is paired independently with a standard Gaussian Lie algebra velocity, giving the terminal law \(\rho_{\rm train}(\dd g)\varphi_3(\bm\xi)\,\dd\bm\xi\) on \(SO(3)\times\mathfrak{so}(3)\). After training, fresh trajectories from the frozen benchmark RCCBM controller are used to fit the marginal velocity score required by the probability-flow evaluator in Appendix~\ref{app:tdm-likelihood}. The \(SO(3)\)-GMM64-to-\(SO(3)\)-RingBand experiment below uses the two-observed-endpoint bridge formulation.
}

Under the common TDM evaluation protocol, RCCBM attains the highest log-likelihood in every row of Table~\ref{tab:rcm-nonabelian-results}. Its margins over the second best performer between RSGM and TDM are \(0.036\), \(0.193\), \(0.088\), and \(0.406\) for \(SO(3)\)-GMM32, \(SO(3)\)-GMM64, \(SO(3)\)-GMM128, and \(SO(3)\)-RingBand, respectively. These results demonstrate that RCCBM can more effectively capture complex multimodal structure and highly concentrated geometric support on non-Abelian groups. The extremely small group-constraint errors reported below further indicate that this distributional fidelity is achieved while preserving the intrinsic group structure of \(SO(3)\).

\begin{figure}[!tbp]
\centering
\includegraphics[width=0.98\textwidth,height=0.58\textheight,keepaspectratio]{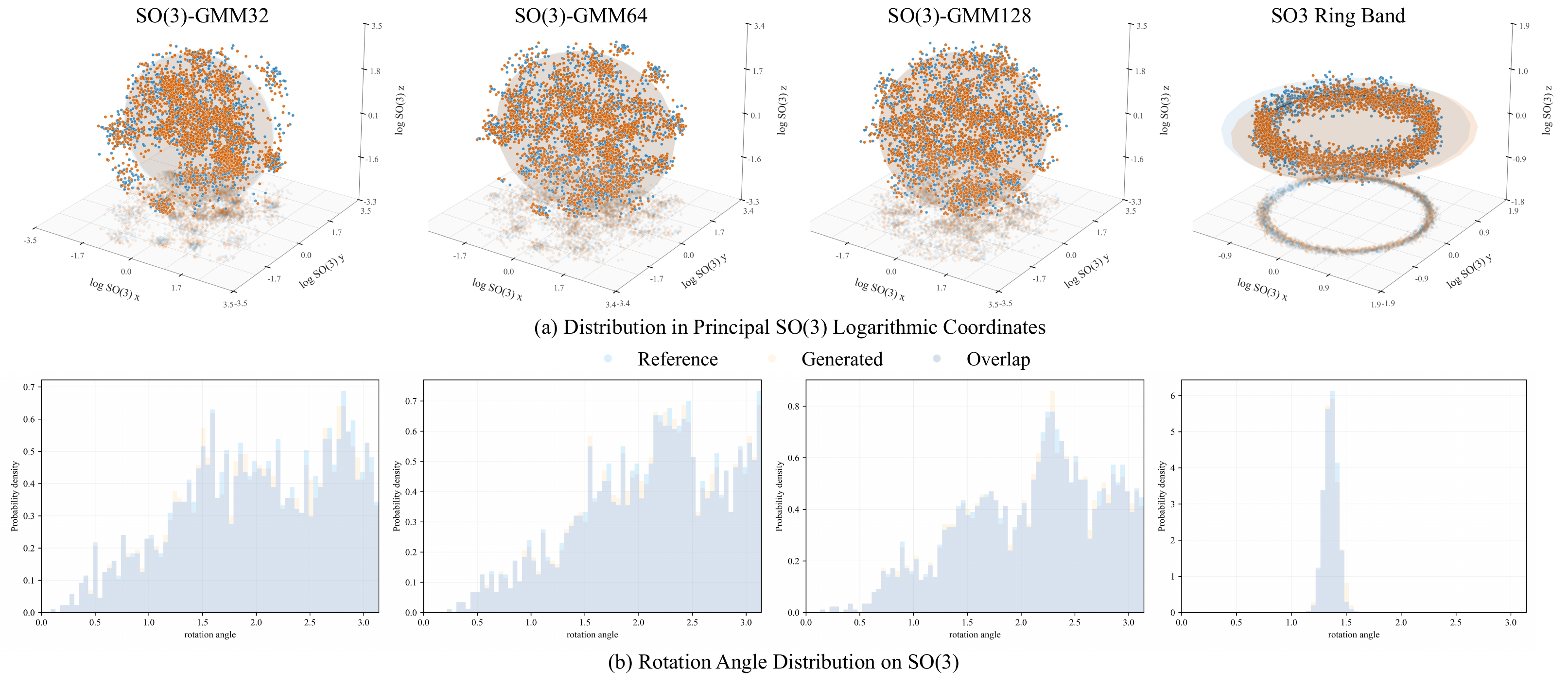}
\caption{Terminal samples for the $SO(3)$-GMM32, $SO(3)$-GMM64, $SO(3)$-GMM128, and $SO(3)$-RingBand benchmarks. 
Top: principal-log-coordinate scatter plots comparing RCCBM-generated samples with the corresponding reference distributions. 
Bottom: rotation-angle marginals for the same benchmarks. 
These complementary views assess the recovery of both multimodal distributions and the concentrated annular support of $SO(3)$-RingBand.}
\label{fig:so3-generated-distribution}
\end{figure}

Fig.~\ref{fig:so3-generated-distribution} provides qualitative principal-log-coordinate projections of the generated terminal rotations. The generated samples reproduce both the multimodal structure of the GMM benchmarks and the concentrated annular support of \(SO(3)\)-RingBand, complementing the likelihood comparison in Table~\ref{tab:rcm-nonabelian-results} and the intrinsic diagnostics reported below.

\paragraph{Distribution-to-distribution transport on \(SO(3)\)}
To further assess the distribution-to-distribution transport capability of RCCBM, we construct a Schr\"odinger bridge from the \(SO(3)\)-GMM64 source distribution to the \(SO(3)\)-RingBand target distribution. Both endpoint distributions are prescribed group-valued laws, and RCCBM learns the stochastic evolution between them under the observed-endpoint bridge formulation.

\begin{figure}[!tbp]
\centering
\includegraphics[width=\textwidth,height=0.58\textheight,keepaspectratio]{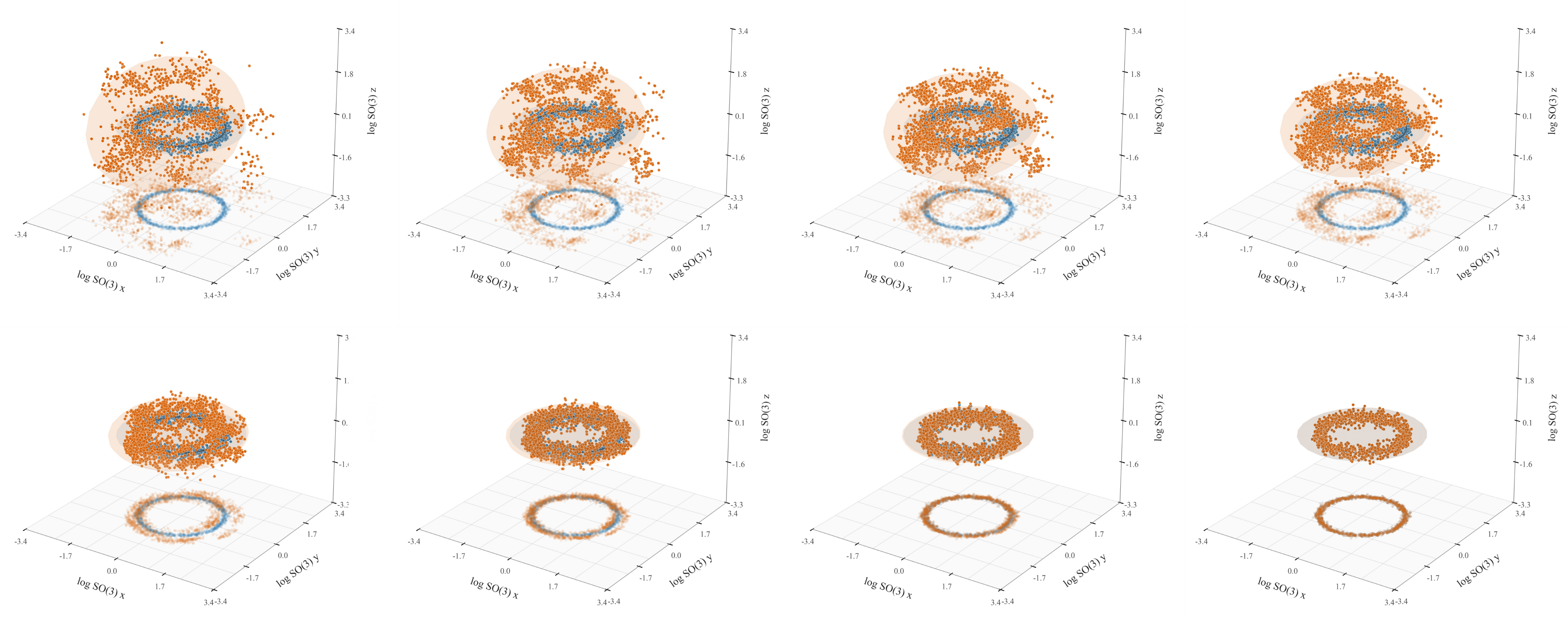}
\caption{Principal-log-coordinate projections of the RCCBM marginal evolution from $SO(3)$-GMM64 to $SO(3)$-RingBand. 
The snapshots illustrate the progressive reorganization of the multimodal source distribution toward the narrow annular support of the prescribed target distribution along the learned Schr\"odinger bridge.}
\label{fig:so3-ringband-dynamics}
\end{figure}

As shown in Fig.~\ref{fig:so3-ringband-dynamics}, the qualitative principal-log-coordinate projections reveal a progressive reorganization of the RCCBM intermediate marginals. Starting from the broad and multimodal support of \(SO(3)\)-GMM64, the generated distribution gradually concentrates toward the narrow annular support of \(SO(3)\)-RingBand, while samples away from the target ring progressively diminish. By the terminal stage, the generated marginal exhibits close qualitative agreement with the geometric structure of the target RingBand distribution.

This experiment shows that RCCBM matches the prescribed terminal law while realizing continuous stochastic transport between complex distributions on the non-Abelian group \(SO(3)\) and representing their intermediate evolution. It therefore highlights the distribution-to-distribution modeling capability of the Lie-group Schr\"odinger bridge beyond the conventional fixed-prior-to-data generative paradigm. \emph{We provide a video to show the complete distributional evolution process in Appendix.}


\paragraph{Theory-facing diagnostics}
Diagnostics aligned with the error components in Theorem~\ref{thm:rcm-conditional-consistency} are reported separately. Endpoint dual and Effective Sample Size (ESS) characterize endpoint calibration; held-out endpoint error and energy assess CondSOC; final relative MSE assesses the post-refinement controller; validation MMD measures terminal distributional agreement; and group error verifies preservation of the matrix-group constraint.

\begin{table}[!tbp]
\centering
\caption{Theory-facing RCCBM diagnostics on $SO(3)$ and $U(4)$, assessing the consistency and numerical stability of the main RCCBM components.}
\label{tab:apbm-diagnostics}
\scriptsize
\setlength{\tabcolsep}{0.7pt}
\renewcommand{\arraystretch}{1.08}
\begin{tabular*}{\textwidth}{@{\extracolsep{\fill}}lcccccccc@{}}
\toprule
Task & \makecell{End.\\dual} & ESS $\uparrow$ & \makecell{Marg.\\err. $\downarrow$} & \makecell{Teach.\\err. $\downarrow$} & \makecell{Teach.\\energy $\downarrow$} & \makecell{Rel.\\MSE $\downarrow$} & \makecell{Val.\\MMD $\downarrow$} & \makecell{Group err.\\$(\times10^{-7})$ $\downarrow$} \\
\midrule
\(SO(3)\)-GMM64 & 6.954 & 1359.900 & 0.001 & 0.010 & 17.930 & 0.157 & 0.014 & 2.880 \\
\(SO(3)\)-RingBand & 5.281 & 886.400 & 0.004 & 0.012 & 17.938 & 0.100 & 0.033 & 2.780 \\
\(U(4)\) oscillator & 6.499 & 592.000 & 0.030 & 0.052 & 100.616 & 0.038 & 0.087 & 5.260 \\
\(U(4)\) TFIM & 7.645 & 893.700 & 0.077 & 0.052 & 99.168 & 0.043 & 0.104 & 5.340 \\
\bottomrule
\end{tabular*}
\end{table}

Table~\ref{tab:apbm-diagnostics} shows that the individual RCCBM modules remain numerically stable and consistent with their intended roles. The relatively high ESS values, together with the small marginal and teacher errors, indicate effective endpoint calibration and reliable CondSOC teacher construction, while the low post-refinement relative MSE and validation MMD further support accurate recovery of the final Markov controller and terminal distribution. Meanwhile, the group-constraint errors remain on the order of \(10^{-7}\) across all tasks, indicating that the distributional fitting and control learning are achieved while preserving the intrinsic matrix-group structure of \(SO(3)\) and \(U(4)\).

\subsection{Protein Conformational Transition Pathway Generation}\label{subsec:protein-transition-pathway-generation}

We evaluate RCCBM on Protein Conformational Transition Pathway Generation, where source and target protein conformations define the endpoint states and the model generates stochastic transition pathways between them. In our formulation, this scientific task is represented as endpoint-conditioned conformational transport on the compact reduced Lie group state space, with observed molecular-dynamics (MD) segments used as reference pathways for evaluation.

\paragraph{Dataset}
mdCATH~\citep{mirarchi2024mdcath} is a large-scale all-atom molecular-dynamics dataset containing 5,398 protein domains from the CATH classification, simulated across independent replicas and temperatures. We use the 320K trajectories and construct the benchmark from domain 1jvmB00. Source-target pairs are selected from persistent metastable conformational regions identified using structural-stability and slow-coordinate analyses, so that each endpoint frame represents a well-characterized conformational basin.

Residue frames are gauge-aligned by Eq.~\eqref{eq:relative-frame-gauge} and reduced by Eq.~\eqref{eq:se3-reduction-map} to \(\G_{\rm red}=SO(3)^{N-1}\times\T^K\). We evaluate \(P=15\) prespecified held-out transitions, each separated by 100 stored MD frames, and compare every observed MD segment with 32 stochastic RCCBM trajectories. Path quality is measured by target-aware path tortuosity and normalized Lie algebra path roughness, where lower values indicate more direct and less locally oscillatory pathways; complete definitions and aggregation rules are given in Appendix~\ref{app:protein-transition-pathway-evaluation}.

\begin{table}[!tbp]
\centering
\caption{Pathway diagnostics for Protein Conformational Transition Pathway Generation on mdCATH domain 1jvmB00, including representative transitions and the aggregate over all $P=15$ held-out cases. 
RCCBM tortuosity and roughness are reported as mean $\pm$ standard deviation over 32 stochastic trajectories for each representative transition.}
\label{tab:protein-transition-pathway-quantitative}
\scriptsize
\setlength{\tabcolsep}{2.0pt}
\renewcommand{\arraystretch}{1.12}
\begin{tabular*}{\textwidth}{@{\extracolsep{\fill}}lcccccc@{}}
\toprule
Pair & \makecell{MD\\tortuosity} & \makecell{RCCBM\\tortuosity} & \makecell{Directness\\gain} & \makecell{MD\\roughness} & \makecell{RCCBM\\roughness} & \makecell{Roughness\\reduction} \\
\midrule
1041 & 51.4241 & $2.3262\pm0.0564$ & 95.48\% & 2.7495 & $0.02409\pm0.00131$ & 99.12\% \\
645  & 65.3731 & $3.1986\pm0.1135$ & 95.11\% & 2.8442 & $0.02159\pm0.00121$ & 99.24\% \\
638  & 63.1760 & $3.3629\pm0.1016$ & 94.68\% & 2.8703 & $0.02153\pm0.00103$ & 99.25\% \\
\midrule
All 15 & -- & -- & \textbf{95.06\%} & -- & -- & \textbf{99.21\%} \\
\bottomrule
\end{tabular*}
\end{table}

Table~\ref{tab:protein-transition-pathway-quantitative} reports representative transitions together with the aggregate result over all \(P=15\) held-out cases. RCCBM achieves an overall directness gain of 95.06\% and a roughness reduction of 99.21\%, with small within-transition variability across the 32 stochastic realizations. These results show that the learned bridge connects the prescribed conformational states through substantially less tortuous and less locally oscillatory pathways than the reference MD segments in the reduced representation. Since the tortuosity metric explicitly penalizes terminal residual, the improvement reflects pathway geometry while retaining the endpoint constraint.

\begin{figure}[!tbp]
\centering
\includegraphics[width=\textwidth,height=0.58\textheight,keepaspectratio]{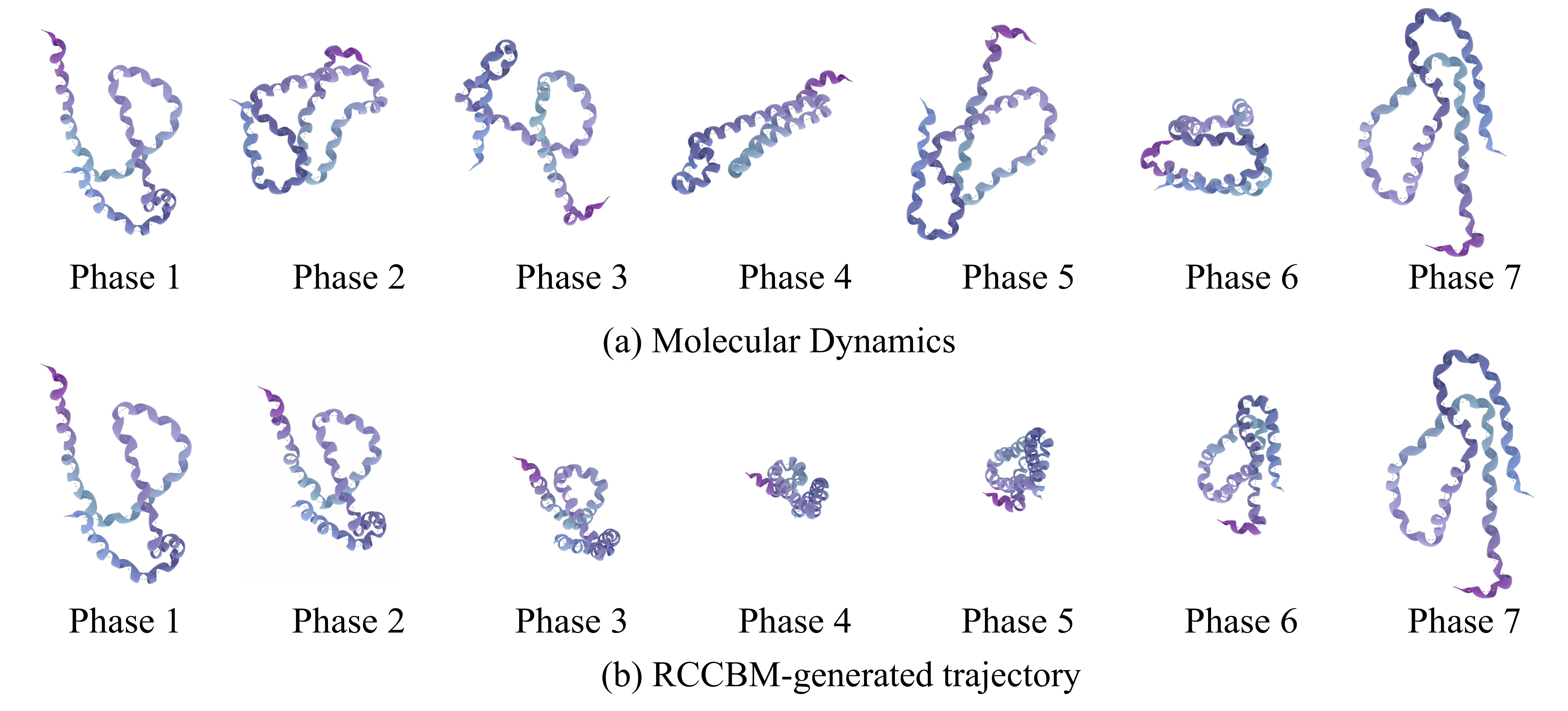}
\caption{Representative protein conformational transition pathways after the same gauge alignment used for quantitative evaluation. 
(a) Reference molecular-dynamics (MD) trajectory and (b) RCCBM-generated trajectory, each visualized at seven representative phases between the prescribed endpoint conformations.}
\label{fig:protein-transition-pathway-comparison}
\end{figure}

Fig.~\ref{fig:protein-transition-pathway-comparison} visualizes representative reference MD and RCCBM transition pathways after the same gauge alignment used for the quantitative evaluation. Compared with the more pronounced inter-phase conformational rearrangements observed along the reference MD trajectory, the RCCBM-generated pathway exhibits a smoother and more continuous conformational evolution, progressively connecting the prescribed source and target conformations. This qualitative behavior is consistent with the reduced path tortuosity and roughness reported above, indicating that RCCBM satisfies the endpoint constraints and generates protein conformational transition pathways that are geometrically more direct and exhibit weaker local oscillations.

\subsection{Ablation and numerical sensitivity analysis}\label{subsec:ablation-studies}

Ablations are evaluated at the validation-selected post-refinement checkpoint. Teacher, control-matching, refinement, and validation diagnostics are taken from the same selected outer iteration, ensuring that component comparisons are made for a consistent final model state.

\subsubsection{WKBC modeling ablation}

\begin{table}[!tbp]
\centering
\caption{WKBC modeling ablations on the General torsion task, comparing the full configuration with removal of endpoint smoothing and replacement of the production time-sampling scheme by uniform sampling. }

\label{tab:wkbc-modeling-ablation}
\scriptsize
\setlength{\tabcolsep}{1.4pt}
\renewcommand{\arraystretch}{1.08}
\begin{tabular*}{\textwidth}{@{\extracolsep{\fill}}lcccccc@{}}
\toprule
Variant & NLL $\downarrow$ & \makecell{Target Sink. $\downarrow$} & \makecell{Energy $\downarrow$} & \makecell{GeoRMSE $\downarrow$} & \makecell{Rama JSD $\downarrow$} & \makecell{ValidRate $\uparrow$} \\
\midrule
Full WKBC & \textbf{0.591} & \textbf{0.221} & \textbf{1.634} & \textbf{1.483} & \textbf{0.125} & \textbf{0.873} \\
w/o endpoint smoothing & 0.605 & 0.244 & 1.665 & 1.507 & 0.126 & 0.868 \\
Uniform time sampling & 0.618 & 0.231 & 1.657 & 1.505 & 0.128 & 0.870 \\
\bottomrule
\end{tabular*}
\end{table}

The General-task WKBC ablations use the same NLL reporting convention as Table~\ref{tab:tdm-table1-nll} and the same endpoint and path diagnostic conventions as Table~\ref{tab:torus-results}. As shown in Table~\ref{tab:wkbc-modeling-ablation}, the full WKBC configuration attains the most favorable value for every metric. Omitting endpoint smoothing or replacing the production time-sampling distribution by uniform sampling increases the principal endpoint and path discrepancies, supporting both design choices in the wrapped-kernel implementation.  



\subsubsection{RCCBM component ablation}

The RCCBM ablations isolate endpoint calibration, conditional-noise resampling, terminal replay, local-scale refinement, distribution refinement, spectral conditioning, and calibrated initial-state sampling. The \(SO(3)\) experiments use 2,048 generated test samples and the \(U(4)\) experiments use 1,024. All variants are evaluated with the fixed production feature map and validation protocol.

\begin{table}[!tbp]
\centering
\caption{Component-wise RCCBM ablations on $SO(3)$-GMM64 and $U(4)$-TFIM under the common validation protocol. 
The variants isolate endpoint calibration, conditional-noise resampling, terminal replay, local-scale and distribution refinement, spectral conditioning, and calibrated initial-state sampling.} 
\label{tab:apbm-ablation}
\scriptsize
\setlength{\tabcolsep}{1.0pt}
\renewcommand{\arraystretch}{1.08}
\begin{tabular*}{\textwidth}{@{\extracolsep{\fill}}llccccc@{}}
\toprule
Task & Variant & ESS $\uparrow$ & \makecell{Teacher err. $\downarrow$} & \makecell{Rel. MSE $\downarrow$} & \makecell{Val. MMD $\downarrow$} & \makecell{Local precision\\radius $\downarrow$} \\
\midrule
\multirow{6}{*}{SO(3)-GMM64}
& Full RCCBM & 1359.9 & 0.010 & 0.157 & 0.014 & 0.070 \\
& \makecell[l]{w/o endpoint\\calibration} & -- & 0.011 & 0.229 & 0.023 & 0.155 \\
& \makecell[l]{Fixed common\\random numbers} & 1359.9 & 0.056 & 0.166 & 0.016 & 0.067 \\
& w/o terminal replay & 1359.9 & 0.014 & 0.237 & 0.026 & 0.066 \\
& w/o local-scale term & 1359.9 & 0.010 & 0.159 & 0.026 & 0.074 \\
& \makecell[l]{w/o distribution\\refinement} & 1359.9 & 0.011 & 0.168 & 0.025 & 0.088 \\
\midrule
\multirow{3}{*}{\(U(4)\)-TFIM}
& Full RCCBM & 893.7 & 0.052 & 0.043 & 0.104 & 0.544 \\
& w/o spectral features & 540.3 & 0.066 & 0.057 & 0.532 & 0.711 \\
& \makecell[l]{w/o calibrated\\initial state} & 810.7 & 0.056 & 0.044 & 0.211 & 0.628 \\
\bottomrule
\end{tabular*}
\end{table}

Table~\ref{tab:apbm-ablation} reports the component-wise RCCBM ablations under the common validation protocol. The ablation results support the complementary roles of the RCCBM components. Fixing the common random numbers markedly increases the held-out teacher error, indicating that stochastic resampling is important for conditional-noise generalization. Omitting endpoint calibration or terminal replay substantially increases the post-refinement relative MSE and validation MMD, while omitting local-scale or distribution refinement also degrades terminal-distribution accuracy. The smaller local precision radius obtained by the fixed-noise or no-replay variants reflects the local nature of this diagnostic, which measures generated-to-target neighborhood proximity; a smaller local radius can occur alongside weaker global distributional agreement or control accuracy. On \(U(4)\)-TFIM, both spectral conditioning and calibrated initial-state sampling improve the distributional and local-precision diagnostics. Taken together, the results indicate complementary and cooperative roles for endpoint correction, conditional-control learning, replay, refinement, and initialization.

\subsubsection{WKBC numerical sensitivity}

To assess whether the numerical settings used by WKBC provide an appropriate balance between computational cost and numerical accuracy, we examine the sensitivity of the General-task results to the lattice radius, the Sinkhorn stopping tolerance, and the number of integration time steps while keeping the trained model fixed. Stability is evaluated using NLL, Target Sinkhorn divergence, and normalized-control energy, with the NLL values following the same reporting convention as Table~\ref{tab:tdm-table1-nll}.

\begin{table}[!tbp]
\centering
\caption{Numerical sensitivity of WKBC on the General torsion task with respect to lattice radius, Sinkhorn stopping tolerance, and the number of integration time steps. 
The $\Delta$ columns report relative changes from the corresponding production settings (lattice radius $2$, Sinkhorn tolerance $10^{-5}$, and $N=800$ time steps).}
\label{tab:wkbc-numerical-refinement}
\scriptsize
\setlength{\tabcolsep}{1.0pt}
\renewcommand{\arraystretch}{1.08}
\begin{tabular*}{\textwidth}{@{\extracolsep{\fill}}llccc!{\vrule width 0.45pt}ccc@{}}
\toprule
Component & Setting & NLL & \makecell{Target\\Sink.} & \makecell{Path\\energy} & $\Delta$NLL & \makecell{$\Delta$Sink.} & \makecell{$\Delta$energy} \\
\midrule
\multirow{3}{*}{Lattice radius}
& 2 & 0.591 & 0.221 & 1.634 & 0.000\% & 0.000\% & 0.000\% \\
& 3 & 0.590 & 0.220 & 1.620 & -0.169\% & -0.452\% & -0.857\% \\
& 4 & 0.588 & 0.223 & 1.627 & -0.508\% & +0.905\% & -0.428\% \\
\midrule
\multirow{3}{*}{Sinkhorn tolerance}
& $10^{-5}$ & 0.591 & 0.221 & 1.634 & 0.000\% & 0.000\% & 0.000\% \\
& $10^{-7}$ & 0.592 & 0.214 & 1.632 & +0.169\% & -3.167\% & -0.122\% \\
& $10^{-9}$ & 0.590 & 0.223 & 1.639 & -0.169\% & +0.905\% & +0.306\% \\
\midrule
\multirow{4}{*}{Time steps $N$}
& 200 & 0.860 & 0.753 & 2.685 & +45.516\% & +240.724\% & +64.321\% \\
& 400 & 0.674 & 0.332 & 1.885 & +14.044\% & +50.226\% & +15.361\% \\
& 600 & 0.593 & 0.230 & 1.642 & +0.338\% & +4.072\% & +0.490\% \\
& 800 & 0.591 & 0.221 & 1.634 & 0.000\% & 0.000\% & 0.000\% \\
\bottomrule
\end{tabular*}
\end{table}

Table~\ref{tab:wkbc-numerical-refinement} shows that further refinement of either numerical parameter yields only marginal and non-uniform changes in the reported metrics. For the integration grid, the additional \(N=600\) evaluation gives an NLL of \(0.593\), only \(0.338\%\) above the production \(N=800\) value \(0.591\), indicating that the likelihood estimate is already close to its production-resolution value by \(N=600\). Increasing the lattice radius from \(2\) to \(3\) reduces NLL, Target Sinkhorn divergence, and path energy by only \(0.169\%\), \(0.452\%\), and \(0.857\%\), respectively. Increasing the radius further to \(4\) produces a \(0.508\%\) reduction in NLL and a \(0.428\%\) reduction in path energy, while Target Sinkhorn divergence increases by \(0.905\%\). Likewise, tightening the Sinkhorn tolerance from \(10^{-5}\) to \(10^{-7}\) changes NLL and path energy by \(+0.169\%\) and \(-0.122\%\), respectively, while Target Sinkhorn divergence decreases by \(3.167\%\); the \(10^{-9}\) setting yields mixed changes across the three metrics. These results indicate that the production settings already operate in a numerically stable regime in which additional lattice expansion or stricter Sinkhorn convergence provides limited and metric-dependent gains. Therefore, the adopted lattice radius and stopping tolerance constitute a practical choice that preserves the generation quality of WKBC while avoiding unnecessary computational overhead.

\section{Limitations and Future Work}\label{sec:limitations}

Although the proposed framework has been validated across several Lie group regimes and scientific datasets, its current empirical evaluation remains concentrated on moderate-dimensional groups and selected molecular and geometric tasks; its behavior in substantially higher-dimensional groups, larger systems, and more heterogeneous scientific settings therefore remains to be characterized. More broadly, the present structure-adapted treatment focuses on compact Abelian groups, compact non-Abelian groups, and rigid frame problems after compact reduction, leaving genuinely noncompact Lie groups and more general geometric state spaces outside the current scope.


These limitations suggest several natural directions for further study. First, broader experiments on higher-dimensional groups, larger molecular systems, and more diverse scientific datasets could clarify the scalability and robustness of the framework. Second, extending the current analysis beyond compact or compact-reduced settings to genuinely noncompact Lie groups, quotient manifolds, and homogeneous spaces would substantially broaden its geometric applicability while retaining the observed-endpoint kinetic Schr\"odinger bridge formulation.

\section{Conclusions}\label{sec:conclusion}

We developed an observed-endpoint Schr\"odinger bridge framework for kinetic dynamics on Lie group state spaces \(\G\times\g\). The formulation imposes endpoint constraints through the observed variables, and the entropy projection determines the conditional law of the latent Lie algebra velocities. Under the endpoint-kernel condition, the resulting bridge admits a two-sided Schr\"odinger factorization and an associated kinetic Doob representation, providing a unified probabilistic foundation for distribution-to-distribution transport on Lie groups.

Building on this formulation, we designed two structure-adapted computational realizations. For compact Abelian groups, WKBC exploits the explicit wrapped kinetic kernel to perform endpoint calibration and recover the corresponding Doob control together with the entropy-optimal initial-state law. For compact non-Abelian groups, where such kernels are generally unavailable, RCCBM combines reciprocal endpoint calibration, mollified finite-energy conditional teachers, importance-corrected control regression, and constrained terminal refinement to learn a single Markov forward controller. The canonical teacher mixture remains within the reciprocal Markov class, which connects the conditional-control learning stage directly to the target Schr\"odinger bridge.

We further established a modular conditional consistency analysis that separates endpoint-calibration, mollification, teacher-approximation, control-regression, initialization, and time-discretization errors in bounded-Lipschitz path distance. This decomposition links the individual computational stages to a single path-law convergence statement and clarifies the role of each approximation through numerical experiments.

Experiments on protein and RNA torsions, \(SO(3)\), \(U(n)\), and protein conformational transition pathways demonstrate the effectiveness of the proposed framework across both compact Abelian and compact non-Abelian settings. In addition, the ablation studies, theory-facing diagnostics, and numerical sensitivity analyses support the roles of the principal WKBC and RCCBM components and confirm that the adopted numerical settings provide a stable and computationally practical realization of the developed bridge constructions.

\bibliographystyle{iclr2027_conference}
\bibliography{main}

@String(NIPS  = {NeurIPS})

@String(ICLR  = {International Conference on Learning Representations})

@String(TIM={IEEE Trans. Instrum. Meas.})

@String(NIPS= {Proc. Int. Conf. Neural Inf. Process. Syst.})

@String(ACLT= {Annu. Conf. Learn. Theory.})

@String(ICLR = {Proc. Int. Conf. Learn. Represent.})

@String(ICML = {Proc. Int. Conf. Mach. Learn.})

@String(AISTS = {Int. Conf. Artif. Intell. Stat.})

@inproceedings{kong2024klmc,
  title={Convergence of kinetic {Langevin} {Monte Carlo} on {Lie} groups},
  author={Kong, Lingkai and Tao, Molei},
  booktitle=ACLT,
  pages={3011--3063},
  year={2024}
}

@inproceedings{falorsi2019relie,
  title={Reparameterizing distributions on {Lie} groups},
  author={Falorsi, Luca and De Haan, Pim and Davidson, Tim R and Forr{\'e}, Patrick},
  booktitle=AISTS,
  pages={3244--3253},
  year={2019}
}

@inproceedings{rezende2020tori,
  title={Normalizing flows on tori and spheres},
  author={Rezende, Danilo Jimenez and Papamakarios, George and Racaniere, S{\'e}bastien and Albergo, Michael and Kanwar, Gurtej and Shanahan, Phiala and Cranmer, Kyle},
  booktitle=ICML,
  pages={8083--8092},
  year={2020}
}

@article{mathieu2020rcnf,
  title={Riemannian continuous normalizing flows},
  author={Mathieu, Emile and Nickel, Maximilian},
  journal=NIPS,
  volume={33},
  pages={2503--2515},
  year={2020}
}

@inproceedings{huang2022rdm,
  author    = {Chin-Wei Huang and Milad Aghajohari and Joey Bose and Prakash Panangaden and Aaron Courville},
  title     = {Riemannian Diffusion Models},
  booktitle = {Advances in Neural Information Processing Systems},
  volume    = {35},
  pages     = {2750--2761},
  year      = {2022}
}

@inproceedings{lou2020manifoldode,
  author    = {A. Lou and D. Lim and I. Katsman and L. Huang and Q. Jiang and S.-N. Lim and C. M. De Sa},
  title     = {Neural manifold ordinary differential equations},
  booktitle = NIPS,
  volume    = {33},
  year      = {2020}
}

@inproceedings{debortoli2022riemannian,
  author    = {V. De Bortoli and E. Mathieu and M. Hutchinson and J. Thornton and Y. W. Teh and A. Doucet},
  title     = {Riemannian score-based generative modelling},
  booktitle = NIPS,
  volume    = {35},
  year      = {2022},
  pages     = {2406--2422}
}

@inproceedings{ho2020ddpm,
  author    = {J. Ho and A. Jain and P. Abbeel},
  title     = {Denoising diffusion probabilistic models},
  booktitle = NIPS,
  volume    = {33},
  year      = {2020},
  pages     = {6840--6851}
}

@inproceedings{song2021sde,
  author    = {Y. Song and J. Sohl-Dickstein and D. P. Kingma and A. Kumar and S. Ermon and B. Poole},
  title     = {Score-based generative modeling through stochastic differential equations},
  booktitle = ICLR,
  year      = {2021}
}

@inproceedings{lipman2023fm,
  author    = {Y. Lipman and R. T. Q. Chen and H. Ben-Hamu and M. Nickel and M. Le},
  title     = {Flow matching for generative modeling},
  booktitle = ICLR,
  year      = {2023}
}

@inproceedings{debortoli2021sb,
  author    = {V. De Bortoli and J. Thornton and J. Heng and A. Doucet},
  title     = {Diffusion {Schr{\"o}dinger} bridge with applications to score-based generative modeling},
  booktitle = NIPS,
  volume    = {34},
  year      = {2021},
  pages     = {17695--17709}
}

@inproceedings{shi2023dsbm,
  author    = {Y. Shi and V. De Bortoli and A. Campbell and A. Doucet},
  title     = {Diffusion {Schr{\"o}dinger} bridge matching},
  booktitle = NIPS,
  volume    = {36},
  year      = {2023},
  pages     = {62183--62223}
}

@inproceedings{liu2024gsbm,
  author    = {G.-H. Liu and Y. Lipman and M. Nickel and B. Karrer and E. A. Theodorou and R. T. Q. Chen},
  title     = {Generalized {Schr{\"o}dinger} bridge matching},
  booktitle = ICLR,
  year      = {2024}
}

@article{schrodinger1931,
  author    = {E. Schr{\"o}dinger},
  title     = {{\"U}ber die Umkehrung der Naturgesetze},
  journal   = {Sitzungsber. Preuss. Akad. Wiss., Phys.-Math. Kl.},
  year      = {1931},
  pages     = {144--153}
}

@incollection{follmer1988,
  author    = {H. F{\"o}llmer},
  title     = {Random fields and diffusion processes},
  booktitle = {{\'E}cole d'{\'E}t{\'e} de Probabilit{\'e}s de Saint-Flour XV--XVII, 1985--1987},
  volume    = {1362},
  year      = {1988},
  pages     = {101--203}
}

@article{leonard2014,
  author    = {C. L{\'e}onard},
  title     = {A survey of the {Schr{\"o}dinger} problem and some of its connections with optimal transport},
  journal   = {Discrete Contin. Dyn. Syst. Ser. A},
  volume    = {34},
  year      = {2014},
  pages     = {1533--1574}
}

@article{chen2021siamreview,
  author    = {Y. Chen and T. T. Georgiou and M. Pavon},
  title     = {Stochastic control liaisons: {Richard Sinkhorn} meets {Gaspard Monge} on a {Schr{\"o}dinger} bridge},
  journal   = {SIAM Rev.},
  volume    = {63},
  year      = {2021},
  pages     = {249--313}
}

@inproceedings{chen2024rfm,
  author    = {R. T. Q. Chen and Y. Lipman},
  title     = {Flow matching on general geometries},
  booktitle = ICLR,
  year      = {2024}
}

@inproceedings{tdm2025,
  author    = {Y. Zhu and T. Chen and L. Kong and E. A. Theodorou and M. Tao},
  title     = {Trivialized momentum facilitates diffusion generative modeling on {Lie} groups},
  booktitle = ICLR,
  year      = {2025}
}

@article{stinchcombe1973ising,
  title={Ising model in a transverse field. I. Basic theory},
  author={Stinchcombe, R. B.},
  journal={Journal of Physics C: Solid State Physics},
  volume={6},
  number={15},
  pages={2459--2483},
  year={1973}
}

@incollection{pavon1991,
  author    = {M. Pavon and A. Wakolbinger},
  title     = {On free energy, stochastic control, and {Schr{\"o}dinger} processes},
  booktitle = {Modeling, Estimation and Control of Systems with Uncertainty},
  year      = {1991},
  pages     = {334--348}
}

@article{liu2025asbs,
  title={Adjoint {Schr{\"o}dinger} bridge sampler},
  author={Liu, Guan-Horng and Choi, Jaemoo and Chen, Yongxin and Miller, Benjamin K and Chen, Ricky TQ},
  journal=NIPS,
  volume={38},
  pages={15673--15708},
  year={2025}
}

@article{pooladian2025plugin,
  author    = {A.-A. Pooladian and J. Niles-Weed},
  title     = {Plug-in estimation of {Schr{\"o}dinger} bridges},
  journal   = {SIAM J. Math. Data Sci.},
  volume    = {7},
  year      = {2025},
  pages     = {1315--1336}
}

@article{chiarini2022kinetic,
  author    = {A. Chiarini and G. Conforti and G. Greco and Z. Ren},
  title     = {Entropic turnpike estimates for the kinetic {Schr{\"o}dinger} problem},
  journal   = {Electron. J. Probab.},
  volume    = {27},
  year      = {2022},
  number    = {131},
  pages     = {1--32}
}

@inproceedings{chen2023dmsb,
  author    = {T. Chen and G.-H. Liu and M. Tao and E. A. Theodorou},
  title     = {Deep momentum multi-marginal {Schr{\"o}dinger} bridge},
  booktitle = NIPS,
  volume    = {36},
  year      = {2023},
  pages     = {57058--57086}
}

@article{peluchetti2023dbmt,
  author    = {S. Peluchetti},
  title     = {Diffusion bridge mixture transports, {Schr{\"o}dinger} bridge problems and generative modeling},
  journal   = {J. Mach. Learn. Res.},
  volume    = {24},
  year      = {2023},
  number    = {374},
  pages     = {1--51}
}

@inproceedings{leach2022so3,
  author    = {A. Leach and S. M. Schmon and M. T. Degiacomi and C. G. Willcocks},
  title     = {Denoising diffusion probabilistic models on {SO(3)} for rotational alignment},
  booktitle = {ICLR 2022 Workshop on Geometrical and Topological Representation Learning},
  year      = {2022}
}

@article{conforti2025kl,
  author    = {G. Conforti and A. Durmus and Gentiloni Silveri, Marta},
  title     = {{KL} convergence guarantees for score diffusion models under minimal data assumptions},
  journal   = {SIAM J. Math. Data Sci.},
  volume    = {7},
  year      = {2025},
  pages     = {86--109}
}

@misc{bertolini2025egsm,
  author    = {M. Bertolini and T. Le and D.-A. Clevert},
  title     = {Generative modeling on {Lie} groups via {Euclidean} generalized score matching},
  year      = {2025},
  note      = {arXiv:2502.02513}
}

@misc{thornton2022rdsb,
  author    = {J. Thornton and M. Hutchinson and E. Mathieu and V. De Bortoli and Y. W. Teh and A. Doucet},
  title     = {Riemannian diffusion {Schr{\"o}dinger} bridge},
  year      = {2022},
  note      = {arXiv:2207.03024}
}

@article{mahmood2026compact,
  author    = {H. Mahmood and A. Halder and A. Akhtar},
  title     = {{Schr{\"o}dinger} bridge over a compact connected {Lie} group},
  journal   = {IEEE Control Syst. Lett.},
  volume    = {10},
  year      = {2026},
  pages     = {1141--1146}
}

@article{mirarchi2024mdcath,
  author    = {A. Mirarchi and T. Giorgino and De Fabritiis, Gianni},
  title     = {{mdCATH}: A large-scale {MD} dataset for data-driven computational biophysics},
  journal   = {Sci. Data},
  volume    = {11},
  year      = {2024},
  pages     = {1299}
}

\appendix

\newpage
\section*{Appendix}
This appendix accompanies the main article and provides complete numerical algorithms, auxiliary assumptions and identities, proofs of all main theoretical results, additional visual diagnostics, and reproducibility details. It first reports the implementation details, then presents the technical results and proofs, and finally presents the additional experiments and the experimental settings used for the main-text results.

\section{Algorithms and Numerical Approximation}\label{app:algorithmic-details}

\subsection{Observed-endpoint WKBC}\label{app:wkbc-implementation}

Choose positive quadrature rules
\begin{equation*}
\nu_{0,M}=\sum_{i=1}^{M_0}\omega_i^0\delta_{y_i^0},
\qquad
\nu_{T,M}=\sum_{j=1}^{M_T}\omega_j^T\delta_{y_j^T},
\end{equation*}
and discrete endpoint masses
\begin{equation*}
\rho_{0,M}=\sum_{i=1}^{M_0}a_i\delta_{y_i^0},
\qquad
\rho_{T,M}=\sum_{j=1}^{M_T}b_j\delta_{y_j^T}.
\end{equation*}
If \(r_{0T}=\dd\mathcal R_{0T}/\dd(\nu_0\otimes\nu_T)\), form
\begin{equation*}
R_{ij}^M
=\frac{r_{0T}(y_i^0,y_j^T)\omega_i^0\omega_j^T}
{\sum_{k,\ell}r_{0T}(y_k^0,y_\ell^T)\omega_k^0\omega_\ell^T}.
\end{equation*}
The discrete Schr\"odinger system is
\begin{equation*}
\Pi_{ij}=f_iR_{ij}^Mg_j,
\qquad
\sum_j\Pi_{ij}=a_i,
\qquad
\sum_i\Pi_{ij}=b_j.
\end{equation*}
Log-domain Sinkhorn updates are used. Positive interpolants \(f_{0,M}\), \(g_{T,M}\) define
\begin{align*}
h_{t,M}(x)
&=\sum_{j=1}^{M_T}q_{T-t}^{\mathcal O_T}(y_j^T\mid x)g_j\omega_j^T,\\
p_{0,M}(x)
&=Z_{0,M}^{-1}q_0(x)f_{0,M}(\mathcal O_0(x))h_{0,M}(x).
\end{align*}

{
For the prior-to-data likelihood specialization, the initial kinetic law is fixed to
\begin{equation}\label{eq:likelihood-haar-gaussian-prior}
\pi^\star(\dd g\,\dd\bm\xi)
=
\nu_{\G}(\dd g)\varphi_d(\bm\xi)\,\dd\bm\xi,
\qquad
\varphi_d(\bm\xi)
=
(2\pi)^{-d/2}\exp\!\left(-\frac12\norm{\bm\xi}^2\right),
\end{equation}
where \(\nu_{\G}\) is a normalized Haar measure.  Thus, the compact group coordinate is Haar-distributed, and the Lie algebra velocity is standard Gaussian, matching the stationary kinetic prior used by TDM~\citep{tdm2025}.  For this benchmark we also initialize the reference law with \(Q_0=\pi^\star\).  Following the TDM benchmark data augmentation, we use the full kinetic state at both benchmark endpoints:
\[
\mathcal O_0^{\rm lik}(g,\bm\xi)
=
\mathcal O_T^{\rm lik}(g,\bm\xi)
=
(g,\bm\xi),
\qquad
P_0^{\rm lik}=\pi^\star,
\qquad
P_T^{\rm lik}(\dd g\,\dd\bm\xi)
=
\rho_{\rm train}(\dd g)\varphi_d(\bm\xi)\,\dd\bm\xi .
\]
Equivalently, every group-valued training datum is paired independently with \(\bm\xi\sim\mathcal N(0,I_d)\).  This specification defines the LL/NLL benchmark.  The scientific observed-endpoint bridges retain their group-valued endpoint observations, with latent velocity conditionals determined by the entropy projection.  The same Sinkhorn scaling applies after replacing both full-state endpoint integrals by positive quadrature or Monte Carlo rules on \(\G\times\g\).

The explicit wrapped \emph{joint} kernel also propagates the calibrated source factor forward.  If \(\{(x_i^0,\omega_i^{0,x})\}_{i=1}^{M_x}\) is a positive quadrature rule for the base measure \(\mu\) on \(\X\), we can define
\begin{equation}\label{eq:wkbc-forward-factor-discrete}
\widehat h_{t,M}(x)
=
\sum_{i=1}^{M_x}
q_{0,t}(x\mid x_i^0)\,
q_0(x_i^0)\,
f_{0,M}(\mathcal O_0(x_i^0))\,
\omega_i^{0,x}.
\end{equation}
With \(p_{t,M}^{\rm WKBC}:=h_{t,M}\widehat h_{t,M}\), consequently,
\begin{equation}\label{eq:wkbc-eval-score-discrete}
\bm s_{{\rm eval},M}^{\rm WKBC}(t,x)
=
\nabla_{\bm\xi}\log h_{t,M}(x)
+
\nabla_{\bm\xi}\log\widehat h_{t,M}(x)
=
\nabla_{\bm\xi}\log p_{t,M}^{\rm WKBC}(x).
\end{equation}
Eq.~\eqref{eq:wkbc-forward-factor-discrete} evaluates the forward factor deterministically from the calibrated source factor returned by Sinkhorn.
}

\paragraph{Sufficient approximation conditions}
For each \(\tau>0\), assume: (i) the truncated wrapped kernels and their \(\bm\xi\)-gradients converge uniformly on \([0,T-\tau]\) after lattice enlargement; (ii) the quadrature rules converge for the uniformly bounded equicontinuous kernel families; (iii) the discrete kernel matrices remain between common positive bounds; (iv) Sinkhorn residuals vanish; (v) the positive interpolants converge in the endpoint norms induced by the kernel operators; and (vi) the self-normalized bridge endpoint weights used by the regression proposal satisfy the usual law-of-large-numbers consistency. Then projective contraction stability gives uniform convergence of the normalized scaling factors, while dominated differentiation gives convergence of \(h_{t,M}\) and its logarithmic gradient on preterminal strips. At the population level the bridge-weighted regression proposal is exactly the Schr\"odinger-bridge occupation law, so no separate domination constant between an off-policy regression measure and the controlled occupation measure is required. These assumptions provide one sufficient set of approximation conditions.

\begin{algorithm}[!htbp]
\caption{Observed-Endpoint WKBC}
\label{alg:wkbc}
\begin{algorithmic}[1]
\Require endpoint quadratures \(\{(y_i^0,a_i)\}\), \(\{(y_j^T,b_j)\}\), wrapped reference kernels, tolerance \(\epsilon_{\rm sk}\), regression model \(\bm s_\theta\)
\State Form the discrete reference matrix \(R^M=(R_{ij}^M)\) and initialize \(f_i,g_j>0\)
\Repeat
    \State \(f_i\gets a_i/\sum_j R_{ij}^Mg_j\) for all \(i\)
    \State \(g_j\gets b_j/\sum_i f_iR_{ij}^M\) for all \(j\)
\Until{both marginal residuals are below \(\epsilon_{\rm sk}\)}
\State Interpolate \(f_{0,M},g_{T,M}\) and compute \(h_{t,M}(x)=\sum_j q_{T-t}^{\mathcal O_T}(y_j^T\mid x)g_j\omega_j^T\)
\For{each regression minibatch}
    \State Sample \(t\sim\mathrm{Unif}[0,T]\) and a reference-path tuple \((X_0,X_t,Y_T)\)
    \State Set \(w_{\rm br}\propto f_{0,M}(\mathcal O_0(X_0))g_{T,M}(Y_T)\) and \(\bm S\gets\nabla_{\bm\xi}\log q_{T-t}^{\mathcal O_T}(Y_T\mid X_t)\)
    \State Update \(\theta\) by the self-normalized bridge-weighted loss \(\sum w_{\rm br}\|\bm s_\theta(t,X_t)-\bm S\|^2\)
\EndFor
\State Define \(p_{0,M}\propto q_0(f_{0,M}\circ\mathcal O_0)h_{0,M}\) and sample \(X_0\sim p_{0,M}\)
\State Set \(\bm u_\theta=\sqrt{2\gamma}\,\bm s_\theta\) and propagate the controlled kinetic dynamics
\State {For the prior-to-data likelihood benchmark, use \(P_0=\pi^\star\), compute \(\widehat h_{t,M}\) from Eq.~\eqref{eq:wkbc-forward-factor-discrete}, and set \(\bm s_{\rm eval}^{\rm WKBC}\) by Eq.~\eqref{eq:wkbc-eval-score-discrete}}
\State {Evaluate held-out LL/NLL with the probability-flow routine in Algorithm~\ref{alg:tdm-aligned-likelihood}}
\Ensure bridge paths, scaling-factor interpolants, and the calibrated initial-state sampler
\end{algorithmic}
\end{algorithm}

\subsection{RCCBM and the reduction wrapper}\label{app:lsapbm-algorithms}

For a learned endpoint pair \(\widehat\beta_N\) and mollifier \(\kappa_{\varepsilon_N}\), the population teacher target is the canonical law \(\Lambda^{\widehat\beta_N,\varepsilon_N}\) in Eq.~(3.18). It can be sampled directly when \(Z_{z,\varepsilon}\) is available, or by an importance proposal over \((x_0,y)\) with the exact Radon--Nikodym correction. Pair balancing, terminal-window oversampling, and teacher-quality proposals do not change the population risk after this correction.

\paragraph{Time-sampling density and its lower bound}
To retain a strictly positive density over the full horizon while allocating a prescribed fraction of samples to the terminal window, RCCBM uses the two-window mixture
\begin{equation}\label{eq:rccbm-production-time-density}
\lambda_{\rm bm}(t)
=\frac{1-r_{\rm term}}{a_{\rm term}T}
\mathbf 1_{[0,a_{\rm term}T)}(t)
+\frac{r_{\rm term}}{(1-a_{\rm term})T}
\mathbf 1_{[a_{\rm term}T,T]}(t),
\qquad 0\le t\le T.
\end{equation}
Equivalently, with probability \(1-r_{\rm term}\) we draw \(t\sim\mathrm{Unif}[0,a_{\rm term}T)\), and with probability \(r_{\rm term}\) we draw \(t\sim\mathrm{Unif}[a_{\rm term}T,T]\). Thus \(r_{\rm term}\) is exactly the terminal-window sampling fraction reported in Table~\ref{tab:essential-training-settings}, and
\begin{equation}\label{eq:rccbm-time-density-floor}
\underline\lambda
:=\inf_{0\le t\le T}\lambda_{\rm bm}(t)
=\frac{1}{T}
\min\!\left\{\frac{1-r_{\rm term}}{a_{\rm term}},
\frac{r_{\rm term}}{1-a_{\rm term}}\right\}.
\end{equation}
The quantity \(\underline\lambda\) is therefore not an independently tuned hyperparameter. The production values \(a_{\rm term}=0.75\), \(r_{\rm term}=0.55\) on \(SO(3)\), and \(r_{\rm term}=0.65\) on \(U(n)\) give \(\underline\lambda T=0.60\) and \(7/15\approx0.467\), respectively; these values are recorded in Table~\ref{tab:essential-training-settings}.

{
We use two feature maps with distinct roles.  The map \(\Phi_{\G}^{\rm id}\) is the fixed continuous injective identification map used by the MMD term, while \(\Phi_{\G}^{\rm net}\) is the feature vector supplied to the controller and may additionally contain spectral or local descriptors that support finite-sample optimization.  The explicit production choices and the injectivity argument are given in \cref{app:feature-identification-details}.
}

The direct control is a residual network
{
\begin{equation*}
\bm h_0=W_{\rm in}[\Phi_\G^{\rm net}(g),\bm\xi,\tau(t)],
\qquad
\bm h_{b+1}=\frac{\bm h_b+F_b(\bm h_b)}{\sqrt2},
\qquad
\bm u_\theta=W_{\rm out}\bm h_B.
\end{equation*}
}
The output clip \(B_N\) is a sieve parameter and increases with sample size; its tail bias is included in \(\varepsilon_{{\rm clip},N}\).

Fresh-rollout refinement uses
{
\begin{align*}
\mathcal L_{\rm dist}
&=w_M\MMD_k^2\!\left((\Phi_{\G}^{\rm id})_{\#}\widehat\rho_T,
(\Phi_{\G}^{\rm id})_{\#}\rho_T\right)
+w_\mu\norm{\widehat{\bm\mu}_g-\widehat{\bm\mu}_t}^2
+w_\Sigma\norm{\widehat{\bm\Sigma}_g-\widehat{\bm\Sigma}_t}_{\rm F}^2\\
&\quad+w_{\rm SW}\SW_2^2+w_{\rm loc}\mathcal L_{\rm loc}
+\frac{\eta_E}{2}\E\norm{\bm u_\theta}^2.
\end{align*}
}
The local term acts as a finite-sample neighborhood regularizer.

\begin{algorithm}[!htbp]
\caption{Reciprocal Conditional-Control Bridge Matching}
\label{alg:apbm-training}
\begin{algorithmic}[1]
\Require reference law \(Q\), endpoint data, centered endpoint networks \((\beta_0,\beta_T)\), direct control \(\bm u_\theta\), mollifier schedule \(\varepsilon_N\), Lie group integrator
\State Draw a reference endpoint pool and compute \(\widehat\beta_N\gets\arg\max_{\beta\in\mathcal B_N}\widehat{\mathscr D}_N(\beta)\)
\State Build \(g_{T,\varepsilon_N}^{\widehat\beta_N}\), self-normalized initial weights, and \(\widehat P_{0,N}^{\varepsilon_N}\)
\For{each outer epoch}
    \State Sample conditions \(z=(x_0,y)\) from a proposal targeting \(\Lambda^{\widehat\beta_N,\varepsilon_N}\)
    \State Solve the mollified CondSOC problem for each \(z\), sample replay times \(t\sim\lambda_{\rm bm}\) from Eq.~\eqref{eq:rccbm-production-time-density}, and generate independent-noise teacher tuples \((t,X_t,\bm v_t)\)
    \State Add teacher tuples to replay and attach the exact proposal-to-target importance weights
    \State Update \(\theta\) by \(\sum_j\widetilde\omega_j\|\bm u_\theta(t_j,X_j)-\bm v_j\|^2\)
    \State Save \(\theta_{\rm pre}\); obtain \(\theta_{\rm cand}\) from fresh-rollout distribution refinement
    \If{\(\widehat{\mathcal L}_{\rm match}(\theta_{\rm cand})\leq
       \widehat{\mathcal L}_{\rm match}(\theta_{\rm pre})+\delta_{\rm ref}\)}
        \State \(\theta\gets\theta_{\rm cand}\)
    \Else
        \State \(\theta\gets\theta_{\rm pre}\)
    \EndIf
\EndFor
\State Sample \(X_0\sim\widehat P_{0,N}^{\varepsilon_N}\) and generate with \(\bm u_\theta\) using the Lie group integrator
\Ensure learned forward controller, bridge samples, and separated calibration/teacher/control diagnostics
\end{algorithmic}
\end{algorithm}

\begin{algorithm}[!htbp]
\caption{Compact reduction wrapper for \(SE(3)^N\) data}
\label{alg:compact-reduced-apbm}
\begin{algorithmic}[1]
\Require rigid frame endpoints \(h_{1:N}\), reduction \(\mathcal R_{\rm red}\), reconstruction \(F\), RCCBM routine
\State Remove global rigid motion: \(\widetilde h_i\gets h_1^{-1}h_i\), \(i=2,\ldots,N\)
\State Extract reduced endpoints \(z\gets\mathcal R_{\rm red}(h_{1:N})\in\G_{\rm red}\)
\State Train RCCBM on the kinetic state space \(\G_{\rm red}\times\g_{\rm red}\) and generate reduced paths
\For{each generated reduced state \(z_t\)}
    \State Reconstruct Cartesian coordinates with \(F(z_t)\)
\EndFor
\Ensure reduced bridge paths and deterministic reconstructions, with diagnostics reported separately
\end{algorithmic}
\end{algorithm}

\subsection{Exact-OU splitting}\label{app:apbm-sampling-details}

For frozen \(\bm u_k=\bm u_\theta(t_k,g_k,\bm\xi_k)\), set
\begin{equation}\label{eq:rcm-exact-ou-coefficients}
a=e^{-\gamma\Delta t},
\qquad
c=\frac{\alpha(1-a)}{\gamma},
\qquad
\sigma=\alpha\sqrt{\frac{1-e^{-2\gamma\Delta t}}{2\gamma}}.
\end{equation}
The step is
\begin{equation*}
\bm\xi_{k+1}=a\bm\xi_k+c\bm u_k+\sigma\bm\varepsilon_k,
\qquad
g_{k+1}=g_k\Exp_\G(\Delta t\,\bm\xi_{k+1}),
\qquad
\bm\varepsilon_k\sim\mathcal N(0,I).
\end{equation*}

\begin{proposition}[Geometry and path-space discretization]\label{prop:rcm-exact-ou-geometry}
Assume the Ornstein--Uhlenbeck specialization \(\bm b(t,g,\bm\xi)=-\gamma\bm\xi\) with \(\gamma>0\). For \(\G=SO(3)\) or \(\G=U(n)\), take the metric \(d_\G\) used in the bounded-Lipschitz path distance to be the standard bi-invariant Riemannian distance induced by an \(\operatorname{Ad}\)-invariant inner product on \(\g\). Let the feedback control \(\bm u:[0,T]\times\G\times\g\to\g\) satisfy, for constants \(L_u,K_u<\infty\),
\begin{align*}
\norm{\bm u(t,g,\bm\xi)-\bm u(s,g',\bm\eta)}
&\le L_u\bigl(|t-s|+d_\G(g,g')+\norm{\bm\xi-\bm\eta}\bigr),\\
\norm{\bm u(t,g,\bm\xi)}
&\le K_u\bigl(1+\norm{\bm\xi}\bigr),
\end{align*}
and assume \(\E\norm{\bm\xi_0}^4<\infty\). Then the update above preserves \(SO(3)\) for skew-symmetric algebra coordinates and \(U(n)\) for skew-Hermitian coordinates, and the velocity update is exact for the frozen-control OU substep. If \(P_{\Delta t}^{\bm u}\) denotes the law of the continuous piecewise OU/exponential interpolation defined in the proof and \(P^{\bm u}\) is the exact controlled path law with the same initial law, then
\begin{equation}\label{eq:rccbm-splitting-error}
d_{\rm BL}(P_{\Delta t}^{\bm u},P^{\bm u})
\leq C_T\Delta t^{1/2}.
\end{equation}
The constant \(C_T\) depends only on \(T,\gamma,\alpha,L_u,K_u\), the fourth moment of the initial velocity, and the fixed group metric. In particular, it is uniform over any control sieve with common Lipschitz and linear-growth bounds.
\end{proposition}

\subsection{{TDM-aligned likelihood evaluation}}\label{app:tdm-likelihood}

{
The likelihood tables use the prior-to-data benchmark specified in Appendix~\ref{app:wkbc-implementation}.  Its source is the Haar-Gaussian law \(\pi^\star\) in Eq.~\eqref{eq:likelihood-haar-gaussian-prior}, and its terminal law is the full-state augmentation \(\rho_{\rm train}(\dd g)\varphi_d(\bm\xi)\,\dd\bm\xi\).  Each group-valued training datum is paired independently with a standard Gaussian Lie algebra velocity, following the TDM benchmark augmentation~\citep{tdm2025}.  The main-text \(SO(3)\)-GMM64-to-\(SO(3)\)-RingBand experiment uses the two-prescribed-endpoint scientific bridge protocol.  The likelihood construction below follows the intrinsic probability-flow change-of-variables principle used by TDM~\citep{tdm2025}.

Let \(P_t\) be the one-time law of a benchmark WKBC or RCCBM generator on \(\X=\G\times\g\), with density \(p_t\) relative to normalized Haar measure times Lebesgue measure.  Under the Ornstein--Uhlenbeck specialization \(\alpha=\sqrt{2\gamma}\), write its stochastic dynamics as
\begin{equation}\label{eq:likelihood-controlled-sde}
\dd g_t=T_eL_{g_t}(\bm\xi_t)\,\dd t,
\qquad
\dd\bm\xi_t
=
\bigl[-\gamma\bm\xi_t+\alpha\bm u_t(g_t,\bm\xi_t)\bigr]\dd t
+\alpha\,\dd\mathbf W_t.
\end{equation}
If
\begin{equation}\label{eq:likelihood-full-score}
\bm s_t(x)=\nabla_{\bm\xi}\log p_t(x),
\end{equation}
then the associated probability-flow vector field is
\begin{equation}\label{eq:likelihood-pf-vector-field}
\bm F_t(g,\bm\xi)
=
\left(
T_eL_g(\bm\xi),\,
-\gamma\bm\xi+\alpha\bm u_t(g,\bm\xi)
-\frac{\alpha^2}{2}\bm s_t(g,\bm\xi)
\right).
\end{equation}

\paragraph{Why Eq.~\eqref{eq:likelihood-pf-vector-field} has the correct marginals}
The Fokker--Planck equation of Eq.~\eqref{eq:likelihood-controlled-sde} is
\begin{align}
\partial_t p_t
&=
-\Div_{\G}\!\left(p_tT_eL_g(\bm\xi)\right)
-\Div_{\bm\xi}\!\left(p_t[-\gamma\bm\xi+\alpha\bm u_t]\right)
+\frac{\alpha^2}{2}\Delta_{\bm\xi}p_t \notag\\
&=
-\Div_{\G}\!\left(p_tT_eL_g(\bm\xi)\right)
-\Div_{\bm\xi}\!\left(
p_t\left[-\gamma\bm\xi+\alpha\bm u_t-\frac{\alpha^2}{2}\bm s_t\right]\right),
\label{eq:likelihood-continuity}
\end{align}
where the second equality uses
\(\Delta_{\bm\xi}p_t=\Div_{\bm\xi}(p_t\nabla_{\bm\xi}\log p_t)\).
Hence Eq.~\eqref{eq:likelihood-continuity} is exactly the continuity equation
\[
\partial_t p_t+\Div_{\mu}(p_t\bm F_t)=0.
\]
Therefore the ODE \(\dot Z_t=\bm F_t(Z_t)\) has the same one-time marginals as the stochastic benchmark generator whenever the score in Eq.~\eqref{eq:likelihood-full-score} is exact.

For WKBC, the full score follows directly from the two propagated factors:
\begin{equation}\label{eq:likelihood-wkbc-exact-score}
\bm s_t^{\rm WKBC}
=
\nabla_{\bm\xi}\log h_t
+
\nabla_{\bm\xi}\log\widehat h_t,
\end{equation}
with the numerical realization given by Eq.~\eqref{eq:wkbc-eval-score-discrete}.  This follows directly from \(p_t^{\rm SB}=h_t\widehat h_t\).

\paragraph{RCCBM marginal score for likelihood evaluation}
RCCBM generation uses the learned forward controller \(\bm u_\theta\).  Probability-flow likelihood evaluation additionally requires the full marginal velocity score.  We freeze the final prior-to-data RCCBM controller, generate fresh trajectories from \(P_0=\pi^\star\), and fit an auxiliary network \(\bm s_\psi(t,g,\bm\xi)\) on these trajectories.  With a time weight \(\lambda(t)>0\), the population implicit score-matching objective is
\begin{equation}\label{eq:rccbm-eval-ism}
\mathcal J_{\rm eval}(\psi)
=
\int_0^T
\lambda(t)\,
\E_{P_t^\theta}
\left[
\frac12\norm{\bm s_\psi(t,X_t)}^2
+
\Div_{\bm\xi}\bm s_\psi(t,X_t)
\right]\dd t.
\end{equation}
Assume \(p_t^\theta\) is positive and \(C^1\) in \(\bm\xi\), the score and candidate fields are square integrable, and the boundary term in the \(\bm\xi\)-integration by parts vanishes.  Writing
\(\bm s_t^\star=\nabla_{\bm\xi}\log p_t^\theta\), integration by parts yields
\[
\E_{P_t^\theta}\Div_{\bm\xi}\bm s_\psi
=
-\E_{P_t^\theta}
\inner{\bm s_\psi}{\bm s_t^\star}.
\]
Consequently,
\begin{equation}\label{eq:rccbm-eval-ism-identity}
\mathcal J_{\rm eval}(\psi)-\mathcal J_{\rm eval}(\bm s^\star)
=
\frac12\int_0^T
\lambda(t)\,
\E_{P_t^\theta}
\norm{\bm s_\psi(t,X_t)-\bm s_t^\star(X_t)}^2\,\dd t.
\end{equation}
Thus the population minimizer of Eq.~\eqref{eq:rccbm-eval-ism} is precisely the velocity score required by the probability-flow ODE.  The fit is performed after the final RCCBM controller and checkpoint are fixed, and the resulting \(\bm s_\psi\) supplies the velocity score used in held-out likelihood evaluation.

\begin{algorithm}[!htbp]
\caption{RCCBM likelihood-score fitting}
\label{alg:rccbm-eval-score}
\begin{algorithmic}[1]
\Require frozen prior-to-data RCCBM controller \(\bm u_\theta\), prior \(\pi^\star\), evaluation score \(\bm s_\psi\), time sampler
\For{each evaluation-score minibatch}
    \State Sample \(X_0\sim\pi^\star\) and generate a fresh RCCBM trajectory with the frozen \(\bm u_\theta\)
    \State Sample \(t\) and collect \(X_t=(g_t,\bm\xi_t)\)
    \State Update \(\psi\) using a Monte Carlo estimate of Eq.~\eqref{eq:rccbm-eval-ism}
\EndFor
\State Freeze \(\psi\) and set \(\bm s_{\rm eval}^{\rm RCCBM}=\bm s_\psi\)
\Ensure marginal-score estimator for held-out likelihood evaluation
\end{algorithmic}
\end{algorithm}

\paragraph{Intrinsic change of variables and divergence}
TDM~\citep{tdm2025} uses the intrinsic instantaneous change-of-variables identity
\begin{equation}\label{eq:likelihood-instantaneous-cov}
\frac{\dd}{\dd t}\log p_t(Z_t)
=
-\Div_{\mu}\bm F_t(Z_t).
\end{equation}
For the compact groups considered here, the left-invariant group field \(T_eL_g(\bm\xi)\) is divergence free with respect to Haar measure.  Hence
\begin{equation}\label{eq:likelihood-divergence}
\Div_{\mu}\bm F_t
=
-\gamma d
+
\Div_{\bm\xi}
\left(
\alpha\bm u_t-\frac{\alpha^2}{2}\bm s_t
\right).
\end{equation}
This is the direct analogue of the Lie-group NLL formula in TDM, with the additional controlled-drift term required by the bridge generator.  Following the stochastic-trace convention used in the TDM-style evaluator, one Gaussian probe \(\bm v\sim\mathcal N(0,I_d)\) is drawn for each likelihood trajectory and held fixed along its time integration:
\begin{equation}\label{eq:likelihood-hutchinson}
\widehat D_t
=
-\gamma d
+
\bm v^\top
\nabla_{\bm\xi}
\left(
\alpha\bm u_t-\frac{\alpha^2}{2}\bm s_t
\right)
\bm v .
\end{equation}
Conditioned on \(Z_t\), Eq.~\eqref{eq:likelihood-hutchinson} is an unbiased trace estimator of Eq.~\eqref{eq:likelihood-divergence}.

Let \(z_T=(g,\bm\zeta)\), where the auxiliary terminal velocity is drawn independently as
\(\bm\zeta\sim\mathcal N(0,I_d)\).  Integrating the probability-flow ODE backward from \(z_T\) to \(z_0\), or equivalently evaluating the forward characteristic joining them, gives
\begin{equation}\label{eq:likelihood-logjoint}
\log p_T(z_T)
=
\log\pi^\star(z_0)
-
\int_0^T
\Div_{\mu}\bm F_t(Z_t)\,\dd t .
\end{equation}
Because \(\pi^\star\) has unit density in the group coordinate relative to normalized Haar measure,
\[
\log\pi^\star(z_0)=\log\varphi_d(\bm\xi_0).
\]
Under the full-state likelihood target,
\[
p_T^{\rm lik}(g,\bm\xi)
=
p_{T,\G}^{\rm lik}(g)\varphi_d(\bm\xi),
\]
so the auxiliary Gaussian factor can be removed directly.  Following the TDM benchmark convention, the reported held-out group score is therefore
\begin{equation}\label{eq:likelihood-benchmark-score}
\widehat\ell(g)
=
\log p_T^{\rm lik}(g,\bm\zeta)
-
\log\varphi_d(\bm\zeta)
-
c_{\G},
\end{equation}
where \(\bm\zeta\sim\mathcal N(0,I_d)\) is the independently augmented terminal velocity and \(c_{\G}\) converts normalized Haar density to the benchmark volume convention.  For the experiments in the main text,
\begin{equation}\label{eq:likelihood-volume-offset}
c_{\T^d}=d\log(2\pi),
\qquad
c_{SO(3)}=\log(8\pi^2).
\end{equation}
The main tables use one independent auxiliary Gaussian draw per held-out sample, following the common benchmark evaluator.  The reported quantities are
\begin{equation}\label{eq:likelihood-ll-nll}
\mathrm{LL}
=
\frac1{N_{\rm test}}\sum_{i=1}^{N_{\rm test}}\widehat\ell(g_i),
\qquad
\mathrm{NLL}=-\mathrm{LL}.
\end{equation}
Accordingly, main-text Table~2 reports the negative sign of the same benchmark score for the torus tasks, while the \(SO(3)\) table reports its log-likelihood sign.

\begin{algorithm}[!htbp]
\caption{TDM-aligned held-out LL/NLL evaluation}
\label{alg:tdm-aligned-likelihood}
\begin{algorithmic}[1]
\Require Frozen benchmark generator \(\bm u\), evaluation score \(\bm s_{\rm eval}\), held-out group samples \(\{g_i\}_{i=1}^{N_{\rm test}}\), prior \(\pi^\star\), probability-flow grid
\For{\(i=1,\ldots,N_{\rm test}\)}
    \State Draw \(\bm\zeta_i\sim\mathcal N(0,I_d)\) and a fixed trace probe \(\bm v_i\sim\mathcal N(0,I_d)\)
    \State Set \(z_{i,T}=(g_i,\bm\zeta_i)\) and integrate Eq.~\eqref{eq:likelihood-pf-vector-field} backward to \(z_{i,0}\)
    \State Along the recovered characteristic, accumulate \(\widehat I_i\approx\int_0^T\widehat D_t\,\dd t\) using Eq.~\eqref{eq:likelihood-hutchinson}
    \State \(\widehat\ell_i\gets\log\pi^\star(z_{i,0})-\widehat I_i-\log\varphi_d(\bm\zeta_i)-c_{\G}\)
\EndFor
\State \(\mathrm{LL}\gets N_{\rm test}^{-1}\sum_i\widehat\ell_i\), \(\mathrm{NLL}\gets-\mathrm{LL}\)
\Ensure held-out LL and NLL under the common benchmark protocol
\end{algorithmic}
\end{algorithm}

For WKBC we use \(\bm s_{\rm eval}=\bm s_{\rm eval}^{\rm WKBC}\) from Eq.~\eqref{eq:wkbc-eval-score-discrete}; for RCCBM we use the frozen \(\bm s_\psi\) returned by Algorithm~\ref{alg:rccbm-eval-score}.  The full velocity score is obtained by explicit two-sided wrapped-kernel propagation for WKBC and by post hoc implicit score matching on fresh frozen-generator trajectories for RCCBM.
}

\section{Technical Details Supporting the Main Text}\label{app:technical-theory}

\subsection{Endpoint-kernel verification and latent-state lift}\label{app:core-bridge-details}

\begin{remark}\label{rem:endpoint-kernel-motivation}
The lower bound in Assumption~3.1 excludes inaccessible endpoint pairs, while the upper bound controls the endpoint likelihood ratio. These uniform bounds are used for Birkhoff--Hilbert contraction and \(L^\infty\) control of the scaling factors; they are stronger than the minimal entropy conditions used in general Schr\"odinger existence theory. For observations restricted to the group coordinate on a compact group, the velocity coordinate is integrated out before the condition is applied.
\end{remark}

\begin{lemma}[Compact group-marginal verification]\label{cor:compact-existence}\label{prop:compact-observed-kernel}
If \(\G\) is compact and connected, \(\mathcal O_0=\mathcal O_T=g\), and the relevant source, target, and reference group densities are continuous and strictly positive, then Assumption~3.1 holds. The H\"ormander and positivity verification for Eq.~(1.3) under (3.1) is given in \cref{app:proof-compact-kernel}.
\end{lemma}

\begin{remark}\label{rem:compact-group-unbounded-velocity}
Although \(\X=\G\times\g\) is noncompact, group coordinate only observations integrate out the velocity coordinate before the endpoint kernel is formed.
\end{remark}

\begin{lemma}[Reference-compatible state lift]\label{lem:canonical-lift}
If \(Q_i(\dd g\dd\bm\xi)=Q_i^\G(\dd g)Q_i(\dd\bm\xi\mid g)\) and \(\rho\ll Q_i^\G\), then the unique state law with group marginal \(\rho\) minimizing \(\KL(\cdot\|Q_i)\) is
\begin{equation*}
\pi^\circ(\dd g\dd\bm\xi)=\rho(\dd g)Q_i(\dd\bm\xi\mid g).
\end{equation*}
The proof is given in \cref{app:proof-canonical-lift}.
\end{lemma}

\subsection{Auxiliary Doob and path-stability identities}\label{app:doob-auxiliary-details}

\begin{lemma}[Control-law relative entropy]\label{lem:control-kl-stability}
For controlled laws with the same initial law and diffusion coefficient,
\begin{equation*}
\KL(P^{\bm u}\|P^{\bm v})
=\frac12\E_{P^{\bm u}}\int_0^T\norm{\bm u_t-\bm v_t}^2\,\dd t
\end{equation*}
whenever the Girsanov exponentials are true martingales. In particular,
\begin{equation*}
\norm{P_t^{\bm u}-P_t^{\bm v}}_{\rm TV}
\leq\frac12\left(\E_{P^{\bm u}}\int_0^T\norm{\bm u_s-\bm v_s}^2\,\dd s\right)^{1/2}.
\end{equation*}
The proof is given in \cref{app:proof-control-kl}.
\end{lemma}

\begin{lemma}[Residual-to-path estimate]\label{thm:residual-stability}
For a positive approximate backward factor \(\widetilde h\), the Kolmogorov residual and endpoint log-factor mismatch control the path-law KL error of the induced Doob control. The exact identity and bound are Eqs.~\eqref{eq:exact-residual-identity}--\eqref{eq:residual-kl-bound} in \cref{app:residual-details}.
\end{lemma}

\subsection{Wrapped kernel and WKBC regression details}\label{app:wkbc-theory-details}

For \(\G=\T^m=\R^m/\Lambda\), define
\begin{equation*}
a_t=e^{-\gamma t},
\quad
r_t=\frac{1-a_t}{\gamma},
\quad
\sigma_{\xi,t}^2=1-a_t^2,
\quad
\sigma_{z\xi,t}=\frac{(1-a_t)^2}{\gamma},
\end{equation*}
\begin{equation*}
\sigma_{z,t}^2=\frac{2t}{\gamma}
-\frac{4(1-a_t)}{\gamma^2}
+\frac{1-a_t^2}{\gamma^2}.
\end{equation*}

\begin{lemma}[Wrapped kinetic kernels]\label{prop:wrapped-kinetic-kernel}
For a lift \(z\in\R^m\) of \(g_0^{-1}g\),
\begin{equation}\label{eq:wrapped-kernel}
q_t(g,\bm\xi\mid g_0,\bm\xi_0)
=\Vol(\Lambda)\sum_{\lambda\in\Lambda}
\mathcal N_{2m}\!\left(
\begin{bmatrix}z+\lambda\\\bm\xi\end{bmatrix};
\begin{bmatrix}r_t\bm\xi_0\\a_t\bm\xi_0\end{bmatrix},
\begin{bmatrix}
\sigma_{z,t}^2I&\sigma_{z\xi,t}I\\
\sigma_{z\xi,t}I&\sigma_{\xi,t}^2I
\end{bmatrix}\right).
\end{equation}
For \(\mathcal O_T(g,\bm\xi)=g\),
\begin{equation}\label{eq:wrapped-position-kernel}
\overline q_t(g\mid g_0,\bm\xi_0)
=\Vol(\Lambda)\sum_{\lambda\in\Lambda}
\mathcal N_m(z+\lambda;r_t\bm\xi_0,\sigma_{z,t}^2I).
\end{equation}
Both sums are independent of the selected lift.
\end{lemma}

\begin{lemma}[WKBC initial-state law]\label{prop:wkbc-initial-law}
The formulas for \(P_0^{\rm SB}\) and \(P_0^{\rm SB}(\dd\bm\xi\mid g)\) stated in Corollary~3.4 follow directly from Eq.~(1.2) after conditioning on the group coordinate. In particular, the source factor cancels inside \(\bm\xi\mid g\); the terminal factor propagated to time zero is what tilts the reference velocity conditional.
\end{lemma}

Let \(S_{t,T}(x,y)=\nabla_{\bm\xi}\log q_{T-t}^{\mathcal O_T}(y\mid x)\).  The calibrated bridge can be used directly as the regression design measure because
\begin{equation*}
\frac{\dd P^{\rm SB}}{\dd Q}
=f_0(\mathcal O_0(X_0))g_T(\mathcal O_T(X_T)).
\end{equation*}
Accordingly, with \(Y_T=\mathcal O_T(X_T)\), define the bridge-weighted population risk
\begin{equation*}
\mathcal J_{\rm WKBC}^{\rm br}(\bm s)
=\frac{1}{T}\int_0^T
\E_{P^{\rm SB}}\!\left[
\norm{\bm s(t,X_t)-S_{t,T}(X_t,Y_T)}^2
\right]\,\dd t.
\end{equation*}
The uniform time density \(1/T\) is fixed rather than tuned. Numerically, the same risk is implemented from a reference-path pool using the self-normalized endpoint weight
\(w_{\rm br}\propto f_{0,M}(\mathcal O_0(X_0))g_{T,M}(Y_T)\), so it requires only the factors already produced by WKBC.

\begin{proposition}[WKBC bridge-weighted regression and numerical consistency]\label{thm:wkbc-stability}
The population minimizer is \(\bm s^*=\nabla_{\bm\xi}\log h_t^{\beta^*}\), and
\begin{equation}\label{eq:wkbc-excess-risk}
\mathcal J_{\rm WKBC}^{\rm br}(\bm s)-\mathcal J_{\rm WKBC}^{\rm br}(\bm s^*)
=\frac{1}{T}\int_0^T
\int\norm{\bm s(t,x)-\bm s^*(t,x)}^2
P_t^{\rm SB}(\dd x)\,\dd t.
\end{equation}
Let \(P^{\bm s,\pi_0}\) denote the controlled law with normalized control
\(\sqrt{2\gamma}\,\bm s\) and initial law \(\pi_0\). If
\(P_0^{\rm SB}\ll\pi_0\) and the relevant Girsanov exponentials are true martingales, then
\begin{equation}\label{eq:wkbc-population-path-error}
\KL(P^{\rm SB}\|P^{\bm s,\pi_0})
=\KL(P_0^{\rm SB}\|\pi_0)
+\gamma T\bigl[\mathcal J_{\rm WKBC}^{\rm br}(\bm s)
-\mathcal J_{\rm WKBC}^{\rm br}(\bm s^*)\bigr].
\end{equation}
In particular, when \(\pi_0=P_0^{\rm SB}\),
\begin{equation*}
d_{\rm BL}(P^{\rm SB},P^{\bm s,P_0^{\rm SB}})
\le
\sqrt{2\gamma T\bigl[\mathcal J_{\rm WKBC}^{\rm br}(\bm s)
-\mathcal J_{\rm WKBC}^{\rm br}(\bm s^*)\bigr]}.
\end{equation*}
Thus the regression-to-path-law conversion has the explicit constant \(\gamma T\) and requires no auxiliary occupation-density-ratio constant.

For positive numerical factors \((f_{0,N},g_{T,N})\), the same reverse-KL argument applies to the corresponding numerical bridge and yields the path-space decomposition in Eq.~\eqref{eq:wkbc-total-error}. Under the approximation conditions in \cref{app:wkbc-implementation}, its numerical-bridge term tends to zero as the lattice, quadrature, Sinkhorn, interpolation, and initial-law approximations are refined, while its regression term vanishes with bridge-weighted risk consistency.
\end{proposition}

\begin{remark}[Why no \(C_{\rm occ}\) is reported]\label{rem:wkbc-no-cocc}
An off-policy regression proposal would require a domination constant of the form
\(C_{\rm occ}=\|\dd\nu^{\bm s}/\dd\mathfrak m\|_\infty\), which is not a tunable hyperparameter and is difficult to certify reliably in high dimension. WKBC avoids this constant by exploiting its explicit calibrated bridge: endpoint reweighting of reference paths makes the population regression design measure equal to the bridge occupation law. The reverse-KL orientation in Eq.~\eqref{eq:wkbc-population-path-error} then evaluates the control error under exactly that law. Hence the reported WKBC implementation does not select or estimate \(C_{\rm occ}\).
\end{remark}

\subsection{Detailed RCCBM calibration, teachers, and learning assumptions}\label{app:rccbm-detailed-theory}

\begin{lemma}[Endpoint entropy disintegration]\label{lem:endpoint-disintegration}
For \(P\ll Q\), \(\Gamma=Y_\#P\), and conditional laws \(P^y,Q^y\),
\begin{equation*}
\KL(P\|Q)=\KL(\Gamma\|\mathcal R_{0T})
+\int\KL(P^y\|Q^y)\,\Gamma(\dd y).
\end{equation*}
Hence, for fixed \(\Gamma\), the unique minimizer is \(P_\Gamma=\int Q^y\Gamma(\dd y)\).
\end{lemma}

Define
\begin{equation*}
(\mathsf Ka)(y_0)=\int k(y_0,y_T)a(y_T)\,\rho_T(\dd y_T),
\qquad
(\mathsf K^*c)(y_T)=\int k(y_0,y_T)c(y_0)\,\rho_0(\dd y_0).
\end{equation*}
\begin{proposition}[Half-bridge correction]\label{prop:rccbm-ipf}
The updates
\begin{equation}\label{eq:rccbm-ipf-updates}
f_0^{(n+1)}=(\mathsf Kg_T^{(n)})^{-1},
\qquad
g_T^{(n+1)}=(\mathsf K^*f_0^{(n+1)})^{-1}
\end{equation}
are alternating KL projections onto the two marginal constraint sets and converge in Hilbert's projective metric to \((f_0,g_T)\), up to reciprocal gauge. The first factor is the source corrector omitted by a terminal-only construction.
\end{proposition}

The empirical endpoint dual is
\begin{align*}
\widehat{\mathscr D}_N(\beta)
&=\frac1{N_0}\sum_{j=1}^{N_0}\beta_0(Y_0^{(j)})
+\frac1{N_T}\sum_{j=1}^{N_T}\beta_T(Y_T^{(j)})\\
&\quad-
\log\left[\frac1{N_Q}\sum_{i=1}^{N_Q}
 e^{\beta_0(Y_0^{Q,i})+\beta_T(Y_T^{Q,i})}\right].
\end{align*}

\begin{assumption}[Endpoint calibration]\label{ass:rcm-calibration}
The doubly centered endpoint classes \(\mathcal B_N\) are uniformly bounded and satisfy uniform empirical-dual convergence, sieve approximation, and optimization consistency. Denote the corresponding errors by \(\delta_{{\rm end},N}\), \(\varepsilon_{{\rm end,app},N}\), and \(\varepsilon_{{\rm end,opt},N}\); all converge to zero in probability. The precise definitions are Eqs.~\eqref{eq:endpoint-uniform-dual-error} and~\eqref{eq:endpoint-optimization-gap}. Effective sample size provides a finite-sample diagnostic for the endpoint weights.
\end{assumption}

\begin{proposition}[Consistency of empirical endpoint calibration]\label{prop:rcm-endpoint-calibration-consistency}
Under Assumption~\ref{ass:rcm-calibration},
\begin{equation}\label{eq:endpoint-calibration-consistency-bound}
\varepsilon_{{\rm end},N}
\leq\varepsilon_{{\rm end,app},N}
+2\delta_{{\rm end},N}
+\varepsilon_{{\rm end,opt},N},
\end{equation}
and hence \(\varepsilon_{{\rm end},N}\to0\) in probability.
\end{proposition}

\begin{assumption}[Conditional-kernel regularity]\label{ass:rccbm-conditional-kernel}
For \(t<T\), \(r_{t,T}^{\mathcal O_T}(y\mid x)\) is strictly positive and continuously differentiable in \(\bm\xi\). Differentiation under the endpoint integral is valid on every strip \([0,T-\tau]\), and the relevant martingale problems and Girsanov exponentials are well posed. The exact pinned controls are required to be square-integrable only on such preterminal strips.
\end{assumption}

\begin{corollary}[Forward-control sufficiency]\label{cor:rccbm-forward-sufficiency}
If \(P_0^\theta=P_0^{\rm SB}\), \(\bm u_\theta=\bm u^*\) in \(L^2(P_t^{\rm SB}\dd t)\), and the Girsanov exponentials are true martingales, then \(P^\theta=P^{\rm SB}\). Propagated factor networks are therefore unnecessary for exact forward sampling once the initial law and forward control are identified.
\end{corollary}

Let \(\ell_{\varepsilon,y}(x)=-\log\kappa_\varepsilon(y,\mathcal O_T(x))\). A normalized conditional control \(\bm v\) acts through the same forced velocity coordinate as the reference process.
\begin{proposition}[CondSOC Gibbs identity]\label{thm:rccbm-condsoc-kl}
For a controlled law \(P^{z,\bm v}\) started at \(x_0\),
\begin{equation*}
\mathcal J_{z,\varepsilon}(\bm v)
=\E_{P^{z,\bm v}}\left[
\frac12\int_0^T\norm{\bm v_t}^2\,\dd t
+\ell_{\varepsilon,y}(X_T)\right]
\end{equation*}
satisfies
\begin{equation}\label{eq:rccbm-condsoc-kl-identity}
\mathcal J_{z,\varepsilon}(\bm v)+\log Z_{z,\varepsilon}
=\KL(P^{z,\bm v}\|Q^{z,\varepsilon}).
\end{equation}
If \(\bm v^{z,\varepsilon}\) is the Doob control of \(Q^{z,\varepsilon}\), then
\begin{equation}\label{eq:rccbm-condsoc-control-identity}
\mathcal J_{z,\varepsilon}(\bm v)-\mathcal J_{z,\varepsilon}(\bm v^{z,\varepsilon})
=\frac12\E_{P^{z,\bm v}}\int_0^T
\norm{\bm v_t-\bm v_t^{z,\varepsilon}(X_t)}^2\,\dd t.
\end{equation}
\end{proposition}

The finite teacher objective is
\begin{align}
\mathcal J_{z,\varepsilon}^{K,S}(\bm v)
&=\E\left[\ell_{\varepsilon,y}(X_T)
+\frac12\sum_{k=0}^{K-1}\Delta t\norm{\bm v_k}^2\right]
\notag\\
&\quad+\lambda_S\E\sum_{k=0}^{K-2}\norm{\bm v_{k+1}-\bm v_k}^2.
\label{eq:rcm-condsoc-objective}
\end{align}
The energy coefficient is fixed at \(1/2\); the smoothness term is a numerical bias and is required to vanish in the asymptotic approximation regime.

\begin{assumption}[High-level teacher and empirical-risk approximation]\label{ass:rccbm-teacher-consistency}\label{ass:rccbm-empirical-risk}
For \(\varepsilon_N\rightarrow0\), the conditional solver, replay correction, direct-control sieve, clipping, empirical-risk optimization, and constrained refinement satisfy the approximation conditions in \cref{app:rccbm-learning-assumptions}. In particular,
\begin{equation*}
\varepsilon_{{\rm app},N},\ 
\varepsilon_{{\rm opt},N},\ 
\varepsilon_{{\rm gen},N},\ 
\varepsilon_{{\rm teach},N},\ 
\varepsilon_{{\rm iw},N},\ 
\varepsilon_{{\rm clip},N},\ 
\delta_{{\rm ref},N}
\longrightarrow0
\end{equation*}
in probability, and \(\lambda_{S,N}R_{S,N}\to0\).
\end{assumption}

\begin{assumption}[Uniform mollification regularity]\label{ass:rccbm-mollification-regularity}
The endpoint-factor family generated by \(\bigcup_N\mathcal B_N\cup\{\beta^*\}\) is uniformly bounded and equicontinuous; \(\kappa_\varepsilon\) is a positive approximate identity uniformly over this family; and the reference observation kernels and their \(\bm\xi\)-gradients admit the preterminal integrable envelope in \cref{app:mollification-details}. When \(L^2\) control convergence is used, the squared controls are uniformly integrable.
\end{assumption}

\begin{proposition}[Mollification limit]\label{prop:rccbm-condsoc-consistency}
Under Assumption~\ref{ass:rccbm-mollification-regularity},
\begin{equation}\label{eq:rccbm-mollification-convergence}
\sup_{\beta\in\mathcal B_\infty}d_{\rm BL}(P^{\beta,\varepsilon},P^\beta)\longrightarrow0.
\end{equation}
Together with Assumption~\ref{ass:rccbm-conditional-kernel}, the envelope condition also yields \(\bm u^{\beta,\varepsilon}\to\bm u^\beta\) locally on every strip \([0,T-\tau]\) and in the corresponding uniformly integrable \(L^2\) topology.
\end{proposition}

\paragraph{Replay, direct control, and refinement}
Let \(\mathfrak M^{\beta,\varepsilon}\) be the canonical joint teacher measure. A replay proposal \(\Pi_{\rm rep}\) is corrected by
\begin{equation}\label{eq:rccbm-replay-importance}
\omega_{\rm rep}=\frac{\dd\mathfrak M^{\beta,\varepsilon}}{\dd\Pi_{\rm rep}}.
\end{equation}
The empirical matching loss is
\begin{equation*}
\widehat{\mathcal L}_{\rm match}(\theta)
=\frac1{N_{\rm rep}}\sum_{j=1}^{N_{\rm rep}}
\widetilde\omega_j
\norm{\bm u_\theta(t_j,G_j,\bm\Xi_j)-\bm v_j}^2.
\end{equation*}
Refinement is accepted only if
\begin{equation}\label{eq:rccbm-constrained-refinement}
\widehat{\mathcal L}_{\rm match}(\theta)
\leq\widehat{\mathcal L}_{\rm match}(\theta_{\rm pre})+\delta_{\rm ref}.
\end{equation}

\begin{proposition}[Feature identification and refinement fixed point]\label{prop:rcm-mmd-identification}
{For a continuous injective identification map \(\Phi_{\G}^{\rm id}\) and a positive finite mixture of Gaussian kernels, feature-space MMD identifies weak convergence on \(\G\); at the exact population minimizer, zero-tolerance constrained refinement preserves the same control equivalence class.  The production identification maps for \(SO(3)\), \(U(n)\), and the compact reduced product group are fixed full-rank target-whitened embeddings and therefore satisfy the required injectivity condition.  The MMD identification map is the fixed full-rank target-whitened embedding, while spectral features enter as auxiliary controller features.  The precise statement and verification are recorded in \cref{app:feature-identification-details}.}
\end{proposition}

For reference samples \((X_0^{(i)},Y_0^{Q,i},Y_T^{Q,i})\), define
\begin{equation*}
\overline w_{i,N}^{\varepsilon}
=\frac{
 e^{\widehat\beta_{0,N}(Y_0^{Q,i})}
 g_{T,\varepsilon}^{\widehat\beta_N}(Y_T^{Q,i})}
{\sum_{\ell=1}^{N_Q}
 e^{\widehat\beta_{0,N}(Y_0^{Q,\ell})}
 g_{T,\varepsilon}^{\widehat\beta_N}(Y_T^{Q,\ell})},
\end{equation*}
and
\begin{equation}\label{eq:rcm-calibrated-initial-law}
\widehat P_{0,N}^{\varepsilon}
=\sum_{i=1}^{N_Q}\overline w_{i,N}^{\varepsilon}\delta_{X_0^{(i)}}.
\end{equation}
This self-normalized empirical law estimates \(P_0^{\widehat\beta_N,\varepsilon}\), with its approximation error quantified by the bounded-Lipschitz distance in Assumption~\ref{ass:rccbm-statistical-stability}.

\begin{assumption}[Statistical and dynamical stability]\label{ass:rccbm-statistical-stability}
The weighted initial-law error
\(\varepsilon_{{\rm init},N}:=d_{\rm BL}(\widehat P_{0,N}^{\varepsilon_N},P_0^{\widehat\beta_N,\varepsilon_N})\)
and the Lie--Trotter path error \(\varepsilon_{{\rm disc},N}\) vanish in probability. Uniformly over the learned control sieve,
\begin{equation}\label{eq:rccbm-initial-stability}
d_{\rm BL}(\mathsf S_{\bm u}(\mu),\mathsf S_{\bm u}(\nu))
\leq C_{\rm init}d_{\rm BL}(\mu,\nu).
\end{equation}
Sufficient Lipschitz and growth conditions are given in \cref{app:apbm-sampling-details}.
\end{assumption}

\subsection{Reduced-space assumptions and reconstruction}\label{app:reduction-details}

The reduced endpoint laws are
\begin{equation*}
\rho_i^{\rm red}=(\mathcal R_{\rm red})_\#\rho_i^{\rm amb}.
\end{equation*}
\begin{assumption}[Reduced representation]\label{ass:compact-internal-representation}
The modeled support is an embedded conformation manifold \(\mathcal S\subset SE(3)^N/\Delta SE(3)\), and for exact ambient statements the reconstruction \(F:\G_{\rm red}\to\mathcal S\) is a \(C^2\)-diffeomorphism with inverse \(\mathcal R_{\rm red}|_{\mathcal S}\). Otherwise all stochastic claims are interpreted on the reduced state space.
\end{assumption}
Define
\begin{equation*}
\Phi_{\rm red}(z,\bm\zeta)
=\left(F(z),\dd F_z[T_eL_z(\bm\zeta)]\right).
\end{equation*}
Under Assumption~\ref{ass:compact-internal-representation}, the induced state-path map used in Proposition~3.10 is a Borel isomorphism onto its image. Without injectivity, only the data-processing statement in Eq.~(3.27) is asserted.

\subsection{Path-space topology and adjoint conventions}\label{app:path-topology}
For \(\mu=\mu_\G\otimes\dd\bm\xi\), define
\begin{equation*}
\chi_i=\Div_{\mu_\G}\widetilde E_i,
\qquad
\chi(g,\bm\xi)=\sum_{i=1}^d\chi_i(g)\xi_i.
\end{equation*}
Fix a complete metric \(d_\G\) compatible with the topology of \(\G\), and set
\begin{equation*}
d_\X\bigl((g,\bm\xi),(g',\bm\eta)\bigr)
=d_\G(g,g')+\norm{\bm\xi-\bm\eta},
\qquad
\overline d_\X=1\wedge d_\X.
\end{equation*}
On \(\Path=C([0,T],\X)\), the bounded uniform path metric is
\begin{equation*}
d_{\rm path}(\omega,\omega')
=\sup_{0\leq t\leq T}\overline d_\X(\omega_t,\omega'_t).
\end{equation*}
For bounded \(F:\Path\to\R\), define
\begin{equation*}
\operatorname{Lip}_{\rm path}(F)
=\sup_{\omega\neq\omega'}
\frac{|F(\omega)-F(\omega')|}{d_{\rm path}(\omega,\omega')},
\qquad
\norm{F}_{\rm BL}=\max\{\norm{F}_\infty,\operatorname{Lip}_{\rm path}(F)\},
\end{equation*}
and
\begin{equation}\label{eq:bl-distance-definition}
d_{\rm BL}(P,R)
=\sup_{\norm{F}_{\rm BL}\leq1}
\left|\int F\,\dd P-\int F\,\dd R\right|.
\end{equation}
The bounded metric \(d_{\rm path}\) is compatible with the Polish path-space topology, so \(d_{\rm BL}\) metrizes weak convergence and \(d_{\rm BL}\leq2\,\TV\). The formal adjoint of the kinetic generator is
\begin{equation*}
L_t^\dagger r=-Ar-\chi r-\Div_{\bm\xi}(\bm b_t r)+\frac{\alpha^2}{2}\Delta_{\bm\xi}r.
\end{equation*}

\subsection{Residual-to-path identity}\label{app:residual-details}
Let \(h>0\) be the exact backward factor, \(\widetilde h>0\) a \(C^{1,2}\) approximation,
\begin{equation*}
r_{\rm K}=\frac{\partial_t\widetilde h+L_t\widetilde h}{\widetilde h},
\qquad
\widetilde{\bm u}=\alpha\nabla_{\bm\xi}\log\widetilde h,
\qquad
\ell=\log(\widetilde h/h).
\end{equation*}
If the induced controlled laws have the same initial-state law and the required integrability holds, then
\begin{equation}\label{eq:exact-residual-identity}
\KL(P^{\widetilde{\bm u}}\|P^{\rm SB})
=\E_{P^{\widetilde{\bm u}}}
\left[\ell_T(X_T)-\ell_0(X_0)-\int_0^T r_{\rm K}(t,X_t)\,\dd t\right],
\end{equation}
and hence
\begin{equation}\label{eq:residual-kl-bound}
\KL(P^{\widetilde{\bm u}}\|P^{\rm SB})
\leq\E|\ell_T(X_T)|+\E|\ell_0(X_0)|
+\E\int_0^T|r_{\rm K}(t,X_t)|\,\dd t.
\end{equation}

\subsection{WKBC numerical error decomposition}\label{app:wkbc-error-decomposition}
For positive numerical factors \((f_{0,N},g_{T,N})\), let \(P_N^*\) be the calibrated numerical bridge, let \(\bm s_N^*=\nabla_{\bm\xi}\log h_{t,N}\), and let
\(\mathcal J_{{\rm WKBC},N}^{\rm br}\) denote the bridge-weighted regression risk defined in Proposition~\ref{thm:wkbc-stability} with \(P^{\rm SB}\) replaced by \(P_N^*\).  The triangle inequality and the reverse-KL regression bound give
\begin{equation}\label{eq:wkbc-total-error}
d_{\rm BL}(P_N^{\bm s},P^{\rm SB})
\le
d_{\rm BL}(P_N^*,P^{\rm SB})
+\sqrt{2\gamma T\bigl[
\mathcal J_{{\rm WKBC},N}^{\rm br}(\bm s)
-\mathcal J_{{\rm WKBC},N}^{\rm br}(\bm s_N^*)
\bigr]}.
\end{equation}
The first term collects the lattice, quadrature, Sinkhorn, factor-interpolation, and numerical initial-law errors. Under the approximation conditions in \cref{app:wkbc-implementation}, these approximations converge and therefore \(d_{\rm BL}(P_N^*,P^{\rm SB})\to0\). The second term is the amortized-control regression error under the numerical bridge occupation law and vanishes under bridge-weighted risk consistency. No occupation-density-ratio constant is needed.

\subsection{RCCBM learning assumptions}\label{app:rccbm-learning-assumptions}
For endpoint calibration, let \(\mathcal B_N\) be the doubly centered endpoint class. Assume
\[
\sup_N\sup_{\beta\in\mathcal B_N}
\bigl(\norm{\beta_0}_\infty+\norm{\beta_T}_\infty\bigr)
\leq C_\beta,
\]
together with
\begin{equation}\label{eq:endpoint-uniform-dual-error}
\delta_{{\rm end},N}
:=\sup_{\beta\in\mathcal B_N}
\left|\widehat{\mathscr D}_N(\beta)-\mathscr D(\beta)\right|
\longrightarrow0
\end{equation}
in probability. There exists \(\beta_N^\circ\in\mathcal B_N\) such that
\begin{equation*}
\varepsilon_{{\rm end,app},N}:=\mathscr D(\beta^*)-\mathscr D(\beta_N^\circ)\longrightarrow0,
\end{equation*}
and the learned \(\widehat\beta_N\) satisfies
\begin{equation}\label{eq:endpoint-optimization-gap}
\widehat{\mathscr D}_N(\widehat\beta_N)
\geq\sup_{\beta\in\mathcal B_N}\widehat{\mathscr D}_N(\beta)
-\varepsilon_{{\rm end,opt},N},
\qquad
\varepsilon_{{\rm end,opt},N}\longrightarrow0
\end{equation}
in probability.

For teacher and empirical-risk approximation, the CondSOC knot classes are dense for the fixed-\(\varepsilon_N\) finite-energy problems; optimization and time-discretization errors vanish; and \(\lambda_{S,N}R_{S,N}\to0\). Replay proposals dominate the canonical mollified teacher law with uniformly integrable importance ratios. Direct-control classes form a sieve with clipping levels \(B_N\rightarrow\infty\), vanishing clipping tails, and a uniform law of large numbers for the importance-weighted loss. Both the pre-refinement empirical solution and every refinement checkpoint eligible for acceptance are required to lie in the same sieve \(\mathcal U_N\), so the same uniform empirical-to-population risk bound applies to both sides of the refinement acceptance comparison. These conditions generate the separated errors \(\varepsilon_{{\rm app},N},\varepsilon_{{\rm opt},N},\varepsilon_{{\rm gen},N},\varepsilon_{{\rm teach},N},\varepsilon_{{\rm iw},N},\varepsilon_{{\rm clip},N},\delta_{{\rm ref},N}\) used in Theorem~3.8.

\subsection{Verification of the RCCBM assumptions on \(\T^m\)}\label{app:torus-rccbm-verification}

This subsection proves Corollary~3.9. Its purpose is to exhibit one explicit, nonempty regime in which the high-level assumptions of Theorem~3.8 can be checked from primitive approximation conditions.

\begin{proof}[Proof of Corollary~3.9]
We verify the assumptions in the order in which they enter Theorem~3.8.

\paragraph{1. Endpoint calibration}
For \(\G=\T^m\) and the reference dynamics in Eq.~(1.3) under the Ornstein--Uhlenbeck specialization (3.1), the observed group transition kernel is the wrapped Gaussian kernel in Eq.~\eqref{eq:wrapped-position-kernel}. On every positive time interval its lattice series and all derivatives converge absolutely and locally uniformly. Hence the observed endpoint density is \(C^\infty\) and strictly positive. Because \(\T^m\times\T^m\) is compact and the assumed source, target, and reference group densities are strictly positive and continuous, the likelihood ratio in Assumption~3.1 is bounded above and below. Theorem~3.2 therefore gives unique positive endpoint factors, up to the usual scalar gauge. The corresponding centered log-potentials \(\beta_0^*,\beta_T^*\) are \(C^2\).

Let \(\sigma_{K_N}\beta_i^*\) be the Fej\'er means of \(\beta_i^*\). Fej\'er approximation on the torus gives
\[
\|\sigma_{K_N}\beta_i^*-\beta_i^*\|_\infty\longrightarrow0,
\qquad i\in\{0,T\},
\]
and the approximants can be centered without changing this convergence. Thus the endpoint approximation gap
\(\varepsilon_{{\rm end,app},N}\) tends to zero. On the bounded finite-dimensional sieve \(\mathcal B_N\), the empirical dual summands are uniformly bounded and Lipschitz in the coefficient vector. A standard covering argument gives
\[
\sup_{\beta\in\mathcal B_N}
\left|\widehat{\mathscr D}_N(\beta)-\mathscr D(\beta)\right|
=o_{\Prob}(1)
\]
whenever \(d_N\log N/N\to0\). Exact empirical maximization over the compact sieve gives
\(\varepsilon_{{\rm end,opt},N}=0\). Hence Assumption~\ref{ass:rcm-calibration} holds, and Proposition~\ref{prop:rcm-endpoint-calibration-consistency} yields
\(\varepsilon_{{\rm end},N}\to0\) in probability.

\paragraph{2. Conditional kernels}
The wrapped formulas in Lemma~\ref{prop:wrapped-kinetic-kernel} show that, for \(t<T\), the observed transition density and its \(\bm\xi\)-derivatives are smooth. On any strip \(0\le t\le T-\tau\), absolute convergence of the wrapped Gaussian series is uniform on compact \(\bm\xi\)-sets. Since the endpoint variable \(y\) ranges over the compact torus, the required derivative envelope can be chosen integrable with respect to Haar measure. This verifies Assumption~\ref{ass:rccbm-conditional-kernel}.

The periodic heat kernels form a positive approximate identity. The endpoint sieve is contained in a common bounded Lipschitz ball, so the family
\(\{e^{\beta_T}:\beta\in\mathcal B_\infty\}\)
is uniformly bounded and equicontinuous. Therefore
\[
\sup_{\beta\in\mathcal B_\infty}
\left\|
\kappa_{\varepsilon}*e^{\beta_T}-e^{\beta_T}
\right\|_\infty
\longrightarrow0.
\]
The same wrapped-kernel bounds control the preterminal derivatives. For each fixed \(\varepsilon>0\), the mollified conditional teacher has finite moments of every order needed below. The schedule \(\varepsilon_N\rightarrow0\) is chosen diagonally slowly enough that the fourth-moment quantity \(M_N\) in Corollary~3.9 satisfies the displayed empirical-process rate. These facts verify Assumption~\ref{ass:rccbm-mollification-regularity} and give
\(\varepsilon_{{\rm mol},N}\to0\).

\paragraph{3. Teacher and direct-control approximation}
In this torus specialization the wrapped kernel makes
\(Z_{z,\varepsilon_N}\), \(Q^{z,\varepsilon_N}\), and the corresponding mollified Doob teacher available at the population level. We therefore sample directly from the canonical teacher measure
\(\mathfrak M^{\beta,\varepsilon_N}\). The replay Radon--Nikodym ratio is then identically one, so
\(\varepsilon_{{\rm iw},N}=0\). The auxiliary CondSOC discretization is not needed for this verification regime; equivalently one may take
\(\lambda_{S,N}=0\), \(R_{S,N}=0\), and
\(\varepsilon_{{\rm teach},N}=0\).

For fixed \(N\), the mollified Markov regression target is smooth on
\([0,T]\times\T^m\times\{\|\bm\xi\|\le R_N\}\).
Tensor products of trigonometric polynomials, splines, and ordinary polynomials are uniformly dense on this compact set. Hence the direct-control sieve can be chosen so that its population approximation error tends to zero. Outside the velocity truncation, the Ornstein--Uhlenbeck reference and its bounded endpoint tilts have Gaussian tails. Consequently, after choosing \(R_N\rightarrow\infty\) and \(B_N\rightarrow\infty\) with
\(B_N^2e^{-cR_N^2}\to0\), both the truncation and clipping contributions vanish.

The squared matching loss over \(\mathcal U_N\) has a finite envelope controlled by \(M_NB_N^4\). The coefficient parameterization is finite dimensional and has covering entropy of order \(p_N\log(1/\eta)\). Therefore
\[
\sup_{\bm u\in\mathcal U_N}
\left|
\widehat{\mathcal L}_{N}(\bm u)
-\mathcal L_{\beta,\varepsilon_N}(\bm u)
\right|
=o_{\Prob}(1)
\]
under
\(M_NB_N^4p_N\log N/N\to0\).
Exact empirical minimization gives
\(\varepsilon_{{\rm opt},N}=0\), while the sieve approximation, generalization, and clipping terms tend to zero. Taking the refinement map to be the identity gives
\(\delta_{{\rm ref},N}=0\). Thus Assumption~\ref{ass:rccbm-teacher-consistency} is satisfied.

\paragraph{4. Calibrated initial law and dynamical stability}
Because the endpoint potentials are uniformly bounded and the periodic heat mollification is strictly positive, the self-normalized initial weights in Eq.~\eqref{eq:rcm-calibrated-initial-law} have finite moments and a nonzero normalizing limit. The bounded-weight law of large numbers therefore gives
\[
d_{\rm BL}\!\left(
\widehat P_{0,N}^{\varepsilon_N},
P_0^{\widehat\beta_N,\varepsilon_N}
\right)
\longrightarrow0
\qquad\text{in probability}.
\]
The direct-control sieves are restricted to a common Lipschitz/linear-growth envelope. Standard synchronous-coupling and Gr\"onwall estimates for the kinetic SDE then give Eq.~\eqref{eq:rccbm-initial-stability} with a constant independent of \(N\). Finally, the Lie--Trotter approximation in Proposition~\ref{prop:rcm-exact-ou-geometry} converges in bounded-Lipschitz path distance when \(\Delta t_N\rightarrow0\). Hence
\(\varepsilon_{{\rm init},N}\to0\) and
\(\varepsilon_{{\rm disc},N}\to0\), proving Assumption~\ref{ass:rccbm-statistical-stability}.

All terms on the right-hand side of Eq.~(3.22) therefore vanish in probability. The stated bounded-Lipschitz path-law convergence follows from Theorem~3.8.
\end{proof}

\subsection{Uniform mollification conditions}\label{app:mollification-details}
Let
\begin{equation*}
\mathcal B_\infty=\left(\bigcup_{N\geq1}\mathcal B_N\right)\cup\{\beta^*\},
\qquad
\mathcal H_T^{\rm end}=\{e^{\beta_T}:\beta\in\mathcal B_\infty\}.
\end{equation*}
The family \(\mathcal H_T^{\rm end}\) is uniformly bounded and equicontinuous, and
\begin{equation*}
\delta_\varepsilon^{\rm mol}
:=\sup_{h\in\mathcal H_T^{\rm end}}\sup_{y'\in\mathsf Y_T}
\left|\int\kappa_\varepsilon(y,y')h(y)\,\nu_T(\dd y)-h(y')\right|
\longrightarrow0.
\end{equation*}
For every \(\tau>0\) and compact \(K\subset\X\), assume an envelope \(H_{\tau,K}\in L^1(\nu_T)\) such that
\begin{equation}\label{eq:mollification-kernel-envelope}
\sup_{0\leq t\leq T-\tau}\sup_{x\in K}
\left[r_{t,T}^{\mathcal O_T}(y\mid x)
+\norm{\nabla_{\bm\xi}r_{t,T}^{\mathcal O_T}(y\mid x)}\right]
\leq H_{\tau,K}(y)
\quad\text{for }\nu_T\text{-a.e. }y.
\end{equation}
When \(L^2\) convergence of the mollified controls is invoked, their squared norms are assumed uniformly integrable with respect to the corresponding occupation measures.

\subsection{Feature identification used in refinement}\label{app:feature-identification-details}
{
The theoretical identification argument requires only the feature map appearing in the MMD term to be continuous and injective.  We therefore distinguish the \emph{identification feature} \(\Phi_{\G}^{\rm id}\) from any auxiliary feature \(\Psi_{\G}^{\rm aux}\) used by the neural controller.  The controller input may be written as
\begin{equation}\label{eq:rccbm-network-feature-split}
\Phi_{\G}^{\rm net}(g)
=
\bigl[\Phi_{\G}^{\rm id}(g),\Psi_{\G}^{\rm aux}(g)\bigr].
\end{equation}
In particular, the spectral descriptors used in the \(U(n)\) experiments are included in \(\Psi_{\G}^{\rm aux}\), while the MMD identification result is established through the continuous injective map \(\Phi_{\G}^{\rm id}\).

For the production groups, define a raw embedding \(E_{\G}\) by
\begin{small}
\begin{equation}\label{eq:rccbm-raw-identification-embedding}
E_{\G}(g)
=
\begin{cases}
\operatorname{vec}(\mathbf{R}),
& \G=SO(3),\quad g=\mathbf{R},\\[1mm]
\left[
\operatorname{vec}(\operatorname{Re}U)^{\!\top},
\operatorname{vec}(\operatorname{Im}U)^{\!\top}
\right]^{\!\top},
& \G=U(n),\quad g=U,\\[1mm]
\left[
\operatorname{vec}(\mathbf{R}_2)^{\!\top},\ldots,
\operatorname{vec}(\mathbf{R}_N)^{\!\top},
\cos\theta_1,\sin\theta_1,\ldots,
\cos\theta_K,\sin\theta_K
\right]^{\!\top},
& \G=SO(3)^{N-1}\times\T^K.
\end{cases}
\end{equation}
\end{small}
The last line uses the standard periodic embedding of each torus coordinate and is therefore independent of the choice of angle representative.  Let \(\bm\mu_{\G}\) and \(\bm\Sigma_{\G}\) be fixed target statistics of \(E_{\G}(g)\), computed before checkpoint selection, and choose a fixed \(\lambda_{\rm w}>0\).  The production identification feature is
\begin{equation}\label{eq:rccbm-production-identification-feature}
\Phi_{\G}^{\rm id}(g)
=
W_{\G}\bigl(E_{\G}(g)-\bm\mu_{\G}\bigr),
\qquad
W_{\G}
=
\bigl(\bm\Sigma_{\G}+\lambda_{\rm w}I\bigr)^{-1/2}.
\end{equation}
Thus the whitening step is a regularized full-dimensional linear transform.  Since \(\lambda_{\rm w}>0\), the matrix \(W_{\G}\) is positive definite and invertible.

We now verify the identification condition used in Proposition~\ref{prop:rcm-mmd-identification}.  The maps \(R\mapsto\operatorname{vec}(R)\) and
\(U\mapsto[\operatorname{vec}(\operatorname{Re}U)^{\top},\operatorname{vec}(\operatorname{Im}U)^{\top}]^{\top}\) are continuous and injective.  For \(SO(3)^{N-1}\times\T^K\), the product of the matrix-entry embeddings with \(\theta\mapsto(\cos\theta,\sin\theta)\) is likewise continuous and injective.  Translation by \(\bm\mu_{\G}\) and multiplication by the invertible matrix \(W_{\G}\) preserve both properties.  Consequently, \(\Phi_{\G}^{\rm id}\) is a continuous injective map for every production group considered above.

Because these groups are compact and Euclidean space is Hausdorff, \(\Phi_{\G}^{\rm id}\) is a homeomorphism between \(\G\) and its image.  A positive finite mixture of Gaussian kernels is characteristic and metrizes weak convergence on this Euclidean image.  Hence, for probability measures \(\mu_n,\rho\in\mathcal P(\G)\),
\begin{equation}\label{eq:rccbm-production-mmd-identification}
\MMD_k\!\left(
(\Phi_{\G}^{\rm id})_{\#}\mu_n,
(\Phi_{\G}^{\rm id})_{\#}\rho
\right)\to0
\quad\Longleftrightarrow\quad
\mu_n\Rightarrow\rho.
\end{equation}
This establishes the feature-identification condition for the MMD term through \(\Phi_{\G}^{\rm id}\).  At the exact population minimizer, Eq.~\eqref{eq:rccbm-constrained-refinement} with \(\delta_{\rm ref}=0\) therefore preserves the same control equivalence class.
}

\section{Theoretical Proofs}\label{app:proofs}

\subsection{Proof of Theorem~3.2}

Let \(\lambda=\rho_0\otimes\rho_T\) and \(\mathcal R_{0T}=k\lambda\). Define
\begin{equation*}
(\mathsf Kg)(y_0)=\int k(y_0,y_T)g(y_T)\,\rho_T(\dd y_T),
\qquad
(\mathsf K^*f)(y_T)=\int k(y_0,y_T)f(y_0)\,\rho_0(\dd y_0).
\end{equation*}
On the cone \(L_{++}^\infty\), use Hilbert's projective metric
\begin{equation*}
d_{\rm H}(u,v)=\operatorname*{ess\,sup}\log\frac uv
-\operatorname*{ess\,inf}\log\frac uv.
\end{equation*}
Reciprocal inversion is an isometry. Since
\begin{equation*}
\frac{k(y_0,y_T)k(y_0',y_T')}
{k(y_0,y_T')k(y_0',y_T)}\leq(M/m)^2,
\end{equation*}
Birkhoff's contraction theorem gives a coefficient
\begin{equation*}
\tau\leq\tanh\!\left(\frac12\log\frac Mm\right)<1
\end{equation*}
for both \(\mathsf K\) and \(\mathsf K^*\). Hence
\(\mathsf S=I\circ\mathsf K^*\circ I\circ\mathsf K\) is a contraction on projective classes with coefficient at most \(\tau^2\).

Normalize by \(\int g\,\dd\rho_T=1\). The contraction makes the projective iterates Cauchy; after normalization, \(\log g_n\) is Cauchy in \(L^\infty\). Banach's fixed-point theorem therefore yields a unique projective fixed point \([g_T]\). Set \(f_0=(\mathsf Kg_T)^{-1}\). If \(\mathsf Sg_T=cg_T\), then
\begin{equation*}
f_0\mathsf Kg_T=1,
\qquad
g_T\mathsf K^*f_0=c^{-1}.
\end{equation*}
Both sides define the total mass of \(f_0g_Tk\lambda\), so \(c=1\). Thus
\begin{equation}\label{eq:proof-schrodinger-scaling-system}
f_0\mathsf Kg_T=1,
\qquad
g_T\mathsf K^*f_0=1.
\end{equation}
The normalization and \(m\leq k\leq M\) give
\begin{equation*}
m\leq\mathsf Kg_T\leq M,
\qquad
M^{-1}\leq f_0\leq m^{-1}.
\end{equation*}
With \(A=\int f_0\,\dd\rho_0\in[M^{-1},m^{-1}]\),
\(mA\leq\mathsf K^*f_0\leq MA\), and therefore
\(m/M\leq g_T\leq M/m\). Continuity on compact spaces follows because the integral operators map bounded functions to continuous functions.

Define \(\Gamma^*=f_0g_T\mathcal R_{0T}\). Equation~\eqref{eq:proof-schrodinger-scaling-system} gives the prescribed marginals. For any \(\Gamma\in\Pi(\rho_0,\rho_T)\) with finite entropy,
\begin{equation*}
\KL(\Gamma\|\mathcal R_{0T})
=\KL(\Gamma\|\Gamma^*)
+\int\log f_0\,\dd\rho_0
+\int\log g_T\,\dd\rho_T.
\end{equation*}
Thus \(\Gamma^*\) is the unique minimizer. A second scaling pair gives a second projective fixed point, hence differs only by reciprocal gauge.

Finally, disintegrate path laws with respect to the observed endpoint map. For \(P\ll Q\),
\begin{equation*}
\KL(P\|Q)=\KL(Y_\#P\|\mathcal R_{0T})
+\int\KL(P^y\|Q^y)\,(Y_\#P)(\dd y).
\end{equation*}
The first term is uniquely minimized by \(\Gamma^*\), and the second by \(P^y=Q^y\). This gives Eq.~(1.2) and path-law uniqueness.

\subsection{Derivation for Lemma~\ref{prop:compact-observed-kernel}}\label{app:proof-compact-kernel}

We first verify the strict positivity of the reference group transition density for the kinetic Ornstein--Uhlenbeck dynamics, and then use compactness to obtain the uniform endpoint-kernel bounds required by Assumption~3.1.

\paragraph{H\"ormander bracket condition and smooth state density}
Let
\[
    \mu(\dd g\dd\bm\xi)=\dd g\,\dd\bm\xi
\]
be normalized Haar measure on \(\G\) times Lebesgue measure on \(\g\simeq\R^d\), and let
\(p_t(x,x')\) denote the state transition density of the reference process with respect to \(\mu\), whenever it exists. Under the Ornstein--Uhlenbeck specialization (3.1), write
\[
    V_0=A-\gamma\bm\xi\cdot\nabla_{\bm\xi},
    \qquad
    V_i=\partial_{\xi_i},\quad i=1,\ldots,d.
\]
The constant factor \(\alpha=\sqrt{2\gamma}>0\) in the diffusion fields is immaterial for the generated Lie algebra. Since
\begin{equation}\label{eq:proof-hormander-bracket}
    [V_i,V_0]
    =\widetilde E_i-\gamma\partial_{\xi_i},
    \qquad i=1,\ldots,d,
\end{equation}
we have
\[
    \partial_{\xi_i}=V_i,
    \qquad
    \widetilde E_i=[V_i,V_0]+\gamma V_i.
\]
Hence
\begin{equation*}
    \operatorname{span}\{V_i,[V_i,V_0]:i=1,\ldots,d\}
    =T_{(g,\bm\xi)}(\G\times\g)
\end{equation*}
for every \((g,\bm\xi)\in\G\times\g\). Thus the parabolic H\"ormander bracket condition holds globally. The coefficients are smooth, the group coordinate remains on the compact manifold \(\G\), and the velocity drift is linear, so the reference diffusion is nonexplosive. H\"ormander hypoellipticity therefore gives, for every \(t>0\), a transition density \(p_t(x,x')\) that is smooth in the terminal state and jointly continuous in the state variables.

\paragraph{Fixed-time controllability}
Because the diffusion fields are constant in the velocity variables, the It\^o--Stratonovich correction vanishes. The deterministic control system associated with the support theorem is therefore
\begin{equation}\label{eq:proof-associated-control-system}
    \dot g_s=T_eL_{g_s}(\bm\xi_s),
    \qquad
    \dot{\bm\xi}_s=-\gamma\bm\xi_s+\alpha\bm u_s,
    \qquad
    \bm u\in L^2([0,t];\g).
\end{equation}
Fix \(t>0\), an initial state \(x_0=(g_0,\bm\xi_0)\), and an arbitrary terminal state \(x_1=(g_1,\bm\xi_1)\). Since \(\G\) is connected, it is smoothly path connected. Moreover, one may choose a \(C^2\) curve \(c:[0,t]\to\G\) satisfying
\begin{equation}\label{eq:proof-prescribed-endpoint-jets}
    c(0)=g_0,
    \qquad
    c(t)=g_1,
    \qquad
    \dot c(0)=T_eL_{g_0}(\bm\xi_0),
    \qquad
    \dot c(t)=T_eL_{g_1}(\bm\xi_1).
\end{equation}
Indeed, for sufficiently small \(\delta>0\), use the short exponential arcs
\[
    c_-(s)=g_0\Exp(s\bm\xi_0),\qquad 0\le s\le\delta,
\]
and
\[
    c_+(s)=g_1\Exp((s-t)\bm\xi_1),\qquad t-\delta\le s\le t.
\]
Their inner endpoints lie in the same connected component of \(\G\), hence can be joined by a smooth interior curve. Smoothing the two interior concatenation points while keeping the endpoint arcs fixed gives a \(C^2\) curve satisfying Eq.~\eqref{eq:proof-prescribed-endpoint-jets}.

Define its left logarithmic velocity by
\begin{equation*}
    \bm\xi_c(s)
    :=T_{c(s)}L_{c(s)^{-1}}\dot c(s)\in\g.
\end{equation*}
Then \(\bm\xi_c\in C^1([0,t];\g)\),
\(\bm\xi_c(0)=\bm\xi_0\), and
\(\bm\xi_c(t)=\bm\xi_1\). Set
\begin{equation*}
    \bm u_c(s)
    :=\frac{1}{\alpha}\bigl(\dot{\bm\xi}_c(s)+\gamma\bm\xi_c(s)\bigr).
\end{equation*}
Then \(\bm u_c\in L^2([0,t];\g)\), and the pair
\((c(s),\bm\xi_c(s))\) solves Eq.~\eqref{eq:proof-associated-control-system} exactly. Consequently every state \(x_1\in\G\times\g\) is reachable from every state \(x_0\in\G\times\g\) at every prescribed time \(t>0\). In particular, the fixed-time reachable set satisfies
\begin{equation*}
    \operatorname{Reach}_t(x_0)=\G\times\g.
\end{equation*}

We also verify that the steering control can be taken to be a regular control for the endpoint map. Let
\(\mathcal E_t:L^2([0,t];\g)\to\G\times\g\) denote the endpoint map of Eq.~\eqref{eq:proof-associated-control-system} with initial state \(x_0\). Linearize around the control \(\bm u_c\) and left-trivialize the group perturbation. Writing \(\bm\eta_s\in\g\) for the group perturbation, \(\bm\zeta_s\in\g\) for the velocity perturbation, and \(\bm v_s\in\g\) for the control variation, the variational equation is
\begin{equation}\label{eq:proof-linearized-control}
    \dot{\bm\eta}_s
    =[\bm\eta_s,\bm\xi_c(s)]+\bm\zeta_s,
    \qquad
    \dot{\bm\zeta}_s
    =-\gamma\bm\zeta_s+\alpha\bm v_s,
    \qquad
    (\bm\eta_0,\bm\zeta_0)=(0,0).
\end{equation}
Thus \(D\mathcal E_t(\bm u_c)\bm v=(\bm\eta_t,\bm\zeta_t)\). To show that this derivative is onto, suppose a terminal covector
\((\bm a_t,\bm b_t)\in\g^*\times\g^*\) annihilates its range. Let
\((\bm a_s,\bm b_s)\) solve the adjoint of Eq.~\eqref{eq:proof-linearized-control} backward from this terminal covector. Duality for the linear control system gives
\begin{equation*}
    0
    =\left\langle(\bm a_t,\bm b_t),(\bm\eta_t,\bm\zeta_t)\right\rangle
    =\alpha\int_0^t\langle\bm b_s,\bm v_s\rangle\,\dd s
    \qquad
    \text{for every }\bm v\in L^2([0,t];\g).
\end{equation*}
Hence \(\bm b_s=0\) for every \(s\). The second block of the adjoint equation is
\begin{equation*}
    -\dot{\bm b}_s=\bm a_s-\gamma\bm b_s,
\end{equation*}
so \(\bm a_s=0\) as well. Therefore the annihilator of
\(\operatorname{Ran}D\mathcal E_t(\bm u_c)\) is trivial, and
\begin{equation}\label{eq:proof-endpoint-map-submersion}
    D\mathcal E_t(\bm u_c):L^2([0,t];\g)\longrightarrow
    T_{x_1}(\G\times\g)
    \quad\text{is surjective}.
\end{equation}
Thus every terminal state is reached by a control at which the fixed-time endpoint map is a submersion.

\paragraph{Strict positivity of the hypoelliptic kernel}
The Stroock--Varadhan support theorem identifies the support of the diffusion with the closure of trajectories of the associated control system, while the standard Ben Arous--L\'eandre positivity criterion states that, under the H\"ormander condition, a point reached by a control at which the endpoint map is a submersion has strictly positive transition density. The global bracket condition in Eq.~\eqref{eq:proof-hormander-bracket}, the exact steering construction above, and the submersion property in Eq.~\eqref{eq:proof-endpoint-map-submersion} therefore give
\begin{equation}\label{eq:proof-state-kernel-positive}
    p_t(x_0,x_1)>0
    \qquad
    \text{for every }t>0\text{ and every }x_0,x_1\in\G\times\g.
\end{equation}
Thus there is no smaller communicating component for the kinetic Ornstein--Uhlenbeck process on a connected \(\G\).

Let \(Q_0(\dd\bm\xi_0\mid g_0)\) be a regular conditional law of the initial velocity given the initial group coordinate. The observed group transition density is obtained by integrating out both latent velocities:
\begin{equation}\label{eq:proof-observed-group-kernel}
    r_t^Q(g_1\mid g_0)
    =\int_{\g}Q_0(\dd\bm\xi_0\mid g_0)
      \int_{\g}
      p_t\bigl((g_0,\bm\xi_0),(g_1,\bm\eta)\bigr)\,\dd\bm\eta.
\end{equation}
The integrand in Eq.~\eqref{eq:proof-observed-group-kernel} is strictly positive by Eq.~\eqref{eq:proof-state-kernel-positive}; hence
\begin{equation}\label{eq:proof-observed-group-positive}
    r_t^Q(g_1\mid g_0)>0
    \qquad
    \text{for all }t>0\text{ and all }g_0,g_1\in\G.
\end{equation}
The continuity of the group densities is part of the hypotheses of Lemma~\ref{prop:compact-observed-kernel}. Therefore, because \(\G\times\G\) is compact, Eq.~\eqref{eq:proof-observed-group-positive} upgrades to a uniform positive lower bound for the reference group transition density.

\paragraph{Uniform endpoint-kernel bounds}
For group observations write the reference group marginal and conditional transition densities as
\(r_0^Q(g_0)\) and \(r_T^Q(g_T\mid g_0)\), and the requested endpoint densities as
\(r_0(g_0)\) and \(r_T(g_T)\), all with respect to normalized Haar measure. Then
\begin{equation*}
    \mathcal R_{0T}(\dd g_0\dd g_T)
    =r_0^Q(g_0)r_T^Q(g_T\mid g_0)\,\dd g_0\dd g_T,
\end{equation*}
whereas
\begin{equation*}
    (\rho_0\otimes\rho_T)(\dd g_0\dd g_T)
    =r_0(g_0)r_T(g_T)\,\dd g_0\dd g_T.
\end{equation*}
Hence
\begin{equation}\label{eq:proof-group-kernel-ratio}
    k(g_0,g_T)
    =\frac{r_0^Q(g_0)r_T^Q(g_T\mid g_0)}{r_0(g_0)r_T(g_T)}.
\end{equation}
By continuity, strict positivity, and compactness there are constants
\(0<a_0\le A_0\), \(0<a_T\le A_T\), \(0<b_0\le B_0\), and \(0<b_T\le B_T\) such that
\begin{equation*}
    a_0\le r_0^Q\le A_0,
    \qquad
    a_T\le r_T^Q\le A_T,
    \qquad
    b_0\le r_0\le B_0,
    \qquad
    b_T\le r_T\le B_T.
\end{equation*}
Substituting these bounds into Eq.~\eqref{eq:proof-group-kernel-ratio} gives
\begin{equation*}
    0<\frac{a_0a_T}{B_0B_T}
    \le k(g_0,g_T)
    \le\frac{A_0A_T}{b_0b_T}<\infty.
\end{equation*}
Thus \(\mathcal R_{0T}\) and \(\rho_0\otimes\rho_T\) are mutually absolutely continuous and the Radon--Nikodym derivative \(k\) satisfies the two-sided uniform bound in Assumption~3.1. This completes the proof of Lemma~\ref{prop:compact-observed-kernel}.

\subsection{Proof of Lemma~\ref{lem:canonical-lift}}\label{app:proof-canonical-lift}
For any state law \(\pi(\dd g\dd\bm\xi)=\rho(\dd g)\pi(\dd\bm\xi\mid g)\), the entropy chain rule gives
\begin{equation*}
\KL(\pi\|Q_i)=\KL(\rho\|Q_i^\G)
+\int \KL\!\left(\pi(\cdot\mid g)\|Q_i(\cdot\mid g)\right)\rho(\dd g).
\end{equation*}
The second term is nonnegative and vanishes uniquely when \(\pi(\dd\bm\xi\mid g)=Q_i(\dd\bm\xi\mid g)\) for \(\rho\)-almost every \(g\), proving the claim.

\subsection{Proof of Theorem~3.3}

Let
\(F_0=f_0(\mathcal O_0(X_0))\) and
\(G_T=g_T(\mathcal O_T(X_T))\). By Theorem~3.2,
\begin{equation*}
\frac{\dd P^{\rm SB}}{\dd Q}=F_0G_T.
\end{equation*}
For any bounded measurable test function \(\psi\), conditioning on \(X_t\) and using the Markov property gives
\begin{align*}
\E_{P^{\rm SB}}[\psi(X_t)]
&=\E_Q[F_0G_T\psi(X_t)]\\
&=\E_Q\!\left[
\E_Q[F_0\mid X_t]\,
\E_Q[G_T\mid X_t]\,
\psi(X_t)\right]\\
&=\int \widehat\varphi_t(x)\varphi_t(x)\psi(x)\,Q_t(\dd x).
\end{align*}
This proves Eq.~(3.2). The terminal factor
\(\varphi_t=P_{t,T}(g_T\circ\mathcal O_T)\) is a backward semigroup evaluation, hence
\begin{equation*}
(\partial_t+L_t)\varphi_t=0,
\qquad
\varphi_T=g_T\circ\mathcal O_T.
\end{equation*}
For a smooth test function \(F\), expand the Doob transform:
\begin{align*}
\varphi_t^{-1}(\partial_t+L_t)(\varphi_tF)
&=\varphi_t^{-1}F(\partial_t+L_t)\varphi_t
   +L_tF
   +\alpha^2\inner{\nabla_{\bm\xi}\log\varphi_t}{\nabla_{\bm\xi}F}\\
&=L_tF
   +\alpha^2\inner{\nabla_{\bm\xi}\log\varphi_t}{\nabla_{\bm\xi}F}.
\end{align*}
Thus only the forced \(\bm\xi\)-drift changes, from \(\bm b_t\) to
\(\bm b_t+\alpha^2\nabla_{\bm\xi}\log\varphi_t\). Because the paper parameterizes the drift correction as \(\alpha\bm u_t\), the normalized control is
\begin{equation*}
\bm u_t^*=\alpha\nabla_{\bm\xi}\log\varphi_t,
\end{equation*}
which proves Eqs.~(1.4) and~(1.5); substituting \(\alpha=\sqrt{2\gamma}\) gives the corresponding Ornstein--Uhlenbeck specialization.

Assume now that \(Q_t=q_t\mu\). From Eq.~(3.2),
\begin{equation*}
p_t^{\rm SB}=q_t\widehat\varphi_t\varphi_t=h_t\widehat h_t.
\end{equation*}
The backward equation for \(h_t=\varphi_t\) was established above. To obtain the forward equation for \(\widehat h_t=q_t\widehat\varphi_t\), test against a compactly supported smooth function \(F\):
\begin{align*}
\frac{\dd}{\dd t}\int F\widehat h_t\,\dd\mu
&=\frac{\dd}{\dd t}\E_Q[F(X_t)F_0]\\
&=\E_Q[(L_tF)(X_t)F_0]
 =\int (L_tF)\widehat h_t\,\dd\mu\\
&=\int F L_t^\dagger\widehat h_t\,\dd\mu.
\end{align*}
Hence \(\partial_t\widehat h_t=L_t^\dagger\widehat h_t\) in the weak sense, and the assumed regularity upgrades it to the displayed classical identity. At \(t=0\), conditioning is trivial and yields
\(\widehat h_0=q_0(f_0\circ\mathcal O_0)\); at \(t=T\),
\(h_T=g_T\circ\mathcal O_T\). This proves Eqs.~(3.3) and~(3.4).

\subsection{Proof of Lemma~\ref{lem:control-kl-stability}}\label{app:proof-control-kl}
The relative-entropy form of Girsanov's theorem gives
\begin{equation*}
\KL(P^{\bm u}\|P^{\bm v})
=\frac12\E_{P^{\bm u}}\int_0^T\norm{\bm u_t-\bm v_t}^2\,\dd t.
\end{equation*}
Pinsker's inequality and contraction of total variation under the evaluation map \(\omega\mapsto\omega_t\) yield the marginal bound in Lemma~\ref{lem:control-kl-stability}.

\subsection{Derivation for Lemma~\ref{thm:residual-stability}}

Let \(\ell=\log(\widetilde h/h)\) and write
\(\bm u^*=\alpha\nabla_{\bm\xi}\log h\),
\(\widetilde{\bm u}=\alpha\nabla_{\bm\xi}\log\widetilde h\). Then
\begin{equation*}
\alpha\nabla_{\bm\xi}\ell
=\widetilde{\bm u}-\bm u^*.
\end{equation*}
Using
\((\partial_t+L_t)h=0\) and
\((\partial_t+L_t)\widetilde h=r_{\rm K}\widetilde h\), the logarithmic chain rule for a diffusion with covariance \(\alpha^2I\) gives
\begin{align*}
(\partial_t+L_t)\log h
&=-\frac{\alpha^2}{2}\norm{\nabla_{\bm\xi}\log h}^2,\\
(\partial_t+L_t)\log\widetilde h
&=r_{\rm K}
 -\frac{\alpha^2}{2}\norm{\nabla_{\bm\xi}\log\widetilde h}^2.
\end{align*}
Subtracting,
\begin{equation*}
(\partial_t+L_t)\ell
=r_{\rm K}
-\frac12\norm{\widetilde{\bm u}}^2
+\frac12\norm{\bm u^*}^2.
\end{equation*}
Under the \(\widetilde{\bm u}\)-controlled law, the generator is
\(L_t+\alpha\widetilde{\bm u}\cdot\nabla_{\bm\xi}\). Therefore
\begin{align*}
(\partial_t+L_t
 +\alpha\widetilde{\bm u}\cdot\nabla_{\bm\xi})\ell
&=r_{\rm K}
 -\frac12\norm{\widetilde{\bm u}}^2
 +\frac12\norm{\bm u^*}^2
 +\inner{\widetilde{\bm u}}{\widetilde{\bm u}-\bm u^*}\\
&=r_{\rm K}
 +\frac12\norm{\widetilde{\bm u}-\bm u^*}^2.
\end{align*}
It\^o's formula thus yields, after localization,
\begin{equation*}
\dd\ell_t(X_t)
=\left[r_{\rm K}(t,X_t)
 +\frac12\norm{\widetilde{\bm u}_t-\bm u_t^*}^2\right]\dd t
+\inner{\widetilde{\bm u}_t-\bm u_t^*}{\dd\mathbf W_t}.
\end{equation*}
Integrating and taking expectations makes the localized martingale term vanish. Hence
\begin{equation*}
\frac12\E_{P^{\widetilde{\bm u}}}
\int_0^T\norm{\widetilde{\bm u}_t-\bm u_t^*}^2\,\dd t
=\E_{P^{\widetilde{\bm u}}}
\left[\ell_T(X_T)-\ell_0(X_0)-\int_0^T r_{\rm K}(t,X_t)\,\dd t\right].
\end{equation*}
By the Girsanov identity in Lemma~\ref{lem:control-kl-stability}, the left side equals
\(\KL(P^{\widetilde{\bm u}}\|P^{\rm SB})\), proving Eq.~\eqref{eq:exact-residual-identity}. Applying the triangle inequality inside the expectation gives Eq.~\eqref{eq:residual-kl-bound}.

\subsection{Derivation for Lemma~\ref{prop:wrapped-kinetic-kernel}}

Work first on the Euclidean lift \((\bm z_t,\bm\xi_t)\in\R^m\times\R^m\), with
\(\bm z_0=0\). Solving the Ornstein--Uhlenbeck equation gives
\begin{equation*}
\bm\xi_t=a_t\bm\xi_0
+\sqrt{2\gamma}\int_0^t e^{-\gamma(t-s)}\,\dd\mathbf W_s,
\qquad a_t=e^{-\gamma t}.
\end{equation*}
Integrating once more,
\begin{align*}
\bm z_t
&=\int_0^t\bm\xi_s\,\dd s\\
&=r_t\bm\xi_0
+\sqrt{2\gamma}\int_0^t
\left(\int_s^t e^{-\gamma(u-s)}\,\dd u\right)\dd\mathbf W_s\\
&=r_t\bm\xi_0
+\sqrt{\frac{2}{\gamma}}
\int_0^t\bigl(1-e^{-\gamma(t-s)}\bigr)\,\dd\mathbf W_s,
\end{align*}
where \(r_t=(1-a_t)/\gamma\). Thus \((\bm z_t,\bm\xi_t)\) is jointly Gaussian with means
\((r_t\bm\xi_0,a_t\bm\xi_0)\). By It\^o isometry,
\begin{align*}
\operatorname{Var}(\bm\xi_t)
&=2\gamma\int_0^t e^{-2\gamma(t-s)}\,\dd s\,I
=(1-a_t^2)I,\\
\operatorname{Cov}(\bm z_t,\bm\xi_t)
&=2\int_0^t
\bigl(1-e^{-\gamma(t-s)}\bigr)e^{-\gamma(t-s)}\,\dd s\,I\\
&=\frac{(1-a_t)^2}{\gamma}I,\\
\operatorname{Var}(\bm z_t)
&=\frac{2}{\gamma}\int_0^t
\bigl(1-e^{-\gamma(t-s)}\bigr)^2\,\dd s\,I\\
&=\left[
\frac{2t}{\gamma}
-\frac{4(1-a_t)}{\gamma^2}
+\frac{1-a_t^2}{\gamma^2}\right]I.
\end{align*}
These are exactly \(\sigma_{\xi,t}^2\), \(\sigma_{z\xi,t}\), and \(\sigma_{z,t}^2\).

Passing from \(\R^m\) to \(\T^m=\R^m/\Lambda\) identifies lifts differing by \(\lambda\in\Lambda\). Therefore the torus transition density is the periodization of the Euclidean Gaussian. If normalized Haar measure is \(\dd g=\Vol(\Lambda)^{-1}\dd z\) on a fundamental cell, conversion from Lebesgue density to Haar density contributes the factor \(\Vol(\Lambda)\), giving Eq.~\eqref{eq:wrapped-kernel}. Replacing the chosen lift \(z\) by \(z+\lambda_0\) merely reindexes the lattice sum, so the result is lift independent. Finally, integrating the joint Gaussian over \(\bm\xi\) leaves its \(\bm z\)-marginal and yields Eq.~\eqref{eq:wrapped-position-kernel}.

\subsection{Proof of Proposition~\ref{thm:wkbc-stability}}

Under the calibrated bridge law, the Markov property and the terminal factorization give
\begin{equation*}
P^{\rm SB}(Y_T\in\dd y\mid X_t=x)
=
\frac{q_{T-t}^{\mathcal O_T}(y\mid x)e^{\beta^*(y)}\nu_T(\dd y)}
{h_t^{\beta^*}(x)}.
\end{equation*}
Differentiation under the endpoint integral therefore yields
\begin{equation*}
\E_{P^{\rm SB}}[S_{t,T}(X_t,Y_T)\mid X_t=x]
=\nabla_{\bm\xi}\log h_t^{\beta^*}(x)
=\bm s^*(t,x).
\end{equation*}
Applying the conditional-variance identity at each \((t,x)\), followed by integration against the uniform time density and \(P_t^{\rm SB}\), proves Eq.~\eqref{eq:wkbc-excess-risk}.

The normalized Doob controls of \(P^{\rm SB}\) and \(P^{\bm s,\pi_0}\) are
\(\sqrt{2\gamma}\,\bm s^*\) and \(\sqrt{2\gamma}\,\bm s\), respectively.  Applying the entropy chain rule at time zero and then Lemma~\ref{lem:control-kl-stability} in the reverse orientation gives
\begin{align*}
\KL(P^{\rm SB}\|P^{\bm s,\pi_0})
&=\KL(P_0^{\rm SB}\|\pi_0)
+\frac12\E_{P^{\rm SB}}\int_0^T
\norm{\sqrt{2\gamma}(\bm s^*-
\bm s)}^2\,\dd t\\
&=\KL(P_0^{\rm SB}\|\pi_0)
+\gamma T\bigl[\mathcal J_{\rm WKBC}^{\rm br}(\bm s)
-\mathcal J_{\rm WKBC}^{\rm br}(\bm s^*)\bigr],
\end{align*}
which proves Eq.~\eqref{eq:wkbc-population-path-error}.  If the initial laws agree, Pinsker's inequality together with \(d_{\rm BL}\le2\,\mathrm{TV}\) gives the displayed bounded-Lipschitz estimate.

For the numerical statement, apply the same argument with the numerical calibrated bridge \(P_N^*\), its control \(\bm s_N^*\), and its own initial law. This yields
\begin{equation*}
d_{\rm BL}(P_N^{\bm s},P_N^*)
\le
\sqrt{2\gamma T\bigl[
\mathcal J_{{\rm WKBC},N}^{\rm br}(\bm s)
-\mathcal J_{{\rm WKBC},N}^{\rm br}(\bm s_N^*)
\bigr]}.
\end{equation*}
The triangle inequality with \(P^{\rm SB}\) proves Eq.~\eqref{eq:wkbc-total-error}. The common positive kernel bounds make the normalized Sinkhorn maps uniformly contractive in Hilbert metric. If the discrete/quadrature kernel operators \(\mathsf K_N,\mathsf K_N^*\) converge uniformly on the normalized bounded cone, the fixed-point perturbation estimate gives
\begin{equation*}
d_{\rm H}(g_{T,N},g_T)
\leq\frac{1}{1-\tau^2}
\sup_g d_{\rm H}(\mathsf S_Ng,\mathsf Sg)+o(1),
\end{equation*}
so the normalized factors converge uniformly; the same follows for \(f_{0,N}\). Lattice and quadrature assumptions then give \(h_{t,N}\to h_t\) and \(\nabla_{\bm\xi}h_{t,N}\to\nabla_{\bm\xi}h_t\) on preterminal strips. Positivity converts these to logarithmic-gradient convergence, while normalization and dominated convergence give convergence of the numerical initial laws. Standard stability of the corresponding Doob laws then yields \(d_{\rm BL}(P_N^*,P^{\rm SB})\to0\). This proves the final assertion.

\subsection{Proof of Lemma~\ref{lem:endpoint-disintegration}}\label{app:proof-endpoint-disintegration}
The entropy chain rule under the endpoint map \(Y\) gives
\begin{equation*}
\KL(P\|Q)=\KL(Y_\#P\|Y_\#Q)
+\int\KL(P^y\|Q^y)\,(Y_\#P)(\dd y).
\end{equation*}
Since \(Y_\#Q=\mathcal R_{0T}\), the displayed identity in Lemma~\ref{lem:endpoint-disintegration} follows. For fixed endpoint law \(\Gamma\), the conditional term is minimized uniquely by \(P^y=Q^y\) for \(\Gamma\)-almost every \(y\), hence \(P_\Gamma=\int Q^y\Gamma(\dd y)\).

\subsection{Proof of Theorem~3.5}

Let \(\Gamma^*=f_0g_T\mathcal R_{0T}\). For any bounded \(\beta\),
\begin{align*}
\KL(\Gamma^*\|\Gamma^\beta)
&=\int\log\frac{f_0(y_0)g_T(y_T)}
{Z(\beta)^{-1}e^{\beta_0(y_0)+\beta_T(y_T)}}\,\Gamma^*(\dd y_0\dd y_T)\\
&=\int\log f_0\,\dd\rho_0+\int\log g_T\,\dd\rho_T
-\mathscr D(\beta).
\end{align*}
The first two terms equal \(\mathscr D(\beta^*)\) for any normalized representative inducing \(\Gamma^*\), which proves the first equality in Eq.~(3.11). Lemma~\ref{lem:endpoint-disintegration} proves the path-space equality.

The functional is concave. For a perturbation \(\psi=(\psi_0,\psi_T)\),
\begin{equation*}
\frac{\dd^2}{\dd s^2}\mathscr D(\beta+s\psi)\bigg|_{s=0}
=-\operatorname{Var}_{\Gamma^\beta}[\psi_0(Y_0)+\psi_T(Y_T)].
\end{equation*}
If this variance is zero, equivalence
\(\Gamma^\beta\sim\rho_0\otimes\rho_T\) implies
\(\psi_0(y_0)+\psi_T(y_T)\) is constant on the product almost everywhere. Fubini's theorem then implies that each \(\psi_i\) is constant. On the doubly centered subspace both constants vanish, so the dual is strictly concave there. Existence follows from Theorem~3.2: choose constants \(c_0,c_T>0\) such that \(\log(c_0f_0)\) and \(\log(c_Tg_T)\) are centered. Their normalized tilt is \(\Gamma^*\), and strict concavity gives uniqueness. Equality of the two factorized densities implies
\(e^{\beta_0^*}/f_0=c_0\) and
\(e^{\beta_T^*}/g_T=c_T\) almost everywhere, with \(c_0c_T=Z(\beta^*)\).

\subsection{Derivation for Proposition~\ref{prop:rccbm-ipf}}

Let \(\Gamma\ll\mathcal R_{0T}\) have density \(a(y_0,y_T)\). Its first marginal relative to \(\rho_0\) is
\begin{equation*}
m_0(y_0)=\int a(y_0,y_T)k(y_0,y_T)\,\rho_T(\dd y_T),
\end{equation*}
and the KL projection onto the affine set of couplings with first marginal \(\rho_0\) is obtained by multiplying the density by \(m_0(y_0)^{-1}\). Indeed, if \(\Gamma'\) has first marginal \(\rho_0\), the entropy chain rule with respect to the projection \((y_0,y_T)\mapsto y_0\) gives
\begin{equation*}
\KL(\Gamma'\|\Gamma)
=\KL(\Gamma'\|\mathsf P_0\Gamma)
+\KL(\mathsf P_0\Gamma\|\Gamma),
\end{equation*}
so \(\mathsf P_0\Gamma\) is the unique information projection. The terminal projection is analogous.

Starting from a factorized density \(f_0^{(n)}(y_0)g_T^{(n)}(y_T)\) relative to \(\mathcal R_{0T}\), its first marginal density is
\begin{equation*}
f_0^{(n)}(y_0)(\mathsf K g_T^{(n)})(y_0).
\end{equation*}
Renormalizing it to one cancels \(f_0^{(n)}\) and gives
\begin{equation*}
f_0^{(n+1)}=(\mathsf K g_T^{(n)})^{-1}.
\end{equation*}
The resulting terminal marginal density is
\(g_T^{(n)}\mathsf K^*f_0^{(n+1)}\), and the second information projection gives
\begin{equation*}
g_T^{(n+1)}=(\mathsf K^*f_0^{(n+1)})^{-1}.
\end{equation*}
This proves Eq.~\eqref{eq:rccbm-ipf-updates}. Under \(m\le k\le M\), the normalized scaling map is the same Hilbert-metric contraction used in Theorem~3.2; therefore the projective class converges geometrically to the unique Schr\"odinger scaling pair. Uniform positive bounds convert projective convergence into uniform convergence after gauge normalization, and dominated convergence then yields convergence of the endpoint KL gap. Because the first projection changes the source factor, a terminal-only correction cannot in general reproduce the two-sided optimum.

\subsection{Proof of Proposition~\ref{prop:rcm-endpoint-calibration-consistency}}

By Theorem~3.5,
\begin{equation*}
\varepsilon_{{\rm end},N}
=\mathscr D(\beta^*)-\mathscr D(\widehat\beta_N).
\end{equation*}
Let \(\beta_N^\circ\in\mathcal B_N\) be the sieve approximant from Assumption~\ref{ass:rcm-calibration}. Adding and subtracting the empirical dual gives
\begin{align*}
\mathscr D(\beta^*)-\mathscr D(\widehat\beta_N)
&=\bigl[\mathscr D(\beta^*)-\mathscr D(\beta_N^\circ)\bigr]
 +\bigl[\mathscr D(\beta_N^\circ)-\widehat{\mathscr D}_N(\beta_N^\circ)\bigr]\\
&\quad+\bigl[\widehat{\mathscr D}_N(\beta_N^\circ)-\widehat{\mathscr D}_N(\widehat\beta_N)\bigr]
 +\bigl[\widehat{\mathscr D}_N(\widehat\beta_N)-\mathscr D(\widehat\beta_N)\bigr].
\end{align*}
The first term is at most \(\varepsilon_{{\rm end,app},N}\); the second and fourth are each bounded above by \(\delta_{{\rm end},N}\) from Eq.~\eqref{eq:endpoint-uniform-dual-error}; and Eq.~\eqref{eq:endpoint-optimization-gap} bounds the third by \(\varepsilon_{{\rm end,opt},N}\). This proves Eq.~\eqref{eq:endpoint-calibration-consistency-bound}. The asserted convergence in probability follows from Assumption~\ref{ass:rcm-calibration}.

\subsection{Proof of Theorem~3.6}

The density \(\dd P^\beta/\dd Q\) is a function of \((X_0,Y_T)\) only. It is therefore constant on each fiber \(Z_c=z\), and the conditional path law of \(P^\beta\) on that fiber equals \(Q^z\). This proves Eq.~(3.12).

Using the Markov property of the reference law and the endpoint tilt defining \(P^\beta\), integrating the terminal factor gives the backward factor in Eq.~(3.13), and the corresponding Doob transform gives Eq.~(3.14). Given \(X_t=x\), the conditional density of \(Y_T=y\) under \(P^\beta\) is
\begin{equation*}
\frac{r_{t,T}^{\mathcal O_T}(y\mid x)e^{\beta_T(y)}}{h_t^\beta(x)}\,\nu_T(\dd y).
\end{equation*}
The point-conditioned reference bridge has control in Eq.~(3.15) on every strip before \(T\). Therefore
\begin{align*}
\E[\bm v_t^{Z_c}(X_t)\mid X_t=x]
&=\frac{\alpha}{h_t^\beta(x)}
\int \nabla_{\bm\xi}\log r_{t,T}^{\mathcal O_T}(y\mid x)
 r_{t,T}^{\mathcal O_T}(y\mid x)e^{\beta_T(y)}\,\nu_T(\dd y)\\
&=\alpha\nabla_{\bm\xi}\log h_t^\beta(x).
\end{align*}
The domination assumption justifies differentiation under the integral. Finally, on \([0,T-\tau]\), the random teacher is square-integrable. Orthogonal projection in the Hilbert space generated by the canonical joint law gives
\begin{equation*}
\E\norm{\bm u-\bm v}^2
=\E\norm{\bm u-\E[\bm v\mid X_t]}^2
+\E\norm{\bm v-\E[\bm v\mid X_t]}^2,
\end{equation*}
which yields Eq.~(3.16). The restriction to preterminal strips is essential: exact point-pinned controls need not have finite energy on the closed interval.

\subsection{Proof of Proposition~\ref{thm:rccbm-condsoc-kl}}

By Eq.~(3.17),
\begin{equation*}
\log\frac{\dd Q^{z,\varepsilon}}{\dd Q^{x_0}}
=\log\kappa_\varepsilon(y,\mathcal O_T(X_T))-\log Z_{z,\varepsilon}
=-\ell_{\varepsilon,y}(X_T)-\log Z_{z,\varepsilon}.
\end{equation*}
For any admissible controlled law \(P^{z,\bm v}\ll Q^{x_0}\), the Radon--Nikodym chain rule gives
\begin{align*}
\KL(P^{z,\bm v}\|Q^{z,\varepsilon})
&=\E_{P^{z,\bm v}}
\log\frac{\dd P^{z,\bm v}}{\dd Q^{x_0}}
-\E_{P^{z,\bm v}}
\log\frac{\dd Q^{z,\varepsilon}}{\dd Q^{x_0}}\\
&=\KL(P^{z,\bm v}\|Q^{x_0})
+\E_{P^{z,\bm v}}\ell_{\varepsilon,y}(X_T)
+\log Z_{z,\varepsilon}.
\end{align*}
Since the control changes only the forced drift from \(\bm b_t\) to \(\bm b_t+\alpha\bm v_t\), Girsanov's theorem yields
\begin{equation*}
\KL(P^{z,\bm v}\|Q^{x_0})
=\frac12\E_{P^{z,\bm v}}
\int_0^T\norm{\bm v_t}^2\,\dd t.
\end{equation*}
Substitution gives Eq.~\eqref{eq:rccbm-condsoc-kl-identity}. Because relative entropy is nonnegative and vanishes only when the two laws agree, the unique law-level minimizer is \(Q^{z,\varepsilon}\).

Let \(\bm v^{z,\varepsilon}\) denote its Doob control. Both \(P^{z,\bm v}\) and the Doob realization of \(Q^{z,\varepsilon}\) start at the same \(x_0\) and have the same diffusion coefficient. A second Girsanov calculation therefore gives
\begin{align*}
\KL(P^{z,\bm v}\|Q^{z,\varepsilon})
&=\frac12\E_{P^{z,\bm v}}\\
&\quad\times\int_0^T
\norm{\bm v_t-\bm v_t^{z,\varepsilon}(X_t)}^2\,\dd t.
\end{align*}
Subtracting the identity at the minimizer \(\bm v^{z,\varepsilon}\) from the identity at \(\bm v\) proves \cref{eq:rccbm-condsoc-control-identity}.

\subsection{Proof of Proposition~3.7}

For fixed \(z=(x_0,y)\), define
\begin{equation*}
h_t^{y,\varepsilon}(x)
=\int r_{t,T}^{\mathcal O_T}(y'\mid x)\kappa_\varepsilon(y,y')\,\nu_T(\dd y').
\end{equation*}
Then \(\bm v_t^{z,\varepsilon}(x)=\alpha\nabla_{\bm\xi}\log h_t^{y,\varepsilon}(x)\). Multiplying Eq.~(3.18) by Eq.~(3.17), the factors \(Z_{z,\varepsilon}\) cancel, and integration in \(y\) gives
\begin{equation*}
\int Q^{z,\varepsilon}\Lambda^{\beta,\varepsilon}(\dd z)
=\frac{e^{\beta_0(\mathcal O_0(X_0))}
 g_{T,\varepsilon}^\beta(\mathcal O_T(X_T))}{Z_\varepsilon(\beta)}Q,
\end{equation*}
proving Eq.~(3.19). This endpoint-factorized law is Markov. Its backward factor is
\begin{equation*}
h_t^{\beta,\varepsilon}(x)
=\int r_{t,T}^{\mathcal O_T}(y'\mid x)
 g_{T,\varepsilon}^\beta(y')\,\nu_T(\dd y')
=\int e^{\beta_T(y)}h_t^{y,\varepsilon}(x)\,\nu_T(\dd y).
\end{equation*}
Given \(X_t=x\), the canonical conditional density of the center \(y\) is proportional to
\(e^{\beta_T(y)}h_t^{y,\varepsilon}(x)\). Hence
\begin{align*}
\E[\bm v_t^{Z,\varepsilon}(X_t)\mid X_t=x]
&=\frac{\alpha\int e^{\beta_T(y)}\nabla_{\bm\xi}h_t^{y,\varepsilon}(x)\,\nu_T(\dd y)}
{h_t^{\beta,\varepsilon}(x)}\\
&=\alpha\nabla_{\bm\xi}\log h_t^{\beta,\varepsilon}(x)
=\bm u_t^{\beta,\varepsilon}(x).
\end{align*}
The conditional-variance identity, integrated against \(\lambda_{\rm bm}\), proves Eq.~(3.20). Positivity and fixed \(\varepsilon>0\) remove the terminal point-pinning singularity.

\subsection{Derivation for Proposition~\ref{prop:rccbm-condsoc-consistency}}

Write \(g_T^\beta=e^{\beta_T}\) and
\begin{equation*}
g_{T,\varepsilon}^\beta(y')
=\int\kappa_\varepsilon(y,y')g_T^\beta(y)\,\nu_T(\dd y).
\end{equation*}
Assumption~\ref{ass:rccbm-mollification-regularity} gives the uniform approximation
\begin{equation}\label{eq:proof-mollifier-uniform}
\delta_\varepsilon
:=\sup_{\beta\in\mathcal B_\infty}\sup_{y'\in\mathsf Y_T}
|g_{T,\varepsilon}^\beta(y')-g_T^\beta(y')|
\leq\delta_\varepsilon^{\rm mol}
\longrightarrow0.
\end{equation}
Uniform boundedness of \(\beta_0\) and \(\beta_T\) implies constants
\(0<c\le C<\infty\) such that all endpoint factors lie in \([c,C]\). Consequently the partition functions satisfy
\begin{equation*}
0<c_Z\le Z_\varepsilon(\beta),Z(\beta)\le C_Z<\infty
\end{equation*}
uniformly in \(\beta\), and Eq.~\eqref{eq:proof-mollifier-uniform} implies
\begin{equation*}
\sup_{\beta\in\mathcal B_\infty}|Z_\varepsilon(\beta)-Z(\beta)|\le C\delta_\varepsilon\to0.
\end{equation*}
The path-law densities relative to \(Q\) are
\begin{equation*}
p_\varepsilon^\beta
=\frac{e^{\beta_0(\mathcal O_0(X_0))}
 g_{T,\varepsilon}^\beta(Y_T)}{Z_\varepsilon(\beta)},
\qquad
p^\beta
=\frac{e^{\beta_0(\mathcal O_0(X_0))}
 g_T^\beta(Y_T)}{Z(\beta)}.
\end{equation*}
The uniform bounds on denominators and numerators therefore yield
\begin{equation*}
\sup_{\beta\in\mathcal B_\infty}\norm{p_\varepsilon^\beta-p^\beta}_{L^\infty(Q)}
\le C'\delta_\varepsilon\to0.
\end{equation*}
Hence
\begin{equation*}
\sup_{\beta\in\mathcal B_\infty}\TV(P^{\beta,\varepsilon},P^\beta)
\le\frac12\sup_{\beta\in\mathcal B_\infty}
\norm{p_\varepsilon^\beta-p^\beta}_{L^1(Q)}\to0,
\end{equation*}
and \(d_{\rm BL}\le2\TV\) proves Eq.~\eqref{eq:rccbm-mollification-convergence}.

For the controls, on any strip \([0,T-\tau]\),
\begin{equation*}
h_t^{\beta,\varepsilon}(x)
=\int r_{t,T}^{\mathcal O_T}(y'\mid x)
 g_{T,\varepsilon}^\beta(y')\,\nu_T(\dd y').
\end{equation*}
The differentiation statement in Assumption~\ref{ass:rccbm-conditional-kernel}, together with the envelope in Eq.~\eqref{eq:mollification-kernel-envelope}, permits dominated differentiation uniformly on compact preterminal sets. Equation~\eqref{eq:proof-mollifier-uniform} then gives local uniform convergence of both
\(h_t^{\beta,\varepsilon}\) and
\(\nabla_{\bm\xi}h_t^{\beta,\varepsilon}\) to their unmollified counterparts. Strict positivity on compact subsets provides a common lower bound, so
\begin{equation*}
\nabla_{\bm\xi}\log h_t^{\beta,\varepsilon}
\longrightarrow
\nabla_{\bm\xi}\log h_t^\beta
\end{equation*}
locally uniformly. Multiplication by \(\alpha\) gives local convergence of the controls. Finally, the assumed uniform integrability of the squared controls permits Vitali's theorem (or dominated convergence under the stronger domination form) to upgrade local convergence to the stated \(L^2\) convergence.

\subsection{Proof of Theorem~3.8}

Let
\(\mathcal R_N(\bm u)=\mathcal L_{\widehat\beta_N,\varepsilon_N}(\bm u)\)
be the exact population mollified matching risk and let
\(\widehat{\mathcal R}_N\) be its importance-weighted empirical counterpart after teacher approximation and clipping. Define
\begin{equation*}
\eta_N
:=\varepsilon_{{\rm gen},N}+\varepsilon_{{\rm teach},N}
+\frac12\bigl(\varepsilon_{{\rm iw},N}+\varepsilon_{{\rm clip},N}\bigr).
\end{equation*}
On the declared uniform-convergence event,
\begin{equation}\label{eq:proof-uniform-risk-event}
\sup_{\bm u\in\mathcal U_N}
|\widehat{\mathcal R}_N(\bm u)-\mathcal R_N(\bm u)|
\le \eta_N.
\end{equation}
All inequalities below are first established on this event. By the constrained-refinement condition in Assumption~\ref{ass:rccbm-teacher-consistency} and the detailed sieve requirement in \cref{app:rccbm-learning-assumptions}, both the pre-refinement control and every accepted refinement checkpoint belong to the same class \(\mathcal U_N\), so Eq.~\eqref{eq:proof-uniform-risk-event} applies to each of them.

Choose \(\bm u_N^\circ\in\mathcal U_N\) such that
\begin{equation}\label{eq:proof-approximation-risk}
\mathcal R_N(\bm u_N^\circ)
-\mathcal R_N(\bm u^{\widehat\beta_N,\varepsilon_N})
\le \varepsilon_{{\rm app},N}.
\end{equation}
If \(\bm u_{N,\rm pre}\in\mathcal U_N\) is the pre-refinement empirical solution, the optimization assumption gives
\begin{equation}\label{eq:proof-pre-empirical-opt}
\widehat{\mathcal R}_N(\bm u_{N,\rm pre})
\le\inf_{\bm u\in\mathcal U_N}\widehat{\mathcal R}_N(\bm u)
+\varepsilon_{{\rm opt},N}
\le\widehat{\mathcal R}_N(\bm u_N^\circ)
+\varepsilon_{{\rm opt},N}.
\end{equation}
Using Eq.~\eqref{eq:proof-uniform-risk-event} at
\(\bm u_{N,\rm pre}\) and \(\bm u_N^\circ\), followed by
Eqs.~\eqref{eq:proof-pre-empirical-opt} and \eqref{eq:proof-approximation-risk}, yields
\begin{align}
&\mathcal R_N(\bm u_{N,\rm pre})
-\mathcal R_N(\bm u^{\widehat\beta_N,\varepsilon_N})
\notag\\
&\quad=
\bigl[\mathcal R_N(\bm u_{N,\rm pre})
      -\widehat{\mathcal R}_N(\bm u_{N,\rm pre})\bigr]
+\bigl[\widehat{\mathcal R}_N(\bm u_{N,\rm pre})
      -\widehat{\mathcal R}_N(\bm u_N^\circ)\bigr]
\notag\\
&\qquad
+\bigl[\widehat{\mathcal R}_N(\bm u_N^\circ)
      -\mathcal R_N(\bm u_N^\circ)\bigr]
+\bigl[\mathcal R_N(\bm u_N^\circ)
      -\mathcal R_N(\bm u^{\widehat\beta_N,\varepsilon_N})\bigr]
\notag\\
&\quad\le
2\eta_N+\varepsilon_{{\rm opt},N}+\varepsilon_{{\rm app},N}.
\label{eq:proof-pre-population-risk}
\end{align}
This is the population excess-risk bound before refinement.

Let \(\bm u_{N,\rm final}:=\bm u_{\theta_N}\in\mathcal U_N\) denote the accepted final checkpoint. The refinement acceptance rule \eqref{eq:rccbm-constrained-refinement}, written in the present notation, is
\begin{equation*}
\widehat{\mathcal R}_N(\bm u_{N,\rm final})
\le
\widehat{\mathcal R}_N(\bm u_{N,\rm pre})
+\delta_{{\rm ref},N}.
\end{equation*}
A second empirical-to-population comparison is necessary at this stage because the acceptance rule is empirical whereas
\(\Delta_{{\rm ctrl},N}\) is defined through the population risk. Applying Eq.~\eqref{eq:proof-uniform-risk-event} separately to the final and pre-refinement controls gives
\begin{align}
\mathcal R_N(\bm u_{N,\rm final})
-\mathcal R_N(\bm u_{N,\rm pre})
&=
\bigl[\mathcal R_N(\bm u_{N,\rm final})
      -\widehat{\mathcal R}_N(\bm u_{N,\rm final})\bigr]
\notag\\
&\quad+
\bigl[\widehat{\mathcal R}_N(\bm u_{N,\rm final})
      -\widehat{\mathcal R}_N(\bm u_{N,\rm pre})\bigr]
\notag\\
&\quad+
\bigl[\widehat{\mathcal R}_N(\bm u_{N,\rm pre})
      -\mathcal R_N(\bm u_{N,\rm pre})\bigr]
\notag\\
&\le 2\eta_N+\delta_{{\rm ref},N}.
\label{eq:proof-refinement-population}
\end{align}
Combining Eqs.~\eqref{eq:proof-pre-population-risk} and
\eqref{eq:proof-refinement-population}, and recalling
\begin{equation*}
\Delta_{{\rm ctrl},N}
=\mathcal R_N(\bm u_{N,\rm final})
-\mathcal R_N(\bm u^{\widehat\beta_N,\varepsilon_N}),
\end{equation*}
we obtain
\begin{align*}
\Delta_{{\rm ctrl},N}
&\le
\varepsilon_{{\rm app},N}+\varepsilon_{{\rm opt},N}
+4\eta_N+\delta_{{\rm ref},N}\\
&=
\varepsilon_{{\rm app},N}+\varepsilon_{{\rm opt},N}
+4\varepsilon_{{\rm gen},N}+4\varepsilon_{{\rm teach},N}
+2\varepsilon_{{\rm iw},N}+2\varepsilon_{{\rm clip},N}
+\delta_{{\rm ref},N},
\end{align*}
which is Eq.~(3.21). Because the uniform-convergence event has probability tending to one under Assumption~\ref{ass:rccbm-teacher-consistency}, this finite-sample inequality holds with probability tending to one. Under the stated assumptions, every component on the right-hand side converges to zero in probability.

We now convert control risk into path-law error. Proposition~3.7 and
\(\lambda_{\rm bm}\ge\underline\lambda\) imply
\begin{equation}\label{eq:proof-control-l2}
\E_{P^{\widehat\beta_N,\varepsilon_N}}
\int_0^T
\norm{\bm u_{\theta_N}-\bm u^{\widehat\beta_N,\varepsilon_N}}^2\,\dd t
\le\frac{\Delta_{{\rm ctrl},N}}{\underline\lambda}.
\end{equation}
Let \(\widetilde P_N^{\theta_N}\) be the continuous controlled law with control
\(\bm u_{\theta_N}\), initialized from the population law
\(P_0^{\widehat\beta_N,\varepsilon_N}\). We apply Girsanov between
\(P^{\widehat\beta_N,\varepsilon_N}\) and
\(\widetilde P_N^{\theta_N}\), which share the same continuous population initial law. By Eq.~\eqref{eq:proof-control-l2},
\begin{equation*}
\KL(P^{\widehat\beta_N,\varepsilon_N}
\|\widetilde P_N^{\theta_N})
\le\frac{\Delta_{{\rm ctrl},N}}{2\underline\lambda}.
\end{equation*}
Pinsker's inequality and \(d_{\rm BL}\le2\TV\) therefore give
\begin{equation}\label{eq:proof-control-bl}
d_{\rm BL}(P^{\widehat\beta_N,\varepsilon_N},
\widetilde P_N^{\theta_N})
\le\sqrt{\frac{\Delta_{{\rm ctrl},N}}{\underline\lambda}}.
\end{equation}

The endpoint dual gap in Theorem~3.5 satisfies
\begin{equation*}
\KL(P^{\rm SB}\|P^{\widehat\beta_N})
=\varepsilon_{{\rm end},N},
\end{equation*}
so again by Pinsker,
\begin{equation*}
d_{\rm BL}(P^{\rm SB},P^{\widehat\beta_N})
\le\sqrt{2\varepsilon_{{\rm end},N}}.
\end{equation*}
By Assumption~\ref{ass:rccbm-mollification-regularity} and Proposition~\ref{prop:rccbm-condsoc-consistency},
\begin{equation*}
d_{\rm BL}(P^{\widehat\beta_N},
P^{\widehat\beta_N,\varepsilon_N})
=\varepsilon_{{\rm mol},N}\longrightarrow0.
\end{equation*}
The initial-law stability assumption yields
\begin{equation*}
d_{\rm BL}(\widetilde P_N^{\theta_N},P_N^{\theta_N})
\le C_{\rm init}\varepsilon_{{\rm init},N},
\end{equation*}
and the numerical integrator contributes
\begin{equation}\label{eq:proof-disc-bl}
d_{\rm BL}(P_N^{\theta_N},P_{N,\Delta t}^{\theta_N})
\le\varepsilon_{{\rm disc},N}.
\end{equation}
Applying the triangle inequality along the chain
\begin{equation*}
P^{\rm SB}
\to P^{\widehat\beta_N}
\to P^{\widehat\beta_N,\varepsilon_N}
\to\widetilde P_N^{\theta_N}
\to P_N^{\theta_N}
\to P_{N,\Delta t}^{\theta_N}
\end{equation*}
and using Eqs.~\eqref{eq:proof-control-bl}--\eqref{eq:proof-disc-bl} gives Eq.~(3.22). Proposition~\ref{prop:rcm-endpoint-calibration-consistency} gives \(\varepsilon_{{\rm end},N}\to0\); Assumption~\ref{ass:rccbm-mollification-regularity} and Proposition~\ref{prop:rccbm-condsoc-consistency} give \(\varepsilon_{{\rm mol},N}\to0\); Eq.~(3.21) together with Assumption~\ref{ass:rccbm-teacher-consistency} gives \(\Delta_{{\rm ctrl},N}\to0\); and Assumption~\ref{ass:rccbm-statistical-stability} supplies the initial-law and discretization limits. Hence every term in Eq.~(3.22) vanishes in probability, proving bounded-Lipschitz path-law convergence.

\subsection{Proof of Proposition~\ref{prop:rcm-exact-ou-geometry}}

We separate geometry preservation, exact integration of the frozen velocity substep, and the path-space error estimate.

\paragraph{Geometry preservation and exact frozen OU update}
For \(SO(3)\), every Lie algebra element \(\Omega\in\mathfrak{so}(3)\) is skew-symmetric. Hence
\begin{equation*}
\exp(\Omega)^\top\exp(\Omega)=I,
\qquad
\det\exp(\Omega)=\exp(\operatorname{tr}\Omega)=1,
\end{equation*}
so \(\exp(\Omega)\in SO(3)\). For \(U(n)\), every \(H\in\mathfrak u(n)\) is skew-Hermitian and therefore
\begin{equation*}
\exp(H)^\dagger\exp(H)=I,
\end{equation*}
so \(\exp(H)\in U(n)\). Consequently, if \(g_k\in\G\), then
\(g_k\Exp_\G(\Delta t\,\bm\xi_{k+1})\in\G\); the reconstruction step therefore preserves the matrix Lie group exactly.

On one interval \([t_k,t_{k+1}]\), freeze
\(\bm u_k=\bm u(t_k,g_k,\bm\xi_k)\). The velocity equation is the affine OU equation
\begin{equation*}
\dd\bm\xi_t
=-\gamma\bm\xi_t\,\dd t
+\alpha\bm u_k\,\dd t
+\alpha\,\dd\mathbf W_t.
\end{equation*}
Variation of constants gives
\begin{equation*}
\bm\xi_{k+1}
=e^{-\gamma\Delta t}\bm\xi_k
+\alpha\frac{1-e^{-\gamma\Delta t}}{\gamma}\bm u_k
+\alpha\int_{t_k}^{t_{k+1}}
 e^{-\gamma(t_{k+1}-s)}\,\dd\mathbf W_s.
\end{equation*}
The stochastic integral is centered Gaussian with covariance
\begin{equation*}
\alpha^2\int_0^{\Delta t}e^{-2\gamma r}\,\dd r\,I
=\alpha^2\frac{1-e^{-2\gamma\Delta t}}{2\gamma}I.
\end{equation*}
This is exactly the velocity update with the coefficients in Eq.~\eqref{eq:rcm-exact-ou-coefficients}.

\paragraph{Continuous interpolation and uniform moment estimates}
Write \(h=\Delta t\), assume first that \(T=Nh\), and let
\(\pi(t)=t_k\), \(\pi^+(t)=t_{k+1}\) for \(t\in[t_k,t_{k+1})\). A final interval shorter than \(h\) is handled by the same estimates. Couple the exact and numerical processes by the same initial state and the same Brownian motion. Denote the numerical grid values by
\((\bar g_k,\bar{\bm\xi}_k)\), and set
\(\bar{\bm u}_k=\bm u(t_k,\bar g_k,\bar{\bm\xi}_k)\). On \([t_k,t_{k+1}]\), define the continuous velocity interpolation by the exact frozen-control OU flow
\begin{equation}\label{eq:proof-ou-interpolation}
\bar{\bm\xi}_t
=e^{-\gamma(t-t_k)}\bar{\bm\xi}_k
+\alpha\frac{1-e^{-\gamma(t-t_k)}}{\gamma}\bar{\bm u}_k
+\alpha\int_{t_k}^{t}
 e^{-\gamma(t-s)}\,\dd\mathbf W_s,
\end{equation}
and define the piecewise exponential group interpolation by
\begin{equation}\label{eq:proof-group-interpolation}
\bar g_t
=\bar g_k\Exp_\G\bigl((t-t_k)\bar{\bm\xi}_{k+1}\bigr).
\end{equation}
Both interpolants are continuous and agree with the stated discrete update at the grid points. The group interpolation is the standard continuous exponential interpolation of the splitting path; it need not be adapted inside a step, whereas the grid-point scheme and the frozen controls \(\bar{\bm u}_k\) are adapted.

We first record the moment and one-step increment estimates used below. Since \(\G\) is compact and the control has linear growth in \(\bm\xi\), It\^o's formula applied to
\(V(\bm\xi)=(1+\norm{\bm\xi}^2)^2\), followed by the Burkholder--Davis--Gundy and Young inequalities, gives for the exact controlled process
\begin{equation*}
\E\sup_{0\le t\le T}V(\bm\xi_t)
\le C_T\bigl(1+\E\norm{\bm\xi_0}^4\bigr).
\end{equation*}
Indeed, the drift term satisfies
\begin{equation*}
\bigl\langle\bm\xi,-\gamma\bm\xi+\alpha\bm u(t,g,\bm\xi)\bigr\rangle
\le C\bigl(1+\norm{\bm\xi}^2\bigr),
\end{equation*}
and the first two derivatives of \(V\) have at most cubic and quadratic growth, respectively, so the finite-variation and quadratic-variation contributions are bounded by \(C\int_0^T V(\bm\xi_s)\,\dd s\) after Young's inequality. The same argument applies on each numerical interval because Eq.~\eqref{eq:proof-ou-interpolation} solves
\begin{equation*}
\dd\bar{\bm\xi}_t
=-\gamma\bar{\bm\xi}_t\,\dd t
+\alpha\bar{\bm u}_k\,\dd t
+\alpha\,\dd\mathbf W_t,
\end{equation*}
with
\(\norm{\bar{\bm u}_k}^4\le C(1+\norm{\bar{\bm\xi}_k}^4)\). Conditional It\^o/BDG estimates on one step and discrete Gr\"onwall therefore yield
\begin{equation}\label{eq:proof-splitting-moments}
\E\sup_{0\le t\le T}
\bigl(1+\norm{\bm\xi_t}^4+\norm{\bar{\bm\xi}_t}^4\bigr)
\le C_T.
\end{equation}
More explicitly, on a numerical interval the conditional It\^o estimate has the form
\begin{equation*}
\E\left[
\sup_{t_k\le r\le t}
V(\bar{\bm\xi}_r)
\,\middle|\,\mathcal F_{t_k}
\right]
\le
V(\bar{\bm\xi}_k)
+C\int_{t_k}^{t}
\E\left[
\sup_{t_k\le q\le r}V(\bar{\bm\xi}_q)
\,\middle|\,\mathcal F_{t_k}
\right]\dd r
+Ch\bigl(1+\norm{\bar{\bm\xi}_k}^4\bigr),
\end{equation*}
where the martingale term is absorbed by BDG and Young inequalities. Gr\"onwall on the step gives
\begin{equation*}
\E\left[
\sup_{t_k\le r\le t_{k+1}}V(\bar{\bm\xi}_r)
\,\middle|\,\mathcal F_{t_k}
\right]
\le (1+Ch)V(\bar{\bm\xi}_k)+Ch,
\end{equation*}
and discrete Gr\"onwall yields Eq.~\eqref{eq:proof-splitting-moments}. Thus all constants below can be chosen independently of \(0<h\le h_0\le1\).

For completeness, the increment estimate follows directly from Eq.~\eqref{eq:proof-ou-interpolation}. If \(r=t-t_k\in[0,h]\), then
\begin{align*}
\E\norm{\bar{\bm\xi}_t-\bar{\bm\xi}_k}^2
&\le
3(1-e^{-\gamma r})^2\E\norm{\bar{\bm\xi}_k}^2
+3\alpha^2\left(\frac{1-e^{-\gamma r}}{\gamma}\right)^2
\E\norm{\bar{\bm u}_k}^2\\
&\quad
+3\alpha^2\int_{t_k}^{t}e^{-2\gamma(t-s)}\dd s
\le Cr.
\end{align*}
Applying the same calculation to the frozen OU flow from time \(t\) to \(t_{k+1}\) gives
\(\E\norm{\bar{\bm\xi}_{k+1}-\bar{\bm\xi}_t}^2\le C(h-r)\). Hence, uniformly for \(t\in[0,T]\),
\begin{equation}\label{eq:proof-splitting-increments}
\E\norm{\bar{\bm\xi}_t-\bar{\bm\xi}_{\pi(t)}}^2
\le Ch,
\qquad
\E\norm{\bar{\bm\xi}_{\pi^+(t)}-\bar{\bm\xi}_t}^2
\le Ch.
\end{equation}
For the group interpolation, the one-parameter curve
\(r\mapsto \bar g_k\Exp_\G(r\bar{\bm\xi}_{k+1})\) has constant speed
\(\norm{\bar{\bm\xi}_{k+1}}\) in the bi-invariant metric. Hence
\begin{equation}\label{eq:proof-group-increment}
\E d_\G(\bar g_t,\bar g_{\pi(t)})^2
\le h^2\E\norm{\bar{\bm\xi}_{\pi^+(t)}}^2
\le Ch^2.
\end{equation}

\paragraph{Velocity error}
Let
\begin{equation*}
\delta\bm\xi_t=\bm\xi_t-\bar{\bm\xi}_t,
\qquad
\Delta\bm u_t
=\bm u(t,g_t,\bm\xi_t)
-\bm u\bigl(\pi(t),\bar g_{\pi(t)},\bar{\bm\xi}_{\pi(t)}\bigr).
\end{equation*}
Because the exact and interpolated velocity equations are driven by the same Brownian motion, the stochastic terms cancel. Variation of constants across successive intervals gives, for every \(t\le T\),
\begin{equation}\label{eq:proof-velocity-error-representation}
\delta\bm\xi_t
=\alpha\int_0^t e^{-\gamma(t-s)}\Delta\bm u_s\,\dd s.
\end{equation}
The joint Lipschitz assumption, followed by the triangle inequality and Eqs.~\eqref{eq:proof-splitting-increments}--\eqref{eq:proof-group-increment}, yields
\begin{align}
\E\norm{\Delta\bm u_t}^2
&\le C\Bigl(
 h^2
 +\E d_\G(g_t,\bar g_{\pi(t)})^2
 +\E\norm{\bm\xi_t-\bar{\bm\xi}_{\pi(t)}}^2
 \Bigr)\notag\\
&\le C\Bigl(
 h
 +\E d_\G(g_t,\bar g_t)^2
 +\E\norm{\delta\bm\xi_t}^2
 \Bigr).
\label{eq:proof-control-difference}
\end{align}
Using \(e^{-\gamma(t-s)}\le1\) and Cauchy--Schwarz in Eq.~\eqref{eq:proof-velocity-error-representation}, we obtain
\begin{equation}\label{eq:proof-velocity-strong}
\E\sup_{0\le r\le t}\norm{\delta\bm\xi_r}^2
\le C\int_0^t
\left(
 h
 +\E d_\G(g_s,\bar g_s)^2
 +\E\norm{\delta\bm\xi_s}^2
\right)\dd s.
\end{equation}

\paragraph{Group error without a local logarithm}
A bi-invariant metric yields the following global pathwise comparison directly on the group. If two absolutely continuous group curves satisfy
\begin{equation*}
\dot g_t=T_eL_{g_t}(\bm\zeta_t),
\qquad
\dot q_t=T_eL_{q_t}(\bm\eta_t),
\end{equation*}
then, with \(r_t=g_tq_t^{-1}\), the left-trivialized velocity of \(r_t\) is
\(\operatorname{Ad}_{q_t}(\bm\zeta_t-\bm\eta_t)\). The norm is \(\operatorname{Ad}\)-invariant, so the length bound for \(r\) gives the pathwise estimate
\begin{equation}\label{eq:proof-group-variation}
d_\G(g_t,q_t)
\le d_\G(g_0,q_0)
+\int_0^t\norm{\bm\zeta_s-\bm\eta_s}\,\dd s.
\end{equation}
For the numerical interpolation \eqref{eq:proof-group-interpolation}, the reconstruction velocity on \([t_k,t_{k+1}]\) is
\(\widehat{\bm\xi}_t=\bar{\bm\xi}_{k+1}=\bar{\bm\xi}_{\pi^+(t)}\). Since the exact and numerical processes start from the same group state, Eq.~\eqref{eq:proof-group-variation} gives
\begin{align*}
\sup_{0\le r\le t}d_\G(g_r,\bar g_r)
&\le\int_0^t\norm{\bm\xi_s-\widehat{\bm\xi}_s}\,\dd s\\
&\le\int_0^t\norm{\delta\bm\xi_s}\,\dd s
+\int_0^t\norm{\bar{\bm\xi}_s-\bar{\bm\xi}_{\pi^+(s)}}\,\dd s.
\end{align*}
Therefore, by Cauchy--Schwarz and Eq.~\eqref{eq:proof-splitting-increments},
\begin{equation}\label{eq:proof-group-strong}
\E\sup_{0\le r\le t}d_\G(g_r,\bar g_r)^2
\le C\int_0^t
\E\sup_{0\le q\le s}\norm{\delta\bm\xi_q}^2\,\dd s
+ C_T h.
\end{equation}
The second term is precisely the splitting/reconstruction defect: although the OU velocity varies inside the step, the Lie--Trotter reconstruction uses the endpoint velocity \(\bar{\bm\xi}_{k+1}\). Equation~\eqref{eq:proof-splitting-increments} shows that its mean-square contribution over the full horizon is \(O(h)\). Thus the contribution associated with noncommutative group reconstruction is controlled at order \(O(h)\) in mean square over the full horizon.

\paragraph{Global strong error and bounded-Lipschitz path law}
Define
\begin{equation*}
A(t)
=
\E\sup_{0\le r\le t}
\left[
 d_\G(g_r,\bar g_r)^2
 +\norm{\bm\xi_r-\bar{\bm\xi}_r}^2
\right].
\end{equation*}
Combining Eqs.~\eqref{eq:proof-control-difference}, \eqref{eq:proof-velocity-strong}, and \eqref{eq:proof-group-strong} gives
\begin{equation*}
A(t)
\le C_T h+C_T\int_0^t A(s)\,\dd s.
\end{equation*}
Gr\"onwall's inequality therefore yields the uniform strong path estimate
\begin{equation}\label{eq:proof-splitting-strong-error}
\E\sup_{0\le t\le T}
\left[
 d_\G(g_t,\bar g_t)^2
 +\norm{\bm\xi_t-\bar{\bm\xi}_t}^2
\right]
\le C_T h.
\end{equation}

Let \(X=(g,\bm\xi)\) and \(\bar X=(\bar g,\bar{\bm\xi})\). Since the path metric in Eq.~\eqref{eq:bl-distance-definition} is built from
\(\overline d_\X=1\wedge(d_\G+\norm{\cdot})\), Eq.~\eqref{eq:proof-splitting-strong-error} implies
\begin{equation*}
\E d_{\rm path}(X,\bar X)
\le
\left(
\E\sup_{0\le t\le T}
\bigl[d_\G(g_t,\bar g_t)+\norm{\bm\xi_t-\bar{\bm\xi}_t}\bigr]^2
\right)^{1/2}
\le C_T h^{1/2}.
\end{equation*}
Hence, for every path functional \(F\) with \(\norm{F}_{\rm BL}\le1\),
\begin{equation*}
\left|\E F(X)-\E F(\bar X)\right|
\le \E d_{\rm path}(X,\bar X)
\le C_T h^{1/2}.
\end{equation*}
Taking the supremum over such \(F\) proves
\begin{equation*}
d_{\rm BL}(P_h^{\bm u},P^{\bm u})
\le C_T h^{1/2},
\end{equation*}
which is Eq.~\eqref{eq:rccbm-splitting-error}. Every constant used above depends on the controller only through the common bounds \(L_u\) and \(K_u\); consequently the estimate is uniform over a control sieve satisfying those bounds.

\subsection{Proof of Proposition~3.10}

Under Assumption~\ref{ass:compact-internal-representation},
\(F:\G_{\rm red}\to\mathcal S\) is a \(C^2\)-diffeomorphism. Hence its differential
\(\dd F_z:T_z\G_{\rm red}\to T_{F(z)}\mathcal S\) is a linear isomorphism at every \(z\). The state lift
\begin{equation*}
\Phi_{\rm red}(z,\bm\zeta)
=\bigl(F(z),\dd F_z[T_eL_z(\bm\zeta)]\bigr)
\end{equation*}
is therefore continuous and injective, with continuous inverse obtained from
\(\mathcal R_{\rm red}|_{\mathcal S}=F^{-1}\), the inverse differential, and the Lie algebra coordinate identification induced by left translation. In particular, it is a Borel isomorphism between the reduced kinetic state space and its image. Applying \(\Phi_{\rm red}\) pointwise to a continuous path defines a Borel isomorphism
\(\mathbf\Phi_{\rm red}\) of the corresponding continuous path spaces onto its image.

Let \(P\ll Q\) be two reduced path laws and set
\(\widetilde P=(\mathbf\Phi_{\rm red})_\#P\),
\(\widetilde Q=(\mathbf\Phi_{\rm red})_\#Q\). Because the map is bijective onto its image,
\begin{equation*}
\frac{\dd\widetilde P}{\dd\widetilde Q}
\bigl(\mathbf\Phi_{\rm red}(\omega)\bigr)
=\frac{\dd P}{\dd Q}(\omega)
\qquad Q\text{-a.e.}
\end{equation*}
Thus, by change of variables,
\begin{align*}
\KL(\widetilde P\|\widetilde Q)
&=\int\log\frac{\dd\widetilde P}{\dd\widetilde Q}\,\dd\widetilde P\\
&=\int\log\frac{\dd P}{\dd Q}\,\dd P
=\KL(P\|Q),
\end{align*}
which proves Eq.~(3.26) with
\(P=P_{\rm red}^{\rm SB}\) and \(Q=\widehat P_{\rm red}\).

If the reconstruction is not injective, an inverse density transformation is unavailable. For any measurable map \(T\), the variational representation of relative entropy gives
\begin{align*}
\KL(T_\#P\|T_\#Q)
&=\sup_f\left\{\E_P[f\circ T]-\log\E_Q[e^{f\circ T}]\right\}\\
&\le\sup_\psi\left\{\E_P[\psi]-\log\E_Q[e^\psi]\right\}
=\KL(P\|Q),
\end{align*}
which is the data-processing inequality in Eq.~(3.27). Finally,
\(SO(3)^{N-1}\times\T^K\) is compact, and it is non-Abelian whenever at least one \(SO(3)\) factor is present. Therefore the reduced learning problem lies in the compact non-Abelian RCCBM regime, provided the endpoint and reference assumptions used earlier are verified for the reduced laws.

The preceding proof sequence establishes the main theoretical dependencies used in the article.

\section{Additional Experimental Results}\label{app:supplementary-results}

\begin{figure}[!tbp]
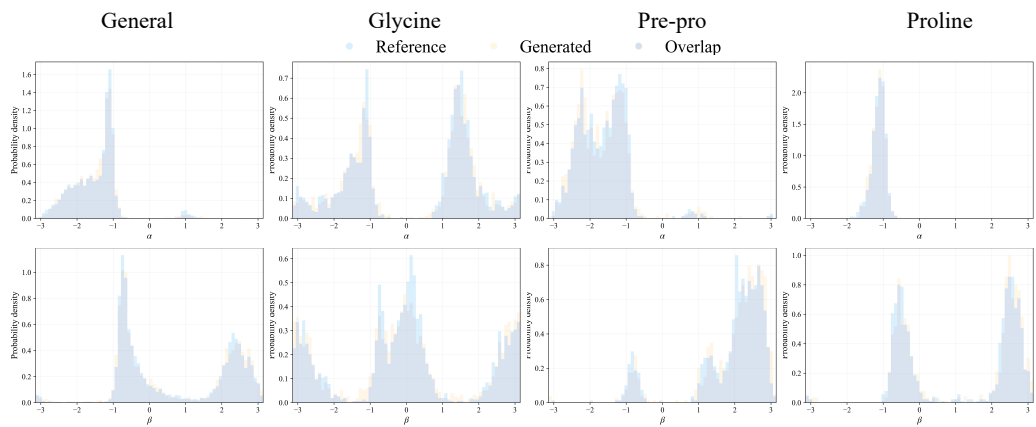

\centering
\pdffig[width=0.98\textwidth,height=0.58\textheight,keepaspectratio]{Protein-histogram}
\caption{One-dimensional protein torsion marginals for the General, Glycine, Pre-Pro, and Proline subsets, comparing the reference and WKBC-generated distributions. 
These coordinate-wise marginals complement the joint wrapped-coordinate endpoint visualization reported in the main text.}
\label{fig:appendix-protein-histograms}
\end{figure}

Fig.~\ref{fig:appendix-protein-histograms} provides one-dimensional protein torsion marginals as an additional diagnostic of the wrapped-coordinate experiments.

\begin{figure}[!tbp]
\centering
\includegraphics[width=0.98\textwidth]{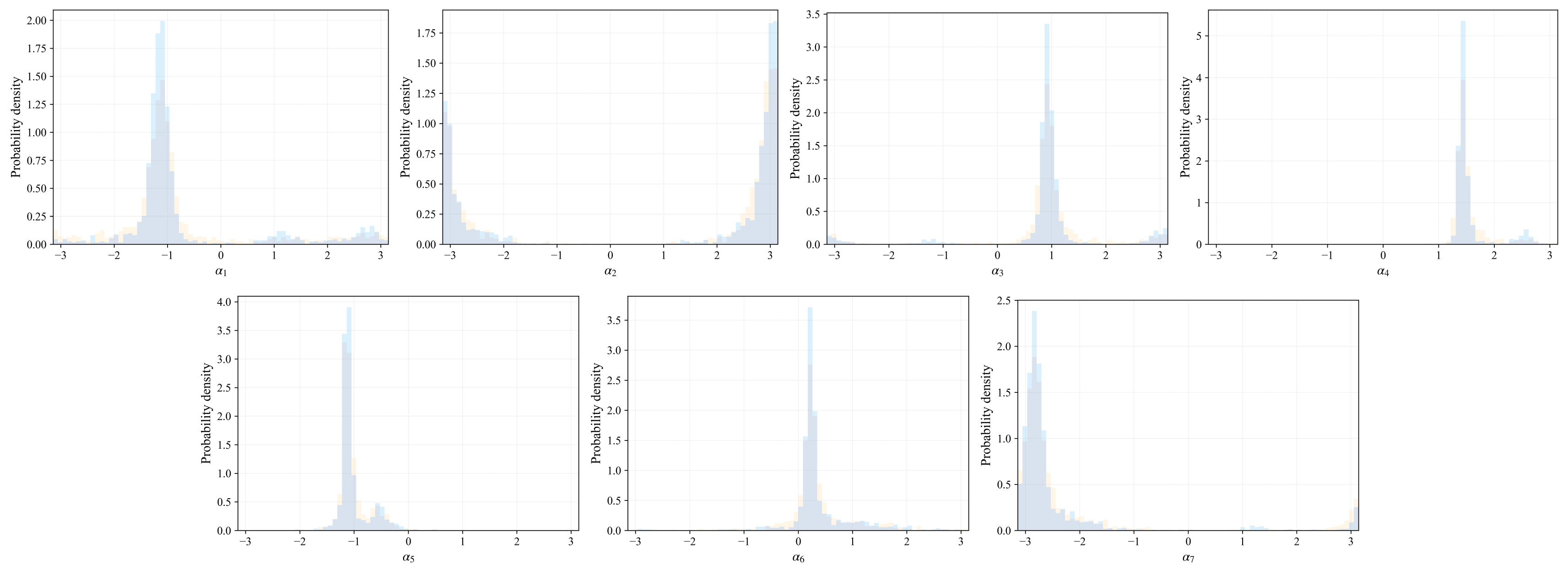}
\caption{Coordinate-wise RNA torsion marginals on $\mathbb{T}^{7}$, comparing the reference and WKBC-generated distributions. 
The seven periodic coordinates exhibit heterogeneous marginal structures ranging from sharply concentrated modes to broader asymmetric profiles.}
\label{fig:appendix-rna-histograms}
\end{figure}

Fig.~\ref{fig:appendix-rna-histograms} shows strongly coordinate-dependent periodic marginals, ranging from sharply concentrated modes to broader asymmetric profiles. These heterogeneous shapes illustrate the non-Euclidean structure of the seven RNA torsions and complement the quantitative endpoint and path results in Table~3.

\begin{figure}[!tbp]
\centering
\includegraphics[width=0.97\textwidth]{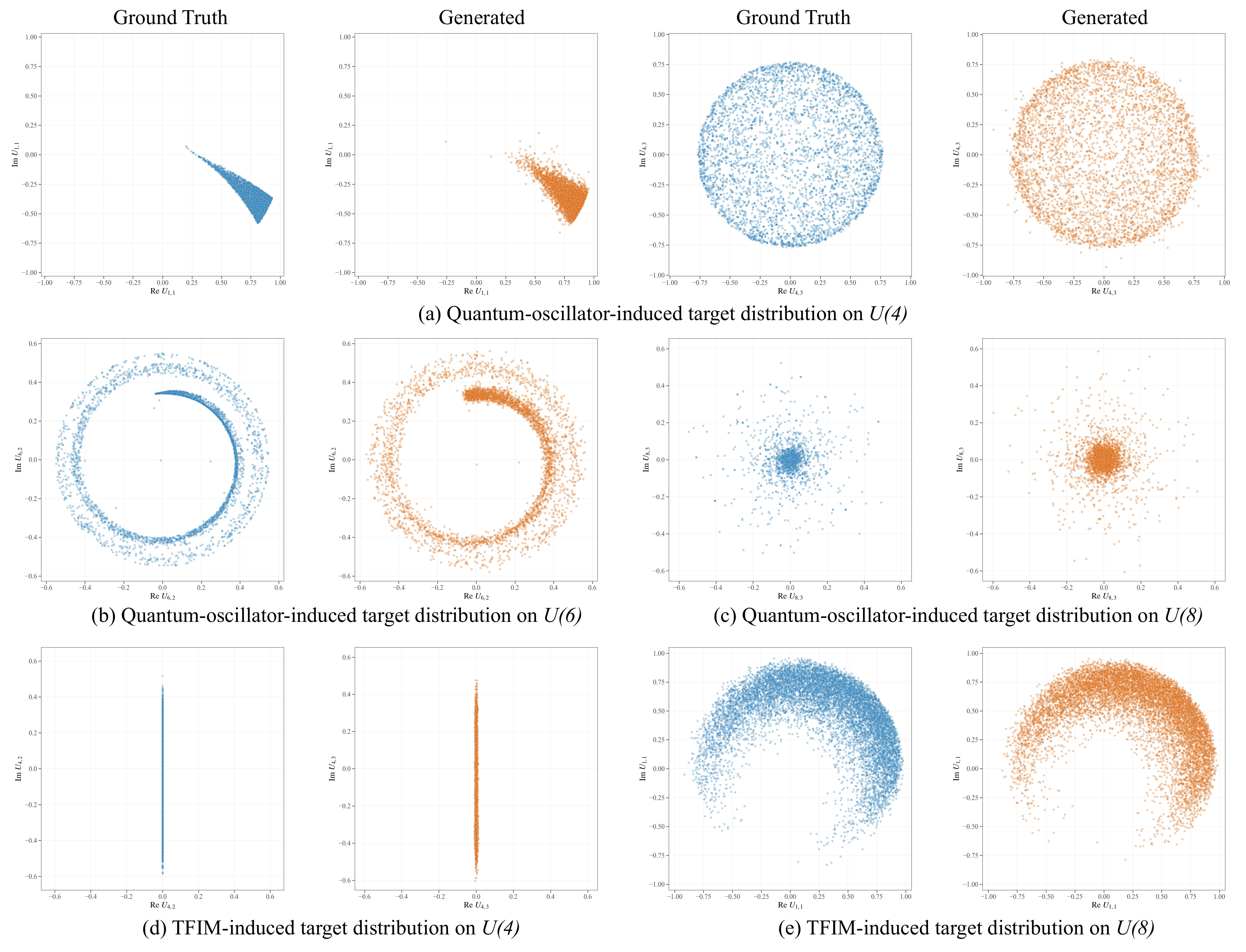}
\caption{Projected reference and RCCBM-generated samples for the oscillator- and TFIM-induced $U(n)$ tasks. 
Panels (a)--(c) show the quantum-oscillator-induced distributions on $U(4)$, $U(6)$, and $U(8)$, respectively, while panels (d)--(e) show the TFIM-induced distributions on $U(4)$ and $U(8)$.}
\label{fig:appendix-un-visualization}
\end{figure}

Fig.~\ref{fig:appendix-un-visualization} shows the projected samples used for the qualitative comparison. Across the oscillator-induced \(U(4)\), \(U(6)\), and \(U(8)\) tasks and the TFIM-induced \(U(4)\) and \(U(8)\) tasks, the paired projections show close agreement in their dominant support geometry, including annular, spiral, crescent-like, and strongly anisotropic structures. The visual agreement complements the RCCBM diagnostics in Table~6 and supports the model's ability to represent structured distributions on unitary groups of different dimensions.

\section{Experiment Settings}\label{app:experiment-settings}

\subsection{Endpoint, reference, and optimization specification}\label{app:training-settings}

Every experiment specifies \((\mathcal O_0,\rho_0)\), \((\mathcal O_T,\rho_T)\), the reference initial density \(q_0\), \((\gamma,\alpha,T)\), the endpoint mollifier \(\kappa_\varepsilon\), and the integration grid. WKBC uses the wrapped-kernel implementation, whereas RCCBM uses two bounded endpoint-corrector networks and a DirectControlNet. The random seed is fixed to 1209 for the reported training and sampling configurations.

For RCCBM, Eq.~\eqref{eq:rcm-calibrated-initial-law} approximates the mollified population initial law and is used for refinement and generation; teacher conditions are sampled or importance-weighted toward Eq.~(3.18). Replay balancing and terminal oversampling use the importance correction in Eq.~\eqref{eq:rccbm-replay-importance}. Refinement is accepted only if Eq.~\eqref{eq:rccbm-constrained-refinement} holds. Endpoint calibration, CondSOC, importance-weighted control matching, constrained refinement, and validation are logged as distinct stages.

\paragraph{RCCBM ablations}
The active-component panel evaluates RCCBM variants w/o endpoint calibration, terminal replay, the local-scale term, distribution refinement, spectral features, or calibrated initial-state sampling, together with a fixed-common-random-number CondSOC variant. The calibrated-initial-state ablation directly probes the empirical initial-law consistency condition in Assumption~\ref{ass:rccbm-statistical-stability}. The General-task WKBC ablation additionally varies endpoint smoothing and the time-sampling distribution.

\begin{table}[!tbp]
\centering
\caption{Training and sampling settings used for the reported RCCBM experiments on $SO(3)$ and $U(n)$. 
The table summarizes endpoint calibration, reciprocal training, control and refinement batches, network architecture, replay, integration, and calibrated-initial-state settings; the final column specifies the associated implementation or reporting convention.}
\label{tab:essential-training-settings}
\small
\setlength{\tabcolsep}{2.6pt}
\renewcommand{\arraystretch}{1.12}
\begin{tabular}{p{0.25\textwidth}p{0.22\textwidth}p{0.22\textwidth}p{0.23\textwidth}}
\toprule
Parameter & \(SO(3)\) v5 & \(U(n)\) v5 & Reporting convention \\
\midrule
Endpoint calibration & 600 steps; pool 4096 & 600 steps; pool 4096 & bounded endpoint nets, width 192, 4 residual blocks each \\
Outer epochs & 30 & 30 & reciprocal outer epochs \\
Updates per epoch & 1200 match + 300 refine & 1100 match + 400 refine & report stages separately \\
Control batch & 4096 & 3072 & replay control matching \\
Refinement batch & 128 & 128 & fresh stochastic rollouts \\
Coupling batch & 512 & 256 & reference-weighted triples \\
CondSOC smoothness \(\lambda_{\rm SE}\) & 0.5 & 0.5 & empirical smoothness weight \\
Direct control net & width 512, 10 blocks & width 512, 18 blocks & LayerNorm--SiLU, residual divided by \(\sqrt2\) \\
Time embedding & 128 & 128 & sinusoidal plus time-to-go \\
Terminal-window start \(a_{\rm term}\) & 0.75 & 0.75 & terminal component supported on \([a_{\rm term}T,T]\) \\
Terminal replay fraction \(r_{\rm term}\) & 0.55 & 0.65 & probability of terminal-window time draw \\
Derived density floor \(\underline\lambda T\) & 0.600 & 0.467 & Eq.~\eqref{eq:rccbm-time-density-floor}; derived, not tuned \\
Generation & 128 stochastic steps & 128 stochastic steps & exact-OU Lie--Trotter \\
Calibrated initial state & enabled & enabled & same law for teacher, refinement, generation \\
On-policy refresh & not applied & not applied & fixed reciprocal training schedule \\
\bottomrule
\end{tabular}
\end{table}

Table~\ref{tab:essential-training-settings} records the training and sampling settings used for the reported RCCBM experiments. The time-density lower bound is determined exactly by Eq.~\eqref{eq:rccbm-time-density-floor}: the reported \(SO(3)\) and \(U(n)\) mixtures give \(\underline\lambda=0.60/T\) and \(7/(15T)\approx0.467/T\), respectively. Thus \(\underline\lambda\) is a derived theorem constant, whereas \(a_{\rm term}\) and \(r_{\rm term}\) are the actual sampling hyperparameters. For WKBC, the bridge-weighted regression of Proposition~\ref{thm:wkbc-stability} uses the fixed uniform time density \(1/T\); the former occupation-ratio constant \(C_{\rm occ}\) is not used or tuned, as explained in Remark~\ref{rem:wkbc-no-cocc}.

The CondSOC terminal kernel uses the reported concentration parameters 600 on \(SO(3)\) and 800 on \(U(n)\). The normalized control-energy coefficient in Eq.~\eqref{eq:rcm-condsoc-objective} is fixed at \(1/2\), and the empirical smoothness coefficient is \(\lambda_{\rm SE}=0.5\) in all reported RCCBM runs. Independent noise is used during optimization and for the held-out teacher rollout, with antithetic sampling when configured. The final checkpoint is chosen by the fixed validation distribution objective. Final validation MSE, relative MSE, and cosine are recomputed after distribution refinement.

\paragraph{Computational workload}
A generated RCCBM trajectory uses 128 direct-control evaluations. Training workload is reported as endpoint calibration plus matching and refinement updates; CondSOC inner iterations and teacher rollout counts are listed separately in the released configuration. WKBC workload reporting remains unchanged and includes lattice, quadrature, Sinkhorn, score-regression, and integration costs.

\paragraph{Reproducibility record}
Every numerical row is indexed by a configuration file recording the seed, data split, reference pool, endpoint checkpoint rule, ESS threshold, CondSOC settings, replay sampler, refinement weights, validation panel, and evaluator settings. The manifest also records the calibrated-initial-state and on-policy-refresh settings used for each run.

\subsection{Geometric and bridge diagnostics}\label{app:metrics}

\subsubsection{RCCBM endpoint, teacher, control, and structure diagnostics}\label{app:bridge-diagnostics}

For normalized pool weights \(\overline w_i\),
\begin{equation*}
    \operatorname{ESS}=\left(\sum_i\overline w_i^2\right)^{-1},
    \qquad
    \operatorname{ESS}(\%)=100\,\frac{\operatorname{ESS}}{N_Q}.
\end{equation*}
Endpoint calibration reports the empirical dual, ESS, and held-out marginal feature error. CondSOC reports train and independent-noise terminal error, normalized teacher energy, and smoothness. Direct control matching reports MSE, relative MSE, and cosine on a fixed validation replay; all three are recomputed after terminal refinement.

For \(SO(3)\), structure error reports
\(\norm{R^\top R-I}_{\rm F}\) and \(|\det R-1|\), with mean and maximum. For \(U(n)\), it reports
\(\norm{U^\dagger U-I}_{\rm F}\). These quantities measure numerical preservation of the Lie group constraint under the exponential update.

\subsubsection{MMD, sliced-Wasserstein, local precision/coverage, and intrinsic costs}

{The RCCBM feature MMD uses a positive finite mixture of Gaussian kernels on the fixed full-rank target-whitened identification features \(\Phi_{\G}^{\rm id}\) defined in Eq.~\eqref{eq:rccbm-production-identification-feature}.  For \(SO(3)\), this map is obtained from all nine matrix entries; for \(U(n)\), it uses all real and imaginary matrix entries; and for the compact reduced product group it additionally uses the standard \((\cos\theta,\sin\theta)\) embedding of each torus coordinate.  The regularized whitening transform is invertible and full dimensional.  Spectral features used in the \(U(n)\) controller serve as auxiliary optimization-facing features, while the fixed target-whitened embedding \(\Phi_{\G}^{\rm id}\) defines the MMD identification map.} A target-versus-target split provides the finite-sample floor. The sliced-Wasserstein term uses a fixed bank of random directions. Local precision and coverage are reported separately from generated-to-target and target-to-generated kNN radii, and their combined local training term aggregates these complementary neighborhood diagnostics.

For empirical endpoint measures, the debiased Sinkhorn divergence is
\begin{equation*}
    S_\varepsilon(\widehat\rho,\widehat\nu)
    =\operatorname{OT}_\varepsilon(\widehat\rho,\widehat\nu)
      -\frac12\operatorname{OT}_\varepsilon(\widehat\rho,\widehat\rho)
      -\frac12\operatorname{OT}_\varepsilon(\widehat\nu,\widehat\nu),
\end{equation*}
using squared intrinsic torus or \(SO(3)\) cost. The regularization, sample count, and stopping rule are fixed across methods. We report the resulting finite-sample entropic geodesic Wasserstein quantity as a proxy for the population Wasserstein discrepancy.

For paired torus samples,
\begin{equation*}
    \operatorname{GeoRMSE}
    =\left[\frac{1}{Nd}\sum_{n=1}^N\sum_{j=1}^d
      \wrap(\widehat\theta_{n,j}-\theta_{n,j})^2\right]^{1/2}.
\end{equation*}
The pairing is fixed before training, so GeoRMSE measures geometric error under this predetermined evaluation correspondence. Ramachandran JSD and ValidRate use a common periodic partition and a validity region fixed without generated test samples.

\subsubsection{Normalized-control energy}

If the velocity drift correction is \(\alpha\bm u_t\), the reported energy is
\begin{equation*}
    \mathcal E_{\rm path}
    =\frac12\E\int_0^T\norm{\bm u_t}^2\,\dd t.
\end{equation*}
Cross-method comparison requires the same reference process, diffusion coefficient, time horizon, initial-state law, and integration grid.

\subsubsection{Protein conformational transition pathway evaluation}\label{app:protein-transition-pathway-evaluation}

The pathway diagnostics for the Protein Conformational Transition Pathway Generation task are evaluated in the same reduced state space used by the generator, using molecular-dynamics trajectories from mdCATH domain 1jvmB00. The production representation is
\[
    \G_{\rm red}=SO(3)^{63}\times\T^{192}.
\]
Write two reduced conformations as
\[
 g=(R_1,\ldots,R_{n_R},\theta_1,\ldots,\theta_K),
 \qquad
 h=(S_1,\ldots,S_{n_R},\phi_1,\ldots,\phi_K),
\]
with \(n_R=63\) and \(K=192\). For each rotation factor and torus coordinate, respectively, define
\[
    \delta_i^R(g,h)=\bigl(\Log(R_i^\top S_i)\bigr)^\vee,
    \qquad
    \delta_j^\theta(g,h)=\wrap(\phi_j-\theta_j)\in(-\pi,\pi].
\]
The intrinsic product distance used by the evaluator is
\begin{align*}
    d_{\rm red}(g,h)^2
    &=\frac{1}{n_R}\sum_{i=1}^{n_R}\norm{\delta_i^R(g,h)}^2
      +\lambda_\theta\frac{1}{K}\sum_{j=1}^{K}\abs{\delta_j^\theta(g,h)}^2,\\
    \lambda_\theta&=\frac{K}{n_R}=\frac{192}{63}\approx3.047619.
\end{align*}
This choice assigns the same coefficient to every Lie algebra coordinate after accounting for the three coordinates in each \(\mathfrak{so}(3)\) factor. It agrees with the production geometry up to a common global scale; the two reported dimensionless diagnostics are unchanged by such a common scaling.

\paragraph{Common temporal grid}
Each MD pair contains the 101 stored states from the source frame through the target frame at lag 100. The RCCBM sampler uses 128 stochastic integration steps. To prevent the roughness statistic from being confounded by different temporal resolutions, both trajectories are evaluated on the same normalized-time grid \(s_m=m/100\), \(m=0,\ldots,100\). RCCBM states are interpolated intrinsically: for \(R_a,R_b\in SO(3)\),
\[
    R(\alpha)=R_a\Exp\!\bigl(\alpha\Log(R_a^\top R_b)\bigr),
    \qquad 0\le\alpha\le1,
\]
and each torus coordinate is interpolated along the wrapped angular difference.

\paragraph{Target-aware path tortuosity}
For a discrete trajectory \(\gamma=(g_0,\ldots,g_M)\) with prescribed target \(g_T^\star\), define its accumulated intrinsic length by
\[
    L(\gamma)=\sum_{k=0}^{M-1}d_{\rm red}(g_k,g_{k+1}).
\]
The target-aware tortuosity is
\[
    \tau_T(\gamma)
    =\frac{L(\gamma)+d_{\rm red}(g_M,g_T^\star)}
           {d_{\rm red}(g_0,g_T^\star)}.
\]
The terminal-residual term is essential for the generated trajectories: it prevents a path that fails to reach the prescribed target from receiving an artificially favorable score merely because it travels a shorter distance. For the observed MD segment, \(g_M=g_T^\star\) and the residual vanishes. By the triangle inequality, \(\tau_T\ge1\); values closer to one indicate less geometric detour. For pair \(p\), with one MD path and \(S=32\) RCCBM samples, the reported pairwise gain is
\[
    \operatorname{DirectnessGain}_p
    =100\left(1-
      \frac{S^{-1}\sum_{s=1}^S\tau_T(\gamma_{p,s}^{\rm SB})}
           {\tau_T(\gamma_p^{\rm MD})}\right)\%.
\]

\paragraph{Normalized Lie algebra path roughness}
For each adjacent pair of states, define the Lie algebra increment \(\eta_k\in\mathfrak g_{\rm red}\) by
\[
    \eta_{k,i}^R=\bigl(\Log((R_{k,i})^\top R_{k+1,i})\bigr)^\vee,
    \qquad
    \eta_{k,j}^\theta=\wrap(\theta_{k+1,j}-\theta_{k,j}).
\]
The squared norm is weighted consistently with the product distance,
\[
    \norm{\eta_k}_W^2
    =\frac{1}{n_R}\sum_{i=1}^{n_R}\norm{\eta_{k,i}^R}^2
     +\lambda_\theta\frac{1}{K}\sum_{j=1}^{K}\abs{\eta_{k,j}^\theta}^2.
\]
The normalized roughness statistic is
\[
    \mathcal R(\gamma)
    =\frac{\displaystyle\sum_{k=0}^{M-2}
                \norm{\eta_{k+1}-\eta_k}_W^2}
           {\displaystyle\sum_{k=0}^{M-1}
                \norm{\eta_k}_W^2+\varepsilon},
    \qquad \varepsilon=10^{-12}.
\]
The numerator measures changes in successive local displacement vectors, whereas the denominator normalizes by the total local motion. Thus a trajectory with approximately persistent local direction and amplitude has small roughness, while repeated reversals, abrupt directional changes, or local oscillations increase the statistic. The quantity is nonnegative but is not constrained to be below one. The pairwise relative reduction is
\[
    \operatorname{RoughnessReduction}_p
    =100\left(1-
      \frac{S^{-1}\sum_{s=1}^S\mathcal R(\gamma_{p,s}^{\rm SB})}
           {\mathcal R(\gamma_p^{\rm MD})}\right)\%.
\]

\paragraph{Aggregation and evaluation protocol}
For each endpoint pair, the 32 stochastic trajectories are summarized by their mean and standard deviation. Aggregate percentages are computed from the ratio of the \(P=15\) pair-level means,
\[
    100\left(1-
    \frac{P^{-1}\sum_{p=1}^{P}\overline\tau_{p}^{\rm SB}}
         {P^{-1}\sum_{p=1}^{P}\tau_{p}^{\rm MD}}\right)
    =95.06\%,
    \qquad
    100\left(1-
    \frac{P^{-1}\sum_{p=1}^{P}\overline{\mathcal R}_{p}^{\rm SB}}
         {P^{-1}\sum_{p=1}^{P}\mathcal R_{p}^{\rm MD}}\right)
    =99.21\%,
\]
with \(P=15\). The evaluation uses a prespecified held-out pair list, and target-aware tortuosity is computed from the prescribed target according to the rule defined above.


\end{document}